\documentclass{article}

\usepackage{arxiv}

\usepackage[utf8]{inputenc}
\usepackage[T1]{fontenc}
\usepackage{hyperref}
\usepackage{url}
\usepackage{booktabs}
\usepackage{amsfonts}
\usepackage{nicefrac}
\usepackage{microtype}
\usepackage{xcolor}

\usepackage{amsmath,amssymb,amsthm,bbm,mathtools}

\usepackage{algpseudocode}
\usepackage{algorithm}
\usepackage{wrapfig}
\usepackage{natbib}

\usepackage[position=top]{subcaption}
\usepackage{graphicx}

\newcommand{\Prb}{\mathbb{P}}
\newcommand{\1}{\mathbbm{1}}

\renewcommand{\star}{*}

\usepackage[capitalize,noabbrev]{cleveref}

\theoremstyle{plain}
\newtheorem{theorem}{Theorem}[section]

\newtheorem{lemma}[theorem]{Lemma}

\theoremstyle{definition}

\theoremstyle{remark}

\title{Bayesian Learning to Select Fair Classifiers}

\author{Saatvik Kher \\ University of California, Irvine \And
Hyungrok Do \\ New York University \And
Judy Zhong \\ Weill Cornell Medicine \And
Padhraic Smyth \\ University of California, Irvine}

\newcommand{\undertitle}{}
\date{}

\begin{document}

\maketitle
\begin{abstract}
  Evaluations of classifiers (AI or human or agents) often assume a \textit{batch} setting in which every agent produces labels on every example in a test set, which can be expensive or impractical in many contexts. This motivates \textit{online} policies that adaptively select agents to query as examples arrive and update policies as outcomes are observed.
  We study the problem of selecting, in an online stream, the agent that best balances (i) predictive performance and (ii) subgroup fairness of a sensitive attribute, and that allows for delayed observations of true class labels.
  We formulate this as a Bayesian multi-armed bandit problem using Thompson Sampling, explicitly balancing accuracy and fairness.
  We conduct empirical evaluations on real-world datasets spanning problems involving tabular classifiers, large language models, and human labelers, demonstrating substantial reductions in regret relative to baselines.
  Our theoretical contributions include guarantees on the performance of the selected agent that incorporate both the accuracy and fairness aspects of the problem.
\end{abstract}

\section{Introduction}

Machine learning systems are increasingly deployed as black-box decision-makers in high-stakes domains, such as healthcare diagnosis, content moderation, and credit decisions. This creates an operational need to evaluate individual models for accuracy as well as to select among multiple candidate decision-making agents (potentially including both ML models and human experts). In many applications, the criterion for selection is multi-faceted and needs to take into account not only prediction accuracy but also the fairness of an agent and (potentially) the cost of querying the agent.

Prior work in this general context often assumes a \textit{batch} setting in which each agent labels each example in a test set. While useful for offline audits, this approach is costly and infeasible when running multiple models, soliciting human expert judgments, or waiting for ground truth. Many scenarios require that models be deployed and predictive tasks allocated sequentially, necessitating that estimation of model performance and identifying the best-performing model be performed in concert with making actual predictions on real cases. Motivated by such scenarios, in this paper, we focus on learning \textit{online} selection protocols for such situations, balancing classification accuracy, fairness, and model querying cost,  making sequential choices and learning from observed outcomes as they are observed.

As a motivating example for our work, consider a scenario where an insurance company wishes to select a large language model (LLM) to classify customer claims into pre-defined categories. Initially, the company has access to a number of different LLMs and perhaps a single human worker to process the claims, but for operational and cost reasons would like to converge on using only a single agent in the longer-term. As a second example, a hospital may wish to evaluate the effectiveness of re-admission prediction models: if released, will a patient be re-admitted within 30 days? The hospital might have 3 internal models available for evaluation, say a simple linear model, a tree-based model, and a deep learning model. It wishes to select a single model to use going forward for operational and legal reasons. For each scenario, the insurance company or the hospital face the same question: is there a better approach to finding the best model in an online setting, beyond randomly or deterministically assigning claims or patients to models? This general model selection problem is common across application scenarios, where the models (or agents) being selected could include ML models and/or human experts.

In this context, we formalize fair online model selection as a multi-armed bandit problem where each arm corresponds to an agent (a pre-trained classifier or a human expert), and the learner selects an arm for every instance it observes from the sequential stream, observing only that agent’s prediction, and the eventual true label when available. Fairness is encoded through subgroup statistics (for example, subgroup positive-rate gaps) and combined with accuracy into a utility so the selector explicitly trades off these objectives. A Bayesian framework is natural here, as maintaining conjugate posteriors lets the learner represent uncertainty, update beliefs incrementally, and use posterior-sampling to balance exploration and exploitation of the different agents. The approach organically allows for delayed reward, per-arm costs, missing observations, and introducing new agents at any time.

Our primary contributions are the following:
\begin{itemize}
    \item We introduce the problem of online classifier selection in the context of balancing prediction accuracy, fairness, and querying cost. We formulate this within a multi-arm bandit (MAB) framework.
    \item We propose a Thompson sampling approach and develop new theoretical regret bounds that show how differences in sensitive attribute subgroup proportions directly affect convergence rates.
    \item We conduct a series of experiments across multiple real-world classification problems, including classifiers, LLMs, and human agents, and demonstrate significant superiority in efficiency of Thompson sampling (and other MAB approaches) over non-bandit methods such as round-robin querying.
\end{itemize}
For research reproducibility, we will make all our code and datasets publicly available on GitHub for the final version of our paper.

\section{Related Work}

Relevant prior research can broadly be grouped into two main categories, each of which we discuss below: (i) work in algorithmic fairness (largely focusing on offline settings), (ii) work on online model selection (distinct from fairness). In this context, a distinguishing novel feature of our work is that we investigate the intersection of these themes, namely, online model selection in the setting of algorithmic fairness.

\paragraph{Evaluating the Fairness of Prediction Models}
There has been a large body of work on the development of
 ``in-process" training algorithms that leverage fairness metrics to learn fair models (e.g., \citet{pessach2022review}). However,  less attention has been paid to the problem of statistical assessment of how fair pretrained models are, for example models that are being deployed ``in the wild." In the batch setting, there has been some work on developing general approaches for statistical assessment of the fairness properties of a pre-trained model,  emphasizing the need for uncertainty quantification in assessing if the model is fair or not \citep{mcnair2018preventing,ji2021active, miller2021model, barrainkua2024uncertainty,herlihy2024structured, yin2026fairness}. In the sequential online setting, \citet{chugg2023auditing} motivate the importance of statistical auditing of the fairness properties of a deployed prediction model and develop a family of sequential hypothesis testing approaches for this problem. Similarly, \citet{maneriker2023online} develop probabilistic bounding techniques for online monitoring of the fairness properties of a prediction model in a sample-efficient manner.

Our work differs from that described above in that we address the problem of sequential selection from a set of prediction agents (rather than monitoring a single agent) with the goal of identifying the agent that provides the best trade-off between accuracy and fairness (rather than establishing if a particular agent meets the definition of ``fair" or not).

\paragraph{Fair Bandit Algorithms} There has also been work on fair bandit algorithms that encourage fair selection of arms/contexts. \citet{schumann2022group} consider a setting in which the arms are partitioned into two or more sensitive groups, and the goal is to balance the arms pulled from both groups. Other examples include individual fairness in contextual bandits \citep{joseph2016fairness}, where lower-merit actions must not be chosen more than high-merit ones; calibrated fairness that matches selection probabilities to an agent's chance of being best \citep{liu2017calibrated}, with cumulative merit accounting and exposure fairness extensions \citep{wang2021fairness}; unknown-metric individual fairness learned from weak regulator feedback in linear contextual bandits \citep{gillen2018online}; and minimum pull-rate guarantees that ensure per-arm exposure for MABs \citep{chen2020fair,patil2021achieving}. These methods encourage fair arms selection but do not address fair predictions across sensitive groups as we do in our work.
At a high level, the closest work to our setting is \citet{huang2022achieving}, though it addresses user-side fairness exclusively within recommendation systems. Our approach provides regret bounds that explicitly incorporate sensitive subgroup probabilities and employs a Bayesian framework.

Another related line of work on \textbf{fair policy algorithms} \citep{jabbari2017fairness, wen2021algorithms,puranik2022dynamic,yin2023long}, studies fairness in reinforcement learning and Markov decision processes by accounting for long-term effects of decisions. However, these methods focus on policy fairness rather than our focus of predictive fairness across sensitive groups.%

\paragraph{Online Model Selection with Pretrained Models}
Another body of research relevant to our approach is prior work on online model selection from a set of pretrained models (e.g., classifiers or LLMs). The focus of this work is often on the design of efficient sequential algorithms to select the models with the best performance (according to some loss metric) in a streaming online setting.
For example,
\citet{huonline} address the problem of online selection of text-based generative models, conditioned on different prompts, using a contextual-bandit framework with kernel-based feature representations and an upper-confidence-bound (UCB) bandit algorithm. \citet{xia2024llm} also use a UCB bandit approach for the specific problem of selecting the best LLM in the context of non-stationarity, where the performance of models (LLMs) may be increasing over time.
\citet{wei2025learning} present an LLM routing algorithm trained online with bandits.
\citet{dai2024cost} develop combinatorial cost-aware bandit algorithms for selecting optimal combinations of multiple LLMs in the context of online user queries.  The works of \citet{karimi2021online}, \citet{kay2025consensus} and \citet{liu2024contextual} address the problem of online model selection from a different perspective, namely how to minimize the number of label queries needed to find the best-performing model. In our work, we focus on a different variant of model selection to the settings in the papers described above, in that we add a fairness dimension to the model selection problem, using a general MAB setting with querying costs and reward delays, with theoretical regret bounds to support our approach.

\section{Problem Statement}

Let ${\mathcal{A}} = \{1,\cdots, N\}$ be an index set of agents. Given an input $X$ each agent $a$ can produce a prediction of a class label,  $\hat{y}_a$.
Let $c_a \geq 0$ be the fixed deterministic {\it query cost}  incurred to obtain a prediction label from agent $a$.  At each round, $t=1,\cdots,T$, a triplet $(X_{t}, S_{t}, Y_{t})$ is drawn IID from an unknown probability distribution over $\mathcal{X} \times \mathcal{S} \times \mathcal{Y}$ where $\mathcal{X}$ is the feature space, $\mathcal{S}$ is the set of sensitive attributes, and $\mathcal{Y}$ is the set of possible outcomes. For now, we focus on two subgroups and binary classification, i.e., $\mathcal{S} = \{0, 1\}$ and $\mathcal{Y} = \{0, 1\}$--we show that the methods and theory we develop can be extended straightforwardly to more than two subgroups and to multi-class problems using appropriate definitions of fairness \citep{shui2022learning, xian2023fair}. We note that the sensitive attributes are distributed as $\mathbb{P}(S_{t} = 0) = \pi_{0}$ and $\mathbb{P}(S_{t} = 1) = \pi_{1}$ where $\pi_{0} + \pi_{1} = 1$ and $\pi_{\min} = \min\{\pi_{0}, \pi_{1}\} > 0$. We assume that this does not change over $t$.

We formulate online selection of agents as a multi-armed bandit problem where each arm represents an ML classifier or human expert, and the goal is to identify the agent that optimally balances classification error, fairness disparity, and cost.

Formally, at round $t$, the policy learner observes both the classification features $X_t$ and the sensitive attribute value $S_t$.  The learner
must choose an arm $A_{t}$ from the set of available arms at time $t$,  $\mathcal{A}_t$, informed by the history of past actions and outcomes. The prediction $\hat{y}_{A_t}$ for the selected arm is then observed with querying cost $c_{A_t}$. Although features are observed, the learner's selection policy is non-contextual.

The true label $Y_t$ is revealed at time $t+L_t$ where $L_t \in \mathbb{Z}_{\geq 0}$
is assumed to be an IID random variable representing a nonnegative integer delay, reflecting real-world scenarios (such as medical image classification, loan approval, etc) where the true $Y$ might not be known until some time in the future after the prediction. This type of delayed reward has been analyzed in the past in the MAB literature, resulting in regret bounds that scale favorably over time horizons, subject to reasonable conditions on delay distributions \citep{wu2022thompson}.

The accuracy of each arm $a$ (measured marginally over $(X, Y)$) is
    \[\rho_a = \Prb\big(\hat{y}_a(X)=Y\big).\]
The highest accuracy (over all arms/agents) is $\rho_\star=\max_a \rho_{a}$.

Fairness metrics can be expressed as functions of subgroup and label-specific conditional probabilities $\mu_{a,s,y}$. The resulting fairness gap (usually a difference or ratio) is a function $g_a = f(\mu_{a,s,y})$ of these conditional probabilities. Let the fairest arm be $\arg\min_{a\in\mathcal{A}} g_a$ and denote $g_\star=\min_{a\in\mathcal{A}} g_a$.

As an example, for demographic-parity (DP) gap, we may define the subgroup-positive rate $\mu_{a,s}=\Prb\big(\hat{y}_a(X)=1 \mid S=s\big)$ for each arm $a \in \mathcal{A}$ and subgroup $s\in\{0,1\}$. The resulting DP gap of arm $a$ is $g_a = |\mu_{a,0}-\mu_{a,1}|$.

Given the framework above, the utility per agent trades off accuracy, fairness, and cost and is defined as
\begin{equation}\label{eqn:utility}
U_a(\lambda) := \rho_a - \lambda\, g_a - c_a,
      \qquad \lambda \geq 0.
\end{equation}
 Here, $\lambda$ is a user-defined (constant) weight that a decision-maker assigns to indicate the relative weight of the fairness gap compared to classification accuracy, with larger positive values of $\lambda$ indicating greater weight on fairness over accuracy. In the context of algorithmic fairness, the weight $\lambda$ represents a \textit{normative} policy preference (how much accuracy the stakeholder is willing to sacrifice for fairness), rather than a latent statistical parameter. This functional form is similar to that used in learning fair agents \citep{kamishima2011fairness, menon2018cost}. The querying cost $c_a$ is assumed to be expressed in the same units as accuracy $\rho_a$, e.g., in some utility units. Our formulation follows the standard decision theoretic framework (\citet[pp.~176--180]{murphy2012machine}, \citet[p.~41]{bishop2006pattern}) where accuracy is the expected utility of making a correct decision and the cost is normalized relative to the value or utility of a correct prediction.

Let $U_\star(\lambda)=\max_a U_a(\lambda)$, i.e., the utility of the optimal arm $a^\star$. The goal of the learner is to select $A_t$ such that the selected arm's utility is as close as possible to the optimum $U_\star$ up to time $T$. This performance measure is formalized as the cumulative regret of the learner - a quantity that the learner aims to minimize
\begin{equation}\label{eqn:regret}
R_\lambda(T) = \sum_{t=1}^T \big(U_\star(\lambda)-U_{A_t}(\lambda)\big).
\end{equation}
This decomposes linearly into \emph{accuracy}, \emph{fairness}, and \emph{cost-}regret terms:
\begin{equation*}
    R_\lambda(T) \;=\; \underbrace{\sum_{t=1}^T\big(\rho_{a^\star}-\rho_{A_t}\big)}_{\text{accuracy regret}}
    \;+\; \lambda\underbrace{\sum_{t=1}^T\big(g_{A_t}-g_{a^\star}\big)}_{\text{fairness regret}}
    \;+\; \underbrace{\sum_{t=1}^T\big(c_{A_t} - c_{a^\star}\big)}_{\text{cost regret}}.
\end{equation*}

For clarity, we note that the accuracy term in this decomposition is formulated as the performance deficit relative to the optimal arm (i.e.  $\rho_{a^\star}-\rho_{A_t}$). This definition effectively flips the sign relative to the utility function to ensure that lower values indicate better performance, consistent with the standard definition of regret. For fairness and cost, the decomposition remains $g_{A_t}-g_{a^\star}$ and $c_{A_t}-c_{a^\star}$ respectively.

\section{Technical Approach: Bayesian Bandits}
\label{sec:proposed_algo}

\begin{algorithm}[htbp]
\caption{Thompson Sampling for Online Selection}
\label{alg:ts}
\algtext*{EndFor}%
\begin{algorithmic}[1]
    \Require Arms $\mathcal{A}$, horizon $T$, trade-off $\lambda\ge0$, per-arm costs $c_a\ge0$, priors
    \Statex \hskip\algorithmicindent $\tilde\mu_{a,s}\sim\mathrm{Beta}(\alpha_{a,s},\beta_{a,s})$ for $(a,s)\in\mathcal{A}\times\{0,1\}$\footnotemark[1],
    \Statex \hskip\algorithmicindent $\tilde\rho_a\sim\mathrm{Beta}(\alpha_a,\beta_a)$ for $a\in\mathcal{A}$.
    \For{$t=1,\dots,T$}
        \ForAll{$a\in\mathcal A_t$}
            \State Draw $\tilde\mu_{a,0}\sim\mathrm{Beta}(\alpha_{a,0},\beta_{a,0})$, \ $\tilde\mu_{a,1}\sim\mathrm{Beta}(\alpha_{a,1},\beta_{a,1})$.\label{alg:ts:draw-mu}
            \State Draw $\tilde\rho_a\sim\mathrm{Beta}(\alpha_a,\beta_a)$.\label{alg:ts:draw-rho}
            \State $\tilde g_a \leftarrow |\tilde\mu_{a,0}-\tilde\mu_{a,1}|$\footnotemark[1]
            \State $U_a \leftarrow \tilde\rho_a - \lambda\tilde g_a - c_a$.
        \EndFor
        \State Select $A_t\in\arg\max_{a\in\mathcal A_t}U_a$ (ties arbitrary).
        \State Observe $S_t$; set $r_{\text{fair}} = \1\{\hat{y}_{A_t}(X_t) = 1\}$\footnotemark[1]
        \State Update fairness posterior: $\alpha_{A_t,S_t}\!\leftarrow\!\alpha_{A_t,S_t}+r_{\text{fair}}, \ \beta_{A_t,S_t}\!\leftarrow\!\beta_{A_t,S_t}+(1-r_{\text{fair}})$.
        \ForAll{$\tau\le t$ whose label $Y_\tau$ arrives at round $t$}
            \State Set $r_{\text{acc}} = \1\{\hat{y}_{A_\tau}(X_\tau) = Y_\tau\}$.
            \State Update accuracy posterior: $\alpha_{A_\tau}\!\leftarrow\!\alpha_{A_\tau} + r_{\text{acc}}$, \ $\beta_{A_\tau}\!\leftarrow\!\beta_{A_\tau} + (1-r_{\text{acc}})$.
        \EndFor
    \EndFor
\end{algorithmic}

    \rule{0.2\linewidth}{0.1pt} \\
    {\tiny $^{1}$For DP gap, we maintain subgroup positive-rate. $\tilde g_a$ and $r_{\text{fair}}$ follow.}
\end{algorithm}

Formulating online selection as an MAB  problem allows us to use any standard MAB algorithm, including but not limited to UCB-style optimistic methods \citep{auer2002finite}, ($\epsilon$-)greedy heuristics \citep{sutton1998reinforcement}, and Thompson Sampling \citep{russo2018tutorial, russo2014learning}. In this work, we focus on Thompson Sampling (TS): the Bayesian posterior view makes it straightforward to maintain and systematically update uncertainty beliefs. The TS approach handles feedback (immediate or delayed) naturally via conjugate updates, and posterior sampling results in an intuitive mechanism that naturally balances exploration and exploitation. The seminal work of \citet{chapelle2011empirical} also demonstrated highly competitive empirical results for TS in traditional MAB settings, such as online advertisement and recommendation systems.

Our work differs from standard work on MABs in that each arm induces two latent “sub-arms” corresponding to sensitive subgroups $S$. Algorithmically, we must maintain and update separate posteriors for each subgroup as well as marginal accuracy estimates. Analytically, this increases the number of parameters to learn and introduces a dependence on subgroup prevalence: regret bounds therefore decompose into accuracy and fairness components and depend on the smallest subgroup probability (and on label frequencies for conditional fairness notions); our theoretical guarantees presented in Section \ref{sec:theory} reflect this structure. In Algorithm \ref{alg:ts}, we formalize Thompson Sampling for online selection of classifiers, balancing accuracy, fairness, and query cost with potential delays in true label observation. We consider a broader class of model selection tasks and frameworks:

\paragraph{Other fairness metrics}
In addition to TS for DP-gap, we can use other fairness metrics for this algorithm. For metrics such as Equalized Odds and Equality of Opportunity, we would maintain posterior beliefs $\tilde\mu_{a,s,y}\sim\mathrm{Beta}(\alpha_{a,s,y},\beta_{a,s,y})$ for $(a,s,y)\in\mathcal{A}\times\{0,1\}\times\{0,1\}$ that include the true label $Y$. Both $\tilde g_a$ and $r_{\text{fair}}$ change accordingly, and the fairness update waits for $Y$.

\paragraph{Multi-class classification ($|\mathcal{Y}| = C > 2$) and Multi-subgroup sensitive attribute ($|\mathcal{S}| = M > 2$)}

We continue to use Beta-Bernoulli posteriors for the accuracy estimate since the reward is still binary (assuming symmetric utilities for different misclassifications). To capture disparities in class predictions (e.g., unequal prediction rates across classes conditioned on sensitive attributes), we replace the subgroup Beta posteriors with $M$ Dirichlet posteriors (under a Dirichlet-Multinomial model), parameterizing the probability vector over the $C$ classes for each subgroup. We sample probability vectors $\tilde{\mu}_{a,s}$ for all subgroups and use the maximum/average pairwise $L_1$ distance computed on these samples to capture disparity \citep{xian2023fair, barocas2023fairness, agarwal2018reductions}.

\paragraph{Real-valued reward }
The framework extends to bounded real-valued rewards, such as the Brier score of a probabilistic classifier or LLM evaluation scores
\citep{liang2023holistic}. Suppose arm $a$ outputs a score $p_a\in[0,1]$ instead of a label. With reward $1-(Y-p_a)^2$, $\rho_a$ is one minus the arm's Brier score, and the fairness gap $g_a$ is the DP gap of the subgroup mean scores $\mu_{a,s}=\mathbb{E}[p_a\mid S=s]$. For 0/1 scores these reduce to accuracy and the DP gap. If the scores are rounded to a grid $\{0,1/K,\dots,1\}$, scores and rewards take finitely many values, and each Beta posterior is replaced by a Dirichlet posterior over these values (conjugate to the categorical likelihood). With one pseudo-count at 0 and one at 1, the prior is $\mathrm{Beta}(1,1)$ for binary data. TS samples Dirichlet weights $w$ and uses the weighted mean $\sum_k w_k v_k$ of the values $v_k$ as $\tilde\rho_a$ or $\tilde\mu_{a,s}$. Up to the prior, this is non-parametric TS \citep{riou2020bandit} with tied observations grouped.

 \section{Theoretical Results}
\label{sec:theory}

We investigate the regret bounds of our Thompson sampling algorithm (Algorithm \ref{alg:ts}). Let $\Delta_{a} = U_{\star}(\lambda) - U_{a}(\lambda)$ be the utility gap and $a^{\star} \in \text{argmax}_{a}U_{a}(\lambda)$. Recall that the cumulative regret is defined as in Equation \eqref{eqn:regret}.
\begin{theorem}[Regret Bound of Thompson Sampling with Fairness]\label{thm:bound}
For the Thompson sampling described in Section~\ref{sec:proposed_algo}, with no delays and arbitrary (deterministic) querying costs, for a universal constant $C > 0$, we have the gap-dependent regret:
\begin{equation}\label{eqn:gap-dep}
\begin{aligned}
\small
    \mathbb{E}[R_{\lambda}(T)] \leq C & \sum_{a:\Delta_a(\lambda)>0}\frac{1}{\Delta_a(\lambda)} \Bigg[ \log T +\frac{\lambda^2}{\pi_{\min}}\log T\\ &+\left( \frac{\lambda}{\sqrt{\pi_{\min}}} +\frac{\lambda^2}{\pi_{\min}} \right)(\log T)^{7/2} \Bigg] .
\end{aligned}
\end{equation}
Moreover, the gap-free bound is
\begin{equation}\label{eqn:gap-free}
\begin{aligned}
\small
    \mathbb{E}[R_{\lambda}(T)]  \leq &O\Bigg(  \sqrt{NT \log T} + \frac{\lambda}{\sqrt{\pi_{\min}}} \sqrt{NT \log T} \\
    & + \sqrt{\bigg( \frac{\lambda}{\sqrt{\pi_{\min}}} + \frac{\lambda^{2}}{\pi_{\min}} \bigg) T (\log T)^{7/2}}  \Bigg).
\end{aligned}
\end{equation}
When $N \geq (\log T)^{5/2}$, the third term of \eqref{eqn:gap-free}, the additional optimal-arm anti-concentration cost, is dominated by the first two and the gap-free bound reduces to $O\big(\sqrt{NT \log T} + \frac{\lambda}{\sqrt{\pi_{\min}}} \sqrt{NT \log T}\big)$.
\end{theorem}
Without fairness, Thompson sampling has the bound $O (\sqrt{N T \log T})$ \citep{agrawal2017near}. Our bounds add two terms to this accuracy cost: the fairness-estimation cost and the optimal-arm anti-concentration cost. The $\pi_{\min}$ factors quantify the statistical cost of estimating the demographic parity gap across subgroups: rarer groups encourage more exploration. However, incorporating fairness preserves the canonical gap-dependent and gap-free regret rates in $T$. The $\lambda$ factors reflect the sensitivity of the utility to fairness-gap estimation error: heavier fairness weighting amplifies the impact of subgroup-rate uncertainty on arm misordering. The optimal-arm anti-concentration cost arises because noisy subgroup draws can inflate the optimal arm's sampled fairness gap even when its accuracy draw is optimistic, which delays its selection. This cost vanishes when $\lambda = 0$ and does not grow with $N$.

Here we provide a sketch of the proof: full proof in Appendix \ref{app:regret}. We follow the regret bound analysis of \citet{agrawal2017near} to the accuracy–fairness-cost utility. For any suboptimal arm, we consider rounds in which (I) the optimal arm's sampled utility is too low, (II) the suboptimal arm's posterior draw is overly optimistic, or (III) its posterior-mean utility is overly optimistic. The selected arm does not depend on the group of the incoming sample, each observation of an arm comes from group $s$ with probability $\pi_{s}$. Thus, the first $n$ observations of any arm contain about $\pi_{s} n$ observations from group $s$, and the smallest group contributes about $\pi_{\min} n$. Concentration of both accuracy and subgroup rates then controls (II) and (III). Known querying costs enter the utility thresholds and gaps as fixed offsets.

Case (I) requires an additional argument. The exchange step of \citet{agrawal2017near} still applies, but their scalar Beta-Binomial anti-concentration time bound does not carry over as the optimal arm's sampled utility combines an accuracy draw with the absolute difference of two subgroup draws. Moreover, the accuracy and subgroup estimates of the optimal arm can be dependent because they are computed from the same observation. Thus, we bound this joint anti-concentration cost using a positive utility margin. We then group arms by utility gap so that anti-concentration is charged once per group in the gap-free bound. At $\lambda = 0$, the joint anti-concentration step disappears and the scalar accuracy argument reduces to the standard rate.

\begin{theorem}[Regret Bound of Thompson Sampling with Fairness under Delayed Reward]\label{thm:delay}
Consider Algorithm \ref{alg:ts} under the assumptions of Theorem \ref{thm:bound}. For each arm, assume that label delays are IID. Fairness posteriors are updated immediately, whereas accuracy posteriors are updated when labels arrive. Choose $q_{a} \in (0, 1]$ such that $d_a(q_a) < \infty$, and let $q_{\min} = \min_{a} q_{a}$. Then, the gap-free regret bound becomes:
\begin{equation}\nonumber
\begin{aligned}
    \mathbb E[R_\lambda(T)]
    \leq O\Bigg(
    &\sqrt{\frac{NT\log T}{q_{\min}}}
    +
    \frac{\lambda}{\sqrt{\pi_{\min}}}
    \sqrt{NT\log T}
    \\
    &
    + (1+\lambda)\sum_{a=1}^{N}d_a(q_a) + \sqrt{T H_D(T)}
    \Bigg),
\end{aligned}
\end{equation}
where $H_D(T)$ summarizes optimal-arm anti-concentration caused by pessimistic utility draws under delayed feedback. Its formal definition is given in Appendix \ref{app:delay-proof}.
\end{theorem}

As accuracy updates require the true labels, their estimation cost is inflated by $q_{\min}^{-1/2}$. Demographic parity feedback is available immediately, so the fairness estimation term scales as $\lambda/\sqrt{\pi_{\min}}$ without an additional delay factor. The quantile argument also produces the additive contribution $(1+\lambda)\sum_{a=1}^{N}d_a(q_a)$. The remaining term, $\sqrt{T H_D(T)}$, covers the optimal arm's anti-concentration, both the accuracy-based part and the part due to joint accuracy-fairness uncertainty. \emph{Proof sketch} (full proof in Appendix~\ref{app:delay-proof}). We combine the delay-quantile counting of \citet{wu2022thompson} with the posterior concentration and exchange arguments of Theorem \ref{thm:bound}. Concentration controls overly optimistic draws of a suboptimal arm, using the arrived labels for accuracy and the predictions observed at selection for fairness. The exchange argument reduces the remaining pulls to the anti-concentration cost of the optimal arm, $H_D(T)$, which we charge once per utility-gap group, and the choice of the gap cutoff yields the stated bound. Without delays, the anti-concentration bound in Theorem \ref{thm:bound} controls $H_D(T)$ and recovers its regret rate.

\section{Experiments}

\subsection{Datasets}
\begin{table}[h]
\centering
{%
\begin{tabular}{lcccc}
\toprule
\textbf{Dataset} & Type of Arms ($N$) & $C$ & Sensitive Attribute ($M$) & $\pi_{\min}$ \\
\midrule
Bank & ML classifiers (18) & 2 & Marital Status (2) & 0.396 \\
Adult & ML classifiers (18) & 2 & Race (2) & 0.146 \\
Credit & ML classifiers (18) & 2 & Sex (2) & 0.392 \\
ACSIncome & ML classifiers (18) & 2 & Race (5) & 0.036 \\
Toxicity & LLMs (56) & 2 & Race (2) & 0.500 \\
Bios & LLMs (16) & 28 & Gender (2) & 0.461 \\
Twitter & Human Experts (10) & 2 & Dialect (2) & 0.494 \\
\bottomrule
\end{tabular}}
\caption{Datasets: $N$ arms, $C$ classes and $M$ subgroups}
\label{tab:datasets}
\end{table}

We conduct experiments on real-world datasets involving classification problems that have a sensitive attribute (see Table \ref{tab:datasets}).
For {\bf ML Classifiers as Arms} we use 4 tabular datasets where the arms are pre-trained classifiers. For each dataset we train models on 60\% of the data and use the remaining 40\% to create an online stream. On every dataset we train the same 18 classifiers. They range from standard models (e.g., Gradient Boosted Trees and Logistic Regression), to models trained with a fairness penalty and deliberately biased models; most do not use the sensitive attribute (details in Appendix~\ref{app:datasets}). The first dataset is Bank Marketing ({\bf Bank}) \citep{bank_marketing_222}, where the classification goal is to predict if the client will subscribe to a bank term deposit, with 19 features and where marital status is the sensitive attribute. {For the {\bf Adult} dataset \citep{adult_2}, the goal is to predict whether a person's income exceeds \$50K, with 11 features and where race (white vs.\ non-white) is the sensitive attribute.
The {\bf Credit} dataset \citep{yeh2009comparisons}  is a credit default classification task, with 22 features, and customer sex is the sensitive attribute.
For a {\bf Multi-subgroup Sensitive Attribute} ($M>2$) we use the updated ({\bf ACSIncome}) \citep{ding2021retiring}, with the same prediction task as Adult, using 9 features. The sensitive attribute is race/ethnicity with five groups
and the DP gap is the largest gap between any two groups.}
Additional experiments with the Coupon dataset \citep{wang2017bayesian} and COMPAS \citep{Larson_Angwin_Kirchner_Mattu_2016} are included in Appendix~\ref{app:datasets}.

For {\bf LLMs as Arms} we use the CivilComments Toxicity detection dataset ({\bf Toxicity}) \citep{borkan2019nuanced}. For this dataset, the arms correspond to 56 LLMs accessed through API queried on 2000 prompts, where the classification task is whether a passage is toxic or not \citep{liang2023holistic}. Race is the sensitive attribute. {We include a {\bf Multi-class Classification} ($C>2$) setup with LLMs, using the   Bias in Bios ({\bf Bios}) dataset \citep{de2019bias}, where the goal is to predict a person's occupation (28 classes) from their online biography. Gender is the sensitive attribute. The arms are 16 open instruction-tuned LLMs (0.36B--9B parameters) queried zero-shot on 30{,}000 biographies.}

For {\bf Humans as Arms} we use the TwitterAAE dataset (\textbf{Twitter}) \citep{keswani2021towards} with Tweets annotated using Mechanical Turk
for whether they contain hate-speech or not \citep{davidson2017automated}
The sensitive attribute is whether the post contains African-American English
\citep{blodgett2017dataset}. The original dataset contained annotations from 170 experts, each covering only ~4\% of the data. To reduce sparsity in the response, we grouped experts into 10 arms by randomly combining 18 experts per arm and using majority voting, resulting in each arm classifying $\approx$60\% of the examples.

\subsection{Baseline Methods}

\noindent\textbf{Round-Robin} The learner cycles deterministically through available arms, guaranteeing even coverage and predictable scheduling; it yields unbiased exploration but no exploitation of observed rewards. It is equivalent, in expectation, to a random policy.

\noindent\textbf{$\epsilon$-Greedy} Under the same Bayesian setup as Thompson Sampling, we consider a greedy strategy that selects the arm with the largest posterior mean.
For the subgroup posterior estimate (Algorithm \ref{alg:ts}, Step~\ref{alg:ts:draw-mu}) $\tilde\mu_{a,s} = \mathbb{E}[\mathrm{Beta}(\alpha_{a,s},\beta_{a,s})]$ for $(a,s)\in\mathcal{A}\times\{0,1\}$, and for accuracy (Algorithm \ref{alg:ts}, Step~\ref{alg:ts:draw-rho})  $\tilde\rho_a = \mathbb{E}[\mathrm{Beta}(\alpha_a,\beta_a)]$ for $a\in\mathcal{A}$. At each round, with probability $1-\epsilon$ the arm with the largest mean is selected, and (to encourage exploration) with probability $\epsilon$ an arm is selected uniformly at random.

We compare against \textbf{UCB1} where the learner selects the arm with the largest optimistic utility \citep{auer2002finite} and \textbf{Bayes-UCB} \citep{kaufmann2012bayesian}, where, under the same Bayesian setup as TS, posterior samples are replaced with posterior quantiles at level $1-1/t$. Additional implementation details in Appendix \ref{app:UCB}.

\subsection{Methodology}\label{sec:methodology}

We average results over 10{,}000 runs. Each run is a stream of $T=50{,}000$ instances sampled with replacement from the test dataset (results with fewer rounds are in Appendix \ref{app:fewer}), and all algorithms and values of $\lambda$ see the same stream. We report the mean regret per example, $R_\lambda(T)/T$ in \%, computed from the arms' true accuracies and gaps on the test dataset. On Twitter, regret is measured against the best available arm in each round. Bayesian methods are always initialized with uninformative priors $\mathrm{Beta}(1,1)$ or $\mathrm{Dirichlet}(1,\dots,1)$.%
For the experiments, the fairness gap is defined as the demographic parity gap (for Bios, half the $L_1$ distance between the subgroups' class distributions; for ACSIncome, the largest gap between any two subgroups) and traded off with the utility score (Equation \ref{eqn:utility}). $\lambda=0$ represents accuracy-only selection ($U_a= \rho_a - c_a$) and $\lambda=\infty$ fairness-only selection ($U_a = -g_a - c_a$). We use 29--38 values of $\lambda$ per dataset (Table~\ref{tab:wins}), log-spaced and denser near the values at which the best arm changes.
For all  results in the main paper, the query cost is the same for all arms (so the cost term drops out of the utility) and the delay $L_t$ is zero. Results with non-zero costs and with delays are in Appendices \ref{app:cost} and \ref{app:delay}.

\subsection{Results}

In Table \ref{tab:endpoints} we compare Thompson Sampling and the 5 baselines across the 7 datasets in Table~\ref{tab:datasets} at the two ``endpoint'' special cases of  (a) $\lambda=0$ (where regret only depends on model accuracy) and (b) $\lambda = \infty$ (where regret only depends on fairness). Thompson sampling has the smallest regret in 12 of the 14 cases, and in the other two it is within 0.1 of the best.

\begin{table}[htbp]
\centering
{%
\setlength{\tabcolsep}{3pt}%
{%
\begin{tabular}{@{}lccccccc@{}}
\toprule
\textbf{Algorithm} & Twitter & Bank & Adult & Credit & Toxicity & Bios & ACSIncome \\
\midrule
\multicolumn{8}{@{}l}{\textbf{(a) Accuracy} ($\lambda=0$)} \\
\midrule
Round-Robin & 4.141 & 3.494 & 3.521 & 3.155 & 16.73 & 17.80 & 7.115 \\
UCB1 & 1.924 & 0.819 & 1.544 & 0.975 & 4.930 & 3.344 & 2.476 \\
Bayes-UCB & 0.864 & 0.563 & 0.679 & 0.615 & 2.229 & 0.669 & 0.845 \\
$\epsilon$-Greedy (0) & 1.200 & 0.836 & 2.505 & 1.109 & 4.151 & 8.327 & 5.255 \\
$\epsilon$-Greedy (0.1) & 0.915 & 0.634 & 0.552 & \textbf{0.530} & 2.406 & 2.236 & 1.001 \\
TS & \textbf{0.770} & \textbf{0.507} & \textbf{0.548} & 0.532 & \textbf{1.948} & \textbf{0.523} & \textbf{0.676} \\
\midrule
\multicolumn{8}{@{}l}{\textbf{(b) Fairness} ($\lambda=\infty$)} \\
\midrule
Round-Robin & 3.706 & 5.612 & 9.266 & 4.956 & 12.89 & 10.44 & 21.37 \\
UCB1 & 3.739 & 2.344 & 6.472 & 1.612 & 11.80 & 10.44 & 15.12 \\
Bayes-UCB & 2.619 & 1.627 & 5.181 & 0.764 & 6.902 & 7.820 & 3.136 \\
$\epsilon$-Greedy (0) & 2.010 & 1.624 & 5.335 & 0.815 & 1.874 & 2.860 & 8.378 \\
$\epsilon$-Greedy (0.1) & 1.715 & 1.123 & 2.435 & 0.653 & 1.453 & \textbf{2.039} & 5.151 \\
TS & \textbf{1.367} & \textbf{1.023} & \textbf{2.416} & \textbf{0.179} & \textbf{0.548} & 2.109 & \textbf{1.495} \\
\bottomrule
\end{tabular}
}}
\caption{Regret per example (\%) after $T=50$k rounds for each algorithm (rows) and dataset (columns), for (a) accuracy ($\lambda=0$) and (b) fairness ($\lambda=\infty$), averaged over 10{,}000 runs (the endpoints of the curves in Figure~\ref{fig:tradeoff-50k-no-delay}). Minimum regret bolded.}
\label{tab:endpoints}
\end{table}

\begin{figure*}[htbp]
    \centering
    \includegraphics[width=\textwidth]{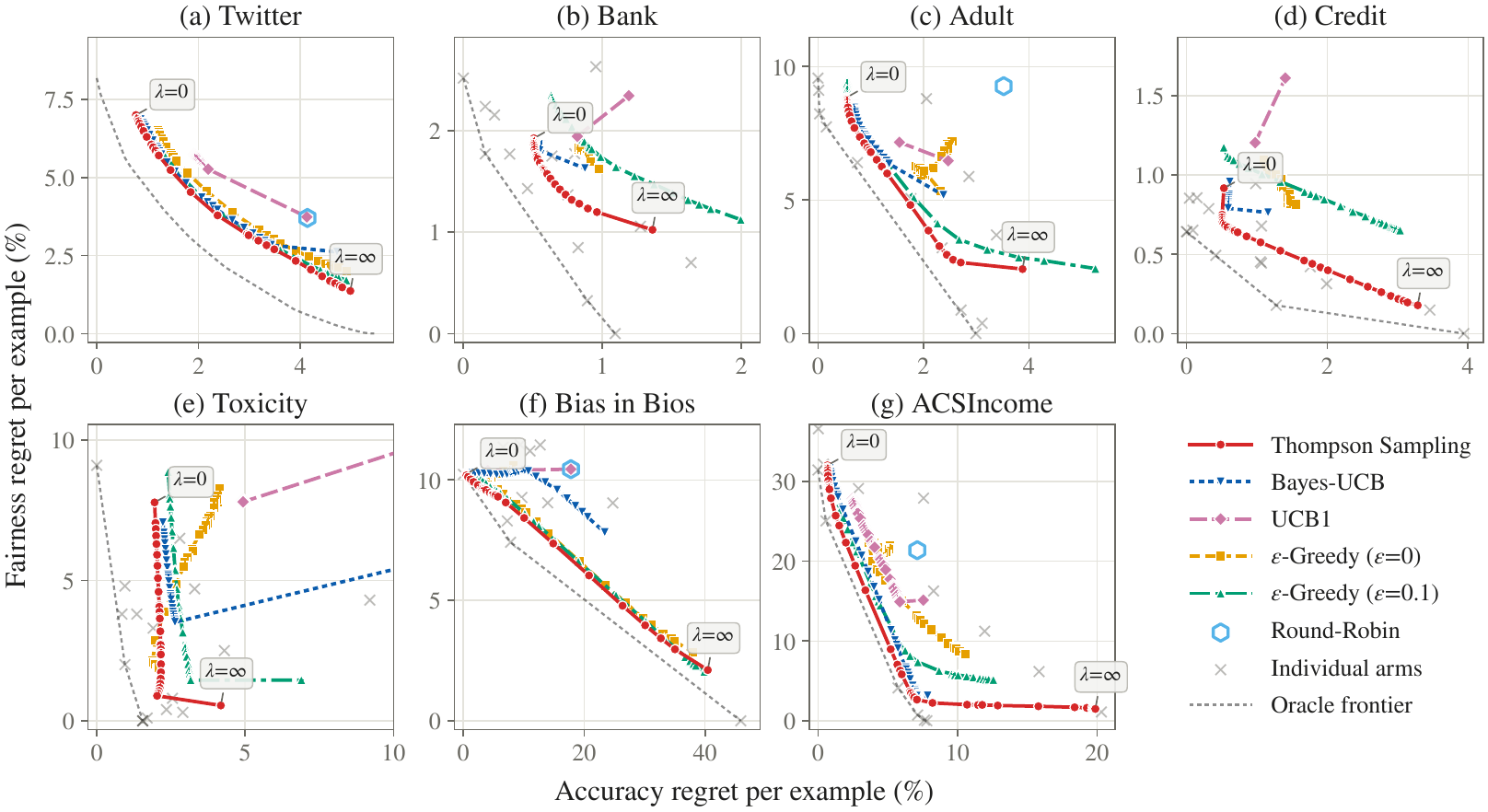}
    \caption{Trade-off between accuracy regret (x-axis) and fairness regret (y-axis) after $T=50$k rounds, as average regret per example (\%) over 10{,}000 runs. Each point on a curve is one value of $\lambda$, from $\lambda=0$ (all weight on accuracy) to $\lambda=\infty$ (all weight on fairness). Datasets (a)--(e) are binary, (f) is multi-class and (g) is multi-subgroup. Grey crosses are the individual arms and the dotted line is the oracle frontier. Round-Robin does not depend on $\lambda$ and is omitted from (b), (d) and (e), where it is off-scale.}
    \label{fig:tradeoff-50k-no-delay}
\end{figure*}

Figure \ref{fig:tradeoff-50k-no-delay} shows the trade-off for $\lambda$ values between these two endpoints, indicating the cumulative accuracy regret (x-axis) and cumulative fairness regret (y-axis) for that $\lambda$ value as a percentage of total rounds, averaged over 10{,}000 runs each of $T=50$k rounds. The endpoints of each curve are the values in Table~\ref{tab:endpoints}. Additional details are provided in Appendix \ref{app:datasets}.

The non-MAB method (Round-Robin) is clearly not competitive with the MAB methods. In fact for the Bank, Credit and Toxicity datasets, Round-Robin is off-the-scale to the upper-right. Among the MAB methods, Thompson Sampling (TS) is closest to the oracle frontier on every dataset (in some cases by a large margin), with occasional exceptions for some values of $\lambda$ (which is not surprising given that none of these methods are provably dominant empirically for all MAB problems and their relative performance is dataset-dependent \citep{russo2018tutorial}). The UCB methods are far from the frontier at large $\lambda$, because their lower bound on the DP gap stays at zero until the subgroup confidence intervals separate. The overall picture from Figure \ref{fig:tradeoff-50k-no-delay} is that (i) MAB methods significantly outperform non-MAB methods, and (ii) TS is more reliable in achieving good solutions than the $\epsilon$-Greedy and UCB methods.

Table~\ref{tab:wins} extends the comparison of Table~\ref{tab:endpoints} to every value of $\lambda$. It summarizes Figure~\ref{fig:tradeoff-50k-no-delay} in tabular form how often each algorithm has the best point. TS has the lowest regret at all values on four datasets and at all but one on Credit. On ACSIncome, Bayes-UCB is lower where two arms have nearly equal utility. On Bios, $\epsilon$-Greedy ($\epsilon=0.1$) is lower at large $\lambda$: with 28 classes, the flat Dirichlet prior inflates the sampled gap of a rarely selected arm, so TS can under-explore the fairest LLM.

\begin{table}[htbp]
\centering
{%
\setlength{\tabcolsep}{3pt}%
{%
\begin{tabular}{@{}lccccccc@{}}
\toprule
\textbf{Algorithm} & Twitter & Bank & Adult & Credit & Toxicity & Bios & ACSIncome \\
\midrule
Round-Robin & 0 & 0 & 0 & 0 & 0 & 0 & 0 \\
UCB1 & 0 & 0 & 0 & 0 & 0 & 0 & 0 \\
Bayes-UCB & 0 & 0 & 0 & 0 & 0 & 0 & 18 \\
$\epsilon$-Greedy (0) & 0 & 0 & 0 & 0 & 0 & 0 & 0 \\
$\epsilon$-Greedy (0.1) & 0 & 0 & 0 & 3 & 0 & 38 & 0 \\
TS & \textbf{100} & \textbf{100} & \textbf{100} & \textbf{97} & \textbf{100} & \textbf{62} & \textbf{82} \\
\bottomrule
\end{tabular}
}}
\caption{\% of $\lambda$ values at which each algorithm has the lowest average cumulative regret after $T=50$k rounds.}
\label{tab:wins}
\end{table}

We report further experiments with real-valued rewards (Appendix~\ref{app:real}), multi-class classification on a second dataset (Appendix~\ref{app:multi-class}), fewer rounds (Appendix~\ref{app:fewer}), fewer arms (Appendix~\ref{app:fewer-arms}), delayed feedback (Appendix~\ref{app:delay}), query costs (Appendix~\ref{app:cost}) and smaller values of $\pi_{\min}$ (Appendix~\ref{app:pi_min_experiment}). TS is consistently the best-performing method across scenarios, with the lowest regret at most values of $\lambda$. The main exceptions are short horizons ($T\leq10$k) and long delays.

\section{Discussion and Conclusion}
\label{sec:discussion}
In terms of {\bf limitations}, our theoretical proofs focus on binary classifiers, non-contextual bandits, and difference-based fairness metrics for binary sensitive attributes. Nonetheless, we believe our results represent an important step in model selection that balances accuracy, fairness, and cost, providing a basis for a variety of directions for {\bf future work} in terms of theory, algorithms, and experiments.
An obvious direction is model selection via \textit{contextual bandits} \citep{chapelle2011empirical}, where the selected arm $A_t$ is dependent on the instance $X_t$,
e.g., modeling the expected reward via GLMs conditioned on features $X_t$ and utilizing Laplace approximations to maintain a tractable Gaussian posterior over the parameters. This could result in policies that   allocate tasks to arms that are better for specific regions of the context space.
Additional potential directions for future research include: extensions to time-to-event models, multi-play algorithms to provide additional exploration of small subgroups (small $\pi_{\min}$), and leveraging probabilistic information from ML models.
The algorithmic template would remain largely unchanged for many of these extensions: maintain per-arm/per-subgroup posteriors for the quantities that define utility, sample those posteriors at decision time, and select based on a utility that trades-off accuracy, fairness, and cost; only the choice of conjugate family and the disparity functional differ.

\clearpage

\bibliographystyle{abbrvnat}
\bibliography{references}

\clearpage
\appendix

\section{Regret Bound Proofs for Thompson Sampling}
\label{app:regret}

We assume the standard setting where outcomes and sensitive attributes are independent across arms and IID over time. The utility $U_{a}(\lambda)$ depends on the underlying Bernoulli means through the accuracy $\rho_a$ and the fairness gap $g_a = |\mu_{a,0} - \mu_{a,1}|$. Since the absolute difference is a 1-Lipschitz function, the overall utility is Lipschitz continuous with respect to the parameters. This allows us to apply standard posterior concentration results (e.g., \citet{agrawal2017near}).

\subsection{Regret Bound of the Thompson Sampling with No Delays (Theorem \ref{thm:bound})}

We follow the proof structure of \citet{agrawal2017near}. We first bound the expected number of selections of each suboptimal arm, then derive the gap-free regret bound. The high-level picture is:
\begin{enumerate}
    \item Decompose selections on the utility scale: we define utility thresholds and split the selections of a suboptimal arm into three cases: (I) an ordinary competitor beats a pessimistic draw of the optimal arm, (II) the competitor's posterior draw is overly optimistic, or (III) its empirical estimate is overly optimistic.
    \item Bound the cases in the order of \citet{agrawal2017near}: we use the exchange step and the optimal-arm observation epochs for (I), empirical concentration for (III), and then conditional Beta concentration for (II). (II) and (III) can be addressed the same way as in \citet{agrawal2017near}; however (I) needs a new optimism estimate.
    \item Combine and group by gap: adding three cases proves the selection-count theorem. For the gap-free regret, we charge small gaps directly and charge the anti-concentration once per group of larger gaps. This keeps the additional anti-concentration contribution free of a factor $\sqrt{N}$.
\end{enumerate}

We note that, at $\lambda=0$, the joint anti-concentration step reduces to the scalar argument of \citet{agrawal2017near}.

We consider $N$ fixed arms that are always available. Cases are IID, the policy is noncontextual, and feedback is immediate. Algorithm \ref{alg:ts} keeps its original independent Beta$(1,1)$ priors on each coordinate, its posterior draws, and its updates from each selected case. Let $C_a$ be the correctness, $Z_a$ the binary prediction, and $S\in\{0,1\}$ the subgroup. We define
\begin{equation}
\rho_a=\mathbb{P}(C_a=1),\qquad \mu_{a,s}=\mathbb{P}(Z_a=1| S=s),\qquad
g_a=|\mu_{a,0}-\mu_{a,1}|,\qquad U_a=\rho_a-\lambda g_a-c_a.
\end{equation}
The costs $c_a\geq0$ are known, $\lambda\geq0$ is finite, and $\pi_s=\mathbb{P}(S=s)>0$. We choose $a^*\in\arg\max_a U_a$, set $\Delta_a=U_*-U_a$, and let $K_a(T)$ count the selections of arm $a$. We write
\begin{equation}
\pi_{\min}=\min_s\pi_s,\qquad V_\lambda=1+\frac{\lambda^2}{\pi_{\min}}.
\end{equation}
Then $R_\lambda(T)=\sum_a\Delta_a K_a(T)$. All bounds below are for $T\geq2$.

For positive fairness weights, we impose the following two population conditions. First, the three true means of the optimal arm lie in $[\kappa_0,1-\kappa_0]$ with the fixed numerical margin $\kappa_0=1/100$. Second, let $P_*$ be the joint law of $(S,Z_*,C_*)$, and define
\begin{align*}
\mathcal{U}(\rho,\mu_0,\mu_1)&=\rho-\lambda|\mu_0-\mu_1|-c_{a^*},\\
K(Q,\theta)&=d(q_C,\rho)+\sum_s q_s d(q_{Z,s},\mu_s),\\
J_\lambda(Q)&=\min_{\theta:\mathcal{U}(\theta)\geq U_*}K(Q,\theta).
\end{align*}
Here $d$ is the Bernoulli KL divergence, $q_C=Q(C=1)$, $q_s=Q(S=s)$, and $q_{Z,s}=Q(Z=1| S=s)$. A conditional term with zero weight is zero. We assume
\begin{equation}
J_\lambda(Q)\leq D(Q\Vert P_*)
\quad\text{\normalfont for every $Q$ with finite divergence}.
\label{eq:information}
\end{equation}
Both sides of \eqref{eq:information} are exponential rates for $n$ observations of the optimal arm with empirical law $Q$. The product Beta posterior has density proportional to $e^{-nK(Q,\theta)}$, so it draws a parameter with utility at least $U_*$ with probability $e^{-nJ_\lambda(Q)}$, up to polynomial factors in $n$. A history with empirical law $Q$ occurs with probability at most $e^{-nD(Q\Vert P_*)}$. The condition therefore says that a history that makes the optimal arm pessimistic is at least as rare as the pessimism it causes, so such histories cannot dominate the expected anti-concentration for an optimistic draw. At $\lambda=0$, the condition holds automatically by \eqref{eq:information-zero}. For $\lambda>0$, $K$ adds the accuracy and subgroup costs separately, whereas $D(Q\Vert P_*)$ is computed from the joint law of correctness and prediction, so \eqref{eq:information} is an additional condition on $P_*$.

Throughout, $C\geq1$ denotes a universal numerical constant that may change from line to line.

\begin{theorem}\label{thm:main}
Under the conditions above, for every arm $a$ with $\Delta_a>0$ and $T\geq2$, Algorithm \ref{alg:ts} satisfies
\begin{equation}
\mathbb{E}[K_a(T)]\leq\frac{C}{\Delta_a^2}\Big[
\underbrace{\log T}_{\text{\normalfont accuracy}}
+\underbrace{\frac{\lambda^2}{\pi_{\min}}\log T}_{\text{\normalfont fairness estimation}}
+\underbrace{\left(\frac{\lambda}{\sqrt{\pi_{\min}}}+
\frac{\lambda^2}{\pi_{\min}}\right)(\log T)^{7/2}}_{\text{\normalfont additional optimal-arm anti-concentration}}
\Big].
\label{eq:main-count}
\end{equation}
At $\lambda=0$, the two additional population conditions are not needed. Moreover, we also have:
\begin{equation}
\mathbb{E}[R_\lambda(T)]\leq C\Big[
\underbrace{\sqrt{NT\log T}}_{\text{\normalfont accuracy}}
+\underbrace{\frac{\lambda}{\sqrt{\pi_{\min}}}\sqrt{NT\log T}}_{\text{\normalfont fairness estimation}}
+\underbrace{\sqrt{\left(\frac{\lambda}{\sqrt{\pi_{\min}}}+
\frac{\lambda^2}{\pi_{\min}}\right)T(\log T)^{7/2}}}_{\text{\normalfont additional optimal-arm anti-concentration}}\Big].
\label{eq:main-free}
\end{equation}
\end{theorem}

The displayed bounds separate three contributions: accuracy, fairness estimation from subgroup observations, and additional optimal-arm anti-concentration. The accuracy term already includes the scalar anti-concentration. The additional anti-concentration cost is shared across competitors, so its gap-free contribution has no factor $\sqrt N$. Both fairness contributions vanish at $\lambda=0$. We do not change the algorithm or its sampling rule.

The anti-concentration term in \eqref{eq:main-free} is absorbed into the first two terms when the number of arms is large or the fairness weight is small. Either of the two conditions
\begin{equation}
N\geq(\log T)^{5/2}
\quad\text{\normalfont or}\quad
\lambda\leq\frac{N\sqrt{\pi_{\min}}}{(\log T)^{5/2}}
\label{eq:absorption-condition}
\end{equation}
implies
\begin{equation}
\mathbb{E}[R_\lambda(T)]\leq C\left(1+\frac{\lambda}{\sqrt{\pi_{\min}}}\right)\sqrt{NT\log T}.
\label{eq:two-term}
\end{equation}
The ratio of the anti-concentration term to the accuracy and fairness-estimation terms in \eqref{eq:main-free} is at most $(\log T)^{5/4}/\sqrt N$, so the anti-concentration term is relatively negligible whenever $N/(\log T)^{5/2}\to\infty$. This compares contributions to the upper bound only. It does not imply that the absolute anti-concentration term decreases with $N$, or that a fixed $N$ suffices as $T\to\infty$. We also have no lower bound showing that the extra logarithmic factors are necessary.

Lemma 2.9 of \citet{agrawal2017near} bounds the expected inverse odds of a scalar Beta optimism event, averaging over a binomial success count. Our exchange argument remains valid, but for $\lambda>0$ the corresponding event is
\begin{equation}
\tilde\rho_*-\rho_*+\epsilon>
\lambda(\tilde g_*-g_*).
\end{equation}
This event involves accuracy and the absolute difference of two subgroup draws with random observation counts. At $g_*=0$, a disagreement in either direction creates a fairness penalty. Although the posterior draws are conditionally independent, their empirical histories can be dependent because accuracy and subgroup observations come from the same cases. Consequently, coordinatewise concentration bounds alone do not control the required expected inverse odds, and products of conditional probabilities cannot generally be averaged by multiplying their expectations.

The anti-concentration proof first combines the scalar bound with concentration on histories where the fairness error can be controlled within the utility margin. The remaining histories require a multivariate argument. The information condition compares the minimum KL cost of an optimistic parameter, $nJ_\lambda(Q)$, with the relative-entropy cost of its history, $nD(Q\Vert P_*)$. Fixed interiority of the optimal arm supports a finite-sample lower bound on the corresponding posterior mass. The positive utility margin strengthens the exponential bound and supplies decay when summing over observation counts. Weighted information-moment bounds then control averaging over the dependent empirical histories. Thus, the information comparison alone does not establish the anti-concentration estimate. Interiority and information dominance are population assumptions, whereas the margin is a positive proof tolerance chosen from the utility gap.

For $\lambda=0$, the information comparison holds automatically:
\begin{equation}
J_0(Q)\leq d(q_C,\rho_*)\leq D(Q\Vert P_*).
\label{eq:information-zero}
\end{equation}
The first inequality follows by setting $\rho=\rho_*$ and the unused subgroup coordinates to their empirical values. The second follows from data processing. In this case, the optimism event is scalar, so Lemma 2.9 of \citet{agrawal2017near} applies directly without either additional population assumption. We invoke \eqref{eq:information} and fixed interiority only in the multivariate argument.

Related reciprocal-probability arguments appear in combinatorial Thompson sampling. For example, \citet{perrault2020statistical} analyze Beta sampling under independent base-arm outcomes and Gaussian sampling under dependence. These results motivate a related proof strategy, but do not directly establish the required anti-concentration bound for our product-Beta sampler with dependent accuracy and subgroup observations.

\subsection*{Proof}
We follow the decomposition and the bounds in \cite[Section 2.1 and Lemmas 2.8--2.12]{agrawal2017near} (A\&G), with thresholds on the utility. The joint anti-concentration estimate below replaces their scalar Lemma 2.9.

\textbf{1. Thresholds and three cases.}
Arms with $c_a-c_{a^*}>1+\lambda$ are never selected, and every remaining arm has $\Delta_a\leq2(1+\lambda)$. For such a suboptimal arm, we set
\begin{equation}
\epsilon_a=\Delta_a/3,\qquad x_a=U_a+\epsilon_a,\qquad y_a=U_*-\epsilon_a.
\end{equation}
Let $\overline U_a(t)$ be the utility evaluated at the coordinate posterior means, and let $\tilde U_a(t)$ be the sampled utility. We define
\begin{equation}
E_a^\mu(t)=\{\overline U_a(t)\leq x_a\},\qquad
E_a^\theta(t)=\{\tilde U_a(t)\leq y_a\}.
\end{equation}
Since $U_a<x_a<y_a<U_*$, the decomposition of Agrawal and Goyal is exact:
\begin{equation}
\begin{aligned}
\mathbb{E}[K_a(T)]
={}&\underbrace{\sum_{t\leq T}\mathbb{P}(A_t=a,E_a^\mu(t),E_a^\theta(t))}_{\text{I}_a}\\
&+\underbrace{\sum_{t\leq T}\mathbb{P}(A_t=a,E_a^\mu(t),E_a^\theta(t)^c)}_{\text{II}_a}
+\underbrace{\sum_{t\leq T}\mathbb{P}(A_t=a,E_a^\mu(t)^c)}_{\text{III}_a}.
\end{aligned}
\label{eq:three-cases}
\end{equation}

\textbf{2. Exchange inequality and optimal-arm epochs (Lemma 2.8 of A\&G).}
Let $\mathcal F_{t-1}$ be the observed history before round $t$. For a set $B$ of suboptimal arms and a common threshold $y<U_*$, let
$p_t(y)=\mathbb{P}(\tilde U_*(t)>y|\mathcal F_{t-1})$. Since the arm draws are conditionally independent,
\begin{equation}
\begin{aligned}
&\mathbb{P}(A_t\in B,\overline U_{A_t}(t)\leq x_{A_t},\tilde U_{A_t}(t)\leq y|\mathcal F_{t-1})\\
&\hspace{15pt}\leq\frac{1-p_t(y)}{p_t(y)}\mathbb{P}(A_t=a^*|\mathcal F_{t-1}).
\end{aligned}
\label{eq:group-exchange}
\end{equation}
To see this, we fix all draws except that of the optimal arm, whose conditional CDF is $F_*$. Let $b$ maximize the other sampled utilities, with maximum value $M$. The event on the left requires $b\in B$, $\overline U_b\leq x_b$, and $M\leq y$. On these draws, its probability divided by the probability of selecting $a^*$ is $F_*(M)/(1-F_*(M))\leq(1-p_t(y))/p_t(y)$. Integrating and dropping the restrictions on the right proves the inequality. Ties have probability zero.

Let $\mathcal H_n$ be the first $n$ observations of the optimal arm in the IID observation stream. We define
\begin{equation}
p_n(\epsilon)=\mathbb{P}(\tilde U_*>U_*-\epsilon|\mathcal H_n),
\label{eq:optimism-probability}
\end{equation}
\begin{equation}
W(\epsilon,T)=\sum_{n<T}\mathbb{E}\left[\min\left\{T,\frac{1-p_n(\epsilon)}{p_n(\epsilon)}\right\}\right].
\label{eq:anti-concentration-definition}
\end{equation}
The probability $p_t(y)$ does not change between successive selections of the optimal arm. Conditional on the starting history of an epoch, summing \eqref{eq:group-exchange} bounds the expected case-I count of the epoch by the inverse odds, since the epoch contains at most one selection of the optimal arm. This count is also at most $T$. We then sum over epochs and drop their reach indicators using the pre-generated IID streams. Thus the event on the left of \eqref{eq:group-exchange} has expected count at most $W(U_*-y,T)$. In particular, $\text{I}_a\leq W(\epsilon_a,T)$.

\textbf{3. Inverse odds (replacement for Lemmas 2.9--2.10 of A\&G).}
\begin{lemma}[Optimal-arm anti-concentration]\label{lem:main-anti-concentration}
Assume that the optimal arm satisfies fixed interiority,
$\rho_*,\mu_{*,0},\mu_{*,1}\in[\kappa_0,1-\kappa_0]$
with $\kappa_0=1/100$, and the information condition
$J_\lambda(Q)\leq D(Q\Vert P_*)$
for every $Q$ with finite divergence. Set
\begin{equation}
B_\lambda(T)=\min\left\{
\frac{\lambda^2}{\pi_{\min}}(\log T)^{9/2},
V_\lambda(\log T)^{5/2}
\right\}.
\label{eq:anti-concentration-numerator}
\end{equation}
For $T\geq2$ and $0<\epsilon\leq1+\lambda$,
\begin{equation}
W(\epsilon,T)\leq
\frac C{\epsilon^2}\left[\log T+B_\lambda(T)\right].
\label{eq:combined-anti-concentration}
\end{equation}
At $\lambda=0$, this reduces to the scalar bound from
Lemma 2.9 of \citet{agrawal2017near}, and neither population
condition is required.
\end{lemma}
We prove this lemma separately. This is the only step that uses the interiority and information assumptions on the optimal arm. Since $\epsilon_a\leq1+\lambda$, the lemma gives
$\text{I}_a\leq C[\log T+B_\lambda(T)]/\epsilon_a^2$.

\textbf{4. Empirical error (Lemma 2.11 of A\&G).}
Hoeffding's inequality, the smoothing error of the posterior mean, and the binomial subgroup counts give, at the fixed prefix of the first $n$ observations of arm $a$,
\begin{equation}
\mathbb{P}(\overline U_{a,n}-U_a>\epsilon_a)
\leq C e^{-n\epsilon_a^2/(CV_\lambda)}.
\end{equation}
Each prefix is charged at most once by a selection of the arm. Therefore,
\begin{equation}
\text{III}_a\leq\sum_{n<T}C e^{-n\epsilon_a^2/(CV_\lambda)}
\leq CV_\lambda/\epsilon_a^2.
\label{eq:case-three}
\end{equation}

\textbf{5. Posterior error (Lemma 2.12 of A\&G).}
On $E_a^\mu(t)\cap E_a^\theta(t)^c$, we have
$\tilde U_a(t)-\overline U_a(t)>\epsilon_a$.
Let $m_a=\lceil CV_\lambda\log T/\epsilon_a^2\rceil$. At most $m_a$ selections occur before the arm reaches this pull count. For $n\geq m_a$ and subgroup counts $N_{a,s}\geq n\pi_s/2$, conditional Beta concentration gives
\begin{equation}
\mathbb{P}(\tilde U_a-\overline U_a>\epsilon_a|\mathcal F_{t-1})\leq C/T.
\end{equation}
Here we allocate the tolerance $\epsilon_a/2$ to accuracy and $\epsilon_a/(4\lambda)$ to each subgroup. Summing this conditional bound over rounds costs at most $C$. The selected prefixes with a deficient subgroup count cost at most
$\sum_{n\geq0}\sum_s e^{-n\pi_s/8}\leq C/\pi_{\min}$.
We need the subgroup control only if $\epsilon_a\leq4\lambda$, in which case $1/\pi_{\min}\leq16V_\lambda/\epsilon_a^2$. Otherwise, the fairness deviation is at most $\lambda<\epsilon_a/4$, and the accuracy concentration alone is enough. Consequently,
\begin{equation}
\text{II}_a\leq CV_\lambda\log T/\epsilon_a^2.
\label{eq:case-two}
\end{equation}

\textbf{6. Add the three bounds (final combination of A\&G).}
Substituting the three bounds into \eqref{eq:three-cases} proves the sharper estimate
\begin{equation}
\mathbb{E}[K_a(T)]\leq\frac C{\Delta_a^2}\left[V_\lambda\log T+B_\lambda(T)\right].
\label{eq:sharp-count}
\end{equation}
The elementary inequality $\min\{b_1,b_2\}\leq\sqrt{b_1b_2}$ gives
\begin{equation}
B_\lambda(T)\leq\frac{\lambda\sqrt{V_\lambda}}{\sqrt{\pi_{\min}}}(\log T)^{7/2}
\leq\left(\frac{\lambda}{\sqrt{\pi_{\min}}}+\frac{\lambda^2}{\pi_{\min}}\right)(\log T)^{7/2}.
\label{eq:simplified-anti-concentration-numerator}
\end{equation}
Expanding $V_\lambda$ proves \eqref{eq:main-count}. At $\lambda=0$, the proof uses only the scalar accuracy bounds.

\subsection*{Gap-free regret: the cutoff of A\&G with shared anti-concentration}
Fix $d>0$. As in \cite[Section 2.2]{agrawal2017near}, gaps at most $d$ contribute at most $Td$. We group the arms with larger gaps as
\begin{equation}
d_j=2^j d,\qquad B_j=\{a:d_j<\Delta_a\leq2d_j\},\qquad y_j=U_*-d_j/3.
\end{equation}
We use the preceding decomposition with the same $x_a$ and the common threshold $y_j$ in group $B_j$. Since $y_j-x_a=(2\Delta_a-d_j)/3\geq\Delta_a/3$, Steps 4 and 5 contribute at most
$CV_\lambda\log T\sum_{\Delta_a>d}\Delta_a^{-1}$.
By Step 2, the case-I count of the whole group is bounded once by $W(d_j/3,T)$. All selectable arms have $\Delta_a\leq2(1+\lambda)$, so the anti-concentration lemma applies in every nonempty group. Therefore, the combined case-I regret is at most
\begin{equation}
\sum_{j:B_j\ne\varnothing}2d_jW(d_j/3,T)
\leq C[\log T+B_\lambda(T)]\sum_{j\geq0}\frac1{d_j}
=\frac C d[\log T+B_\lambda(T)].
\end{equation}
This grouping is the only refinement of the gap cutoff of Agrawal and Goyal: it avoids charging the same optimal-arm anti-concentration cost separately to every competitor. We obtain
\begin{equation}
\mathbb{E}[R_\lambda(T)]\leq Td+CV_\lambda\log T\sum_{\Delta_a>d}\frac1{\Delta_a}
+\frac C d[\log T+B_\lambda(T)].
\label{eq:grouped-cutoff}
\end{equation}
We bound the sum by $N/d$, absorb the scalar $\log T$ term into $NV_\lambda\log T$, and optimize over $d$. This gives
\begin{equation}
\mathbb{E}[R_\lambda(T)]\leq C\sqrt{T[NV_\lambda\log T+B_\lambda(T)]}.
\label{eq:sharp-free}
\end{equation}
Separating the square root and using \eqref{eq:simplified-anti-concentration-numerator} proves \eqref{eq:main-free}. Under either condition in \eqref{eq:absorption-condition}, we have $\frac{\lambda}{\sqrt{\pi_{\min}}}(\log T)^{5/2}\leq\left(1+\frac{\lambda}{\sqrt{\pi_{\min}}}\right)N$, so the anti-concentration term in \eqref{eq:main-free} is at most the sum of the first two terms, which proves \eqref{eq:two-term}. Finally, $B_0(T)=0$ gives the scalar gap-free bound without the additional population conditions.

\subsection{Proof of Lemma \ref{lem:main-anti-concentration}}

Let $Q_n$ be the empirical joint law, $N_s$ the subgroup counts,
and $\hat\rho,\hat\mu_s$ the empirical means. Set
\begin{align}
\mathcal L_n(\theta)&=n d(\hat\rho,\rho)
  +\sum_sN_s d(\hat\mu_s,\mu_s),
&M_n&=1+\min_sN_s,\nonumber\\
X_n&=\inf_{\substack{Q_{SZ}=Q_{n,SZ}\\Q_C=Q_{n,C}}}nD(Q\Vert P_*),
&j_n&=\inf_{\mathcal U(\theta)\ge U_*-\epsilon/2}\mathcal L_n(\theta),
\label{eq:contracted-information}\\
b_\delta(m)&=\min\{1,\sqrt{\delta(m+1)}\},
&D_{n,\delta}&=[b_\delta(n)b_\delta(N_0)b_\delta(N_1)]^{-1}.\nonumber
\end{align}
Zero-count likelihood terms vanish and $X_0=j_0=0$.
The conditional density is $e^{-\mathcal L_n}/Z_n$.
Information dominance and data processing imply
\begin{equation}
nJ_\lambda(Q_n)\le X_n,
\qquad \mathcal L_n(\theta_*)\le2X_n,
\qquad X_n\ge nD(Q_{n,SZ}\Vert P_{*,SZ}).
\label{eq:contracted-properties}
\end{equation}
The first inequality is interpreted as zero at $n=0$.

\begin{lemma}[History moments]\label{lem:weighted}
For $0<\delta<1$ and $0\le u\le1-\delta$,
\begin{equation}
\mathbb E[D_{n,\delta}e^{uX_n}]\le C\delta^{-2}.
\label{eq:weighted-moment}
\end{equation}
For $n\ge1$ and $\delta\le1/64$, also
\begin{equation}
\mathbb E \Bigg[ \frac{D_{n,\delta}e^{uX_n}}
{\sqrt{\max\{M_n,\delta^{-1}\}}} \Bigg]
\le C\left[
\frac{\delta^{-2}}{\sqrt{1+n\pi_{\min}}}
+\delta^{-3/2}e^{-\delta n\pi_{\min}/C}\right].
\label{eq:inverse-count}
\end{equation}
\end{lemma}
\begin{proof}
First, for a sum $Y$ of at most four independent binomials with $m$
total trials and $I(x)=\sup_t\{tx-\log\mathbb E e^{tY}\}$, we claim
\begin{equation}
\mathbb E \Big[ e^{uI(Y)} \Big]
\le C\min\{\delta^{-1/2},\sqrt{m+1}\}.
\label{eq:pb-moment}
\end{equation}
Remove deterministic trials, translating $Y$ and redefining $m$.
For the remaining groups $(m_j,p_j)$, exponential tilting gives
$q_j(t)=p_je^t/(1-p_j+p_je^t)$,
$x=\sum_jm_jq_j$ and $v=\sum_jm_jq_j(1-q_j)$.
Thus $dx/dt=v$, $|dv/dx|\le1$, and
$I(x(t))=\sum_jm_jd(q_j(t),p_j)$.
The Bernoulli-sum atom bound
yields
$\mathbb P(Y=k)\le Ce^{-I(k)}/\sqrt{1+v(k)}$, including the endpoint
tilts by continuity. Compare each summand with the unit interval
toward $\mathbb E Y$, leaving at most two central terms. Monotonicity
of $I$ on either side of its minimum and $|dv/dx|\le1$ give
\begin{align*}
\mathbb E \Big[ e^{uI(Y)} \Big]
&\le C+C\int_0^m\frac{e^{-\delta I(x)}}{\sqrt{v(x)}}dx\\
&\le C+C\sum_j\sqrt{m_j}
\int_0^1\frac{e^{-\delta m_jd(q,p_j)}}{\sqrt{q(1-q)}}dq
\le C\min\{\delta^{-1/2},\sqrt{m+1}\}.
\end{align*}
The second line uses $dq_j/dt=q_j(1-q_j)$ and
$\sqrt{\sum_j a_j}\le\sum_j\sqrt{a_j}$.
The final inequality follows from $q=\sin^2\alpha$,
$d(q,p_j)\ge c(\alpha-\arcsin\sqrt{p_j})^2$, and at most four groups.
The deterministic case is immediate, and restoring removed trials
only enlarges the bound.

Write $n_{sz}$ for the $(S,Z_*)$ counts, $K_C$ for total correctness,
and $\gamma_{sz}=\mathbb P(C_*=1| S=s,Z_*=z)$.
The KL chain rule and information projection
give
\begin{equation}
\begin{split}
X_n&=n d(N_0/n,\pi_0)
+\sum_sN_s d(\hat\mu_s,\mu_{*,s})+I_n(K_C),\\
I_n(k)&=\sup_t\left\{
tk-\sum_{s,z}n_{sz}\log(1-\gamma_{sz}+\gamma_{sz}e^t)\right\}.
\end{split}
\label{eq:contracted-chain}
\end{equation}
Restrict to the support of $P_*$ and use limiting tilts at boundaries.
Conditional on $(n_{sz})$, correctness is a sum of four independent
binomials. Apply \eqref{eq:pb-moment} and pair this moment with
$b_\delta(n)^{-1}$ using
\[
b_\delta(m)^{-1}
\min\{\delta^{-1/2},\sqrt{m+1}\}=\delta^{-1/2}.
\]
Conditional on $(N_0,N_1)$, apply the same argument to the two
independent prediction counts, pairing with $b_\delta(N_s)^{-1}$.
The subgroup-count moment contributes a fourth $\delta^{-1/2}$
factor, proving \eqref{eq:weighted-moment}.
At $n=0$, $D_{0,\delta}=\delta^{-3/2}\le\delta^{-2}$.

For \eqref{eq:inverse-count}, split at
$\min_sN_s=n\pi_{\min}/2$. On the upper event, use
\eqref{eq:weighted-moment}. On the lower event,
$X_n\ge n\pi_{\min}/8$ and the inverse-count factor is at most
$\sqrt\delta$, so its contribution is bounded by
\[
\sqrt\delta e^{-\delta n\pi_{\min}/16}
\mathbb E[D_{n,\delta/2}e^{(u+\delta/2)X_n}]
\le C\delta^{-3/2}e^{-\delta n\pi_{\min}/16}.
\]
\end{proof}

\begin{lemma}[Conditional posterior anti-concentration]\label{lem:volume}
Under the population conditions, for $0<\delta\le1/64$,
\begin{align}
\sqrt{j_n}
&\le\left(\sqrt{X_n}
-\frac{\epsilon\sqrt n}{C\sqrt{V_\lambda}}\right)_+,
\label{eq:prefix-level-shift}\\
p_n(\epsilon)
&\ge\frac{\delta^{3/2}}{CD_{n,\delta}}
e^{-(1-\delta)j_n-4\delta X_n}
\min\left\{1,
\frac{\epsilon\sqrt{\max\{M_n,\delta^{-1}\}}}
{\lambda\sqrt\delta}\right\}.
\label{eq:shifted-volume}
\end{align}
The minimum is one at $\lambda=0$.
\end{lemma}
\begin{proof}
For any $Q$, let $D=D(Q\Vert P_*)$ and
$\hat\theta=(q_C,q_{Z,0},q_{Z,1})$.
Pinsker and the split $q_s\ge\pi_s/2$ give
$h=(U_*-\mathcal U(\hat\theta))_+\le C\sqrt{V_\lambda D}$.
On the complementary subgroup event use $D\ge\pi_s/8$ and error
at most one. Let $J_{\lambda,r}$ denote the cost with utility
threshold $U_*-r$. It vanishes if $h\le r$.
Otherwise, interpolate an unrelaxed minimizer toward
$\hat\theta$ by $r/h$. Concavity of utility, convexity of KL,
and information dominance imply
\[
J_{\lambda,r}(Q)\le(1-r/h)D
\le D-\frac{r\sqrt D}{C\sqrt{V_\lambda}}.
\]
Take square roots, set $r=\epsilon/2$, and minimize $D$ over the
recorded marginals to obtain \eqref{eq:prefix-level-shift}.

For \eqref{eq:shifted-volume}, let
$\mathcal L_n(\theta^\circ)=j_n$, put $m=(n,N_0,N_1)$,
and take $v=\theta_*+w$ with
$|w_j|\le r_j=C^{-1}\min\{1,[\delta(m_j+1)]^{-1/2}\}$.
Fixed interiority and Taylor's formula give
$\delta\mathcal L_n(v)\le4\delta X_n+1$.
Hence convexity bounds the likelihood cost at
$\theta=(1-\delta)\theta^\circ+\delta v$ by
$(1-\delta)j_n+4\delta X_n+1$.
Retain $w_\rho\ge0$ and
$|w_0-w_1|\le\epsilon/(4\delta\lambda)$.
The image has utility at least
$U_*-(1-\delta)\epsilon/2-\epsilon/4>U_*-\epsilon$,
and the retained box fraction is at least
$c\min\{1,\epsilon/(\delta\lambda\max_s r_s)\}$.
The strip is omitted at $\lambda=0$.
Since the Jacobian is $\delta^3$ and
$Z_n\le C\prod_j(m_j+1)^{-1/2}$ by Pinsker, the identities
\[
\delta r_j\sqrt{m_j+1}=C^{-1}\sqrt\delta b_\delta(m_j),
\qquad
\delta\max_s r_s
=\frac{\sqrt\delta}{C\sqrt{\max\{M_n,\delta^{-1}\}}}
\]
prove \eqref{eq:shifted-volume}, including zero counts.
\end{proof}

\begin{proof}[Proof of Lemma~\ref{lem:main-anti-concentration}]
Write $p_n=p_n(\epsilon)$, $O_n=(1-p_n)/p_n$ and
$q_n=\mathbb P(\tilde\rho_*>\rho_*-\epsilon/2\mid\mathcal H_n)$.
The Beta-Binomial identity and Chernoff bounds
\citep[Appendix A]{agrawal2017near}, together with
\citet[Lemma 2.9]{agrawal2017near}, give
\begin{equation}
\begin{gathered}
\mathbb{E} [(1-p_n)]\le Ce^{-n\epsilon^2/(CV_\lambda)},
\qquad
\mathbb{E}  \Big[ M_n^{-1/2} \Big]\le C(1+n\pi_{\min})^{-1/2},\\
\sum_{n<T}\mathbb{E} \Bigg[\frac{1-q_n}{q_n}\Bigg]\le C\log T/\epsilon^2.
\end{gathered}
\label{eq:standard}
\end{equation}
For the first bound allocate $\epsilon/2$ to accuracy and, when
$\lambda>0$, $\epsilon/(4\lambda)$ to each subgroup. Average the binomial
counts, omitting tolerances above one. The second uses a split
at half the expected count. Boundary correctness means follow
directly from the Beta formula.
If $\epsilon\ge2\lambda$, then $p_n\ge q_n$, proving the claim.
Otherwise set
$\nu_f=\lambda^2/(\pi_{\min}\epsilon^2)$ and
$\nu=V_\lambda/\epsilon^2$, so $\nu\ge\nu_f>1/2$.

Posterior means and variances, Pinsker, and
\eqref{eq:contracted-properties} yield
\[
\mathbb P(\lambda(\tilde g_*-g_*)\le h_n\mid\mathcal H_n)\ge3/4,
\qquad
h_n=\min\left\{\lambda,
\frac{C\lambda}{\sqrt{\pi_{\min}}}
\sqrt{\frac{X_n+1}{n+1}}\right\}.
\]
Indeed, Chebyshev gives posterior error at most $3/\sqrt{N_s+1}$
from the empirical mean with probability $15/16$ for each positive
count. If $N_s\ge n\pi_s/2$, use
$X_n\ge N_s d(\hat\mu_s,\mu_{*,s})$.
Otherwise $X_n\ge n\pi_s/8$ and the error is at most one.
The case $n=0$ is immediate.
On $h_n\le\epsilon/2$, conditional independence gives $p_n\ge3q_n/4$.
Splitting at $p_n=1/2$ and $q_n=1/2$, and using \eqref{eq:standard},
therefore gives
\begin{equation}
\begin{gathered}
W(\epsilon,T)\le C\nu+C\log T/\epsilon^2
+\sum_{n<T}\mathbb E[\min\{T,O_n\}\1_{F_n}],\\
F_n=\{h_n>\epsilon/2,p_n<1/2\}
\quad\Longrightarrow\quad n+1\le C\nu_f(X_n+1).
\end{gathered}
\label{eq:localization}
\end{equation}

On $F_n\cap\{x\le X_n<2x\}$ with $x\ge1$,
$n\le C\nu_fx$ and \eqref{eq:prefix-level-shift} gives
$j_n\le X_n-c\sqrt{nX_n/\nu}$ after decreasing $c$.
Invert \eqref{eq:shifted-volume} with $\delta=(64x)^{-1}$,
use $e^{X_n}\le Ce^{(1-\delta)X_n}$ on this shell, and apply
Lemma~\ref{lem:weighted}. For $n\ge1$ this gives
\begin{equation}
\mathbb E[O_n\1_{F_n\cap\{x\le X_n<2x\}}]
\le Ce^{-c\sqrt{nx/\nu}}
\left[x^{7/2}
+\frac{\lambda x^3}{\epsilon\sqrt{1+n\pi_{\min}}}
+\frac{\lambda x^{5/2}}{\epsilon}e^{-n\pi_{\min}/(Cx)}\right].
\label{eq:localized-shell}
\end{equation}
Only $n\le C\nu_fx$ contribute. Integral comparison bounds the
three sums, after removing their prefactors, by
\[
C\min\{\nu_fx,\nu/x\},
\quad
\frac C{\sqrt{\pi_{\min}}}
\min\{\sqrt{\nu_fx},\sqrt{\nu/x}\},
\quad
C\min\{\nu_fx,\nu/x,x/\pi_{\min}\},
\]
respectively. Using $\lambda/\epsilon=\sqrt{\pi_{\min}\nu_f}$
and $\min\{a,b\}\le\sqrt{ab}$ for the last term yields
\begin{equation}
\sum_{n<T}\mathbb E[O_n\1_{F_n\cap\{x\le X_n<2x\}}]
\le C\min\{\nu_fx^{9/2},\nu x^{5/2}\}.
\label{eq:shell-sum}
\end{equation}
For $F_n\cap\{X_n<1\}$, use $n\le C\nu_f$,
\eqref{eq:shifted-volume} at $\delta=1/64$, and \eqref{eq:standard}.
The contribution is at most
$C\nu_f+C\lambda/\epsilon+
C\lambda\sqrt{\nu_f}/(\epsilon\sqrt{\pi_{\min}})\le C\nu_f$.
For $X_n>8\log T$, $D_{n,1/2}\ge1$ and
\eqref{eq:weighted-moment} give
$T\sum_{n<T}\mathbb P(X_n>8\log T)\le CT^{-2}\le C\nu_f$.
Finally, sum \eqref{eq:shell-sum} over dyadic shells up to $8\log T$.
Each geometric sum is dominated by its largest scale.
Substitute into \eqref{eq:localization} and use
$\nu=\epsilon^{-2}+\nu_f$ to conclude.
\end{proof}

\subsection{Regret Bound of the Thompson Sampling with Delays (Theorem \ref{thm:delay})}

\begin{theorem}[Regret Bound of Thompson Sampling with Fairness under Delayed Reward]
Consider Algorithm \ref{alg:ts} under the assumptions of
Theorem \ref{thm:bound}. For each arm, assume that label delays
are IID and independent of the cases and algorithmic randomization.
Fairness posteriors are updated immediately, whereas accuracy
posteriors are updated when labels arrive.
Choose $q_a\in(0,1]$ such that $d_a(q_a)<\infty$, and let
$q_{\min}=\min_a q_a$. Then, the gap-free regret bound becomes
\begin{equation}\nonumber
\begin{aligned}
\mathbb E[R_\lambda(T)]
\leq O\Bigg(
\sqrt{\frac{NT\log T}{q_{\min}}}
+\frac{\lambda}{\sqrt{\pi_{\min}}}\sqrt{NT\log T}+(1+\lambda)\sum_{a=1}^N d_a(q_a)
+\sqrt{T H_D(T)}
\Bigg),
\end{aligned}
\end{equation}
where $H_D(T)$ summarizes optimal-arm anti-concentration caused by
pessimistic utility draws under delayed feedback. Its formal
definition is given in Appendix \ref{app:delay-proof}.
\end{theorem}

\subsection*{Proof}
\label{app:delay-proof}

We use the notation and decomposition from the proof of
Theorem \ref{thm:bound}. In particular, $U_a$, $\Delta_a$,
$K_a(T)$, $\overline U_a(t)$, $\tilde U_a(t)$, and
$\mathcal F_{t-1}$ have exactly the same meanings. We retain
\begin{equation}
V_\lambda=1+\frac{\lambda^2}{\pi_{\min}},
\qquad
B_\lambda(T)=\min\left\{
\frac{\lambda^2}{\pi_{\min}}(\log T)^{9/2},
V_\lambda(\log T)^{5/2}
\right\}.
\label{eq:delay-main-quantities}
\end{equation}
Recall that $W(\epsilon,T)$ in that proof averages the truncated
inverse optimism odds over the first $n$ complete observations
of the optimal arm, and its anti-concentration lemma gives
\begin{equation}
W(\epsilon,T)\le
\frac{C}{\epsilon^2}\left[\log T+B_\lambda(T)\right],
\qquad 0<\epsilon\le1+\lambda.
\label{eq:delay-recalled-optimism}
\end{equation}
Thus $B_\lambda(T)$ is the additional fairness contribution.
The scalar $\log T$ contribution is absorbed into
$NV_\lambda\log T$ in Theorem \ref{thm:bound}.

We first specify the two changes needed for delayed labels.
A label generated at round $u$ with integer delay $D_u$ is used
at the end of round $u+D_u$. Write
\begin{equation}
\begin{aligned}
K_a^{\mathrm{obs}}(t-1)
&=\sum_{u<t}\1\{A_u=a,u+D_u<t\},\\
N_{a,s}(t)
&=\sum_{u<t}\1\{A_u=a,S_u=s\}.
\end{aligned}
\label{eq:delay-observation-counts}
\end{equation}
Here $d_a(q_a)$ is the $q_a$-quantile of arm $a$'s delay
distribution. The accuracy posterior uses
$K_a^{\mathrm{obs}}(t-1)$ labels.
The subgroup posterior uses the immediate count $N_{a,s}(t)$,
as in Theorem \ref{thm:bound}.
In particular, $K_a(T)$ continues to count selections.

For the anti-concentration term, retain the optimism probability
from the main proof:
\begin{equation}
p_t(y)=
\mathbb P\left(\tilde U_*(t)>y|\mathcal F_{t-1}\right).
\label{eq:delay-optimism-probability}
\end{equation}
Partition the decision rounds into maximal consecutive epochs
during which all three optimal-arm coordinate posteriors remain
unchanged, including the subgroup posteriors when $\lambda=0$.
Let $\sigma_e$ be the first round of epoch $e$. Define
\begin{equation}
\mathcal W_D(\epsilon,T)=
\mathbb E\left[
\sum_{e:\sigma_e\le T}
\min\left\{
T,
\frac{1-p_{\sigma_e}(U_*-\epsilon)}
{p_{\sigma_e}(U_*-\epsilon)}
\right\}
\right],
\label{eq:delay-inverse-odds}
\end{equation}
and
\begin{equation}
H_D(T)=
\sup_{0<\epsilon\le1+\lambda}
\epsilon^2\mathcal W_D(\epsilon,T).
\label{eq:delay-H}
\end{equation}
Here $\mathcal W_D$ plays the role of $W$ in the exchange argument,
with the actual delayed posterior histories.
Consequently, $H_D(T)$ bounds the total optimal-arm
anti-concentration coefficient. Its immediate-feedback
counterpart is $\log T+B_\lambda(T)$, rather than $B_\lambda(T)$
alone. We establish this connection explicitly below.

As in Theorem \ref{thm:bound}, arms with
$c_a-c_{a^*}>1+\lambda$ are never selected, and every selectable
arm satisfies $\Delta_a\le2(1+\lambda)$.
The cases $N=1$ and $N>T$ follow from zero regret and the trivial
bound $2(1+\lambda)T$, respectively. We therefore take
$2\le N\le T$. All arm groups below contain selectable arms.

Use exactly the thresholds and events of the main proof:
\begin{equation}
\begin{gathered}
\epsilon_a=\Delta_a/3,
\qquad x_a=U_a+\epsilon_a,
\qquad y_a=U_*-\epsilon_a,\\
E_a^\mu(t)=\{\overline U_a(t)\le x_a\},
\qquad
E_a^\theta(t)=\{\tilde U_a(t)\le y_a\}.
\end{gathered}
\end{equation}
The same three-case decomposition gives
\begin{equation}
\begin{aligned}
\mathbb E[K_a(T)]
={}&\underbrace{\sum_{t\le T}
\mathbb P(A_t=a,E_a^\mu(t),E_a^\theta(t))}_{\mathsf I_a}\\
&+\underbrace{\sum_{t\le T}
\mathbb P(A_t=a,E_a^\mu(t),E_a^\theta(t)^c)}_{\mathsf {II}_a}\\
&+\underbrace{\sum_{t\le T}
\mathbb P(A_t=a,E_a^\mu(t)^c)}_{\mathsf {III}_a}.
\end{aligned}
\label{eq:delay-three-cases}
\end{equation}

The exchange inequality in Step 2 of the main proof applies
conditionally on $\mathcal F_{t-1}$ without modification.
It is the same conditional-draw argument as
\citet[Lemma 2.8]{agrawal2017near}.
For a group $B$ and a common threshold $y<U_*$, it gives
\begin{equation}
\begin{aligned}
\mathbb P\left(
A_t\in B,E_{A_t}^\mu(t),\tilde U_{A_t}(t)\le y
|\mathcal F_{t-1}
\right)\le
\frac{1-p_t(y)}{p_t(y)}
\mathbb P(A_t=a^*|\mathcal F_{t-1}).
\end{aligned}
\label{eq:delay-exchange}
\end{equation}
Within epoch $e$, $p_t(y)=p_{\sigma_e}(y)$.
Each epoch contains at most one optimal-arm selection, because
that selection immediately updates a subgroup posterior.
An accuracy-label arrival can end the epoch earlier.
Conditional on its starting history, summing
\eqref{eq:delay-exchange} bounds the expected group case-I count
by the inverse odds, and also by $T$.
Averaging and summing over the reached epochs gives
\begin{equation}
\mathbb{E}\Bigg[\sum_{t\le T}
\1\{A_t\in B,E_{A_t}^\mu(t),\tilde U_{A_t}(t)\Bigg]\le y\}
\le \mathcal W_D(U_*-y,T).
\label{eq:delay-group-I}
\end{equation}
In particular, $\mathsf I_a\le\mathcal W_D(\epsilon_a,T)$.
This is precisely the use of $W(\epsilon_a,T)$ in the main
proof, with delayed posterior epochs.

For cases II and III, we use the coordinate concentration
calculations in Steps 4 and 5 of that proof.
The observation-counting change is
\begin{equation}
\begin{aligned}
\mathbb E \Bigg[\sum_{t\le T}
\1\{A_t=a,K_a^{\mathrm{obs}}(t-1)<m\}\Bigg]
&\le \frac{m}{q_a}+d_a(q_a),\\
\mathbb E \Bigg[\sum_{t\le T}
\1\{A_t=a,N_{a,s}(t)<m\}\Bigg]
&\le \frac{m}{\pi_s}
\end{aligned}
\label{eq:delay-counting}
\end{equation}
for every integer $m\ge1$.
For the first inequality, pre-generate the arm's IID delays
in selection order. The expected index of the $m$th delay at
most $d_a(q_a)$ is at most $m/q_a$.
After that selection, all $m$ designated labels arrive within
$d_a(q_a)$ rounds. At most $d_a(q_a)$ further selections can
occur during those rounds.
The second inequality is the same negative-binomial argument
for immediate subgroup observations.
This is the delay-quantile counting principle of
\citet{wu2022thompson}.

For completeness, the accuracy concentration calculation remains
valid for the observed count in \eqref{eq:delay-observation-counts}.
Fix a calendar round $t$. Let $C_u$ be the correctness observed
for the case selected at round $u$, and let $\hat\rho_a(t)$
be the empirical accuracy from arm $a$'s arrived labels.
For analysis, expose complete past
cases, then the current action and delay, before the fresh case.
The mask $w_u=\1\{A_u=a,u+D_u<t\}$ is predictable.
Delay independence and noncontextual selection give
\begin{equation}
\mathbb E\Bigg[\exp\left(
\eta\sum_{u<t}w_u(C_u-\rho_a)
-\frac{\eta^2}{8}\sum_{u<t}w_u
\right)\Bigg] \le 1.
\end{equation}
Taking $\eta=4v$ and $\eta=-4v$ yields
\begin{equation}
\mathbb P\left(
K_a^{\mathrm{obs}}(t-1)\ge m,
|\hat\rho_a(t)-\rho_a|>v
\right)\le2e^{-2mv^2}.
\label{eq:delay-accuracy-concentration}
\end{equation}
The subgroup argument is identical after exposing the current
group. Posterior-mean smoothing and conditional Beta
concentration are unchanged. No independence between accuracy
and subgroup empirical histories is used.

As in the main proof, allocate tolerance $\epsilon_a/2$ to
accuracy and $\epsilon_a/(4\lambda)$ to each subgroup.
Choose observation thresholds of orders
$\log T/\epsilon_a^2$ and
$\lambda^2\log T/\epsilon_a^2$, respectively.
Omit a coordinate condition when its tolerance is at least one,
and omit the subgroup conditions at $\lambda=0$.
Above these thresholds, the empirical and posterior error
probabilities at each fixed round are at most $C/T^3$.
Below them, apply \eqref{eq:delay-counting}.
Combining the two cases before charging the insufficient-count
events gives
\begin{equation}
\begin{aligned}
\mathsf {II}_a+\mathsf {III}_a
&\le
\frac{C\log T}{\Delta_a^2}
\left(q_a^{-1}+\frac{\lambda^2}{\pi_{\min}}\right)
+d_a(q_a)\\
&=
\frac{C\log T}{\Delta_a^2}
\left(V_\lambda+q_a^{-1}-1\right)
+d_a(q_a).
\end{aligned}
\label{eq:delay-cases-II-III}
\end{equation}
Threshold ceilings and the summed failure probabilities are
absorbed into the leading term, using
$\Delta_a^2\le8(q_a^{-1}+\lambda^2/\pi_{\min})$.
Thus only the accuracy contribution $1$ in $V_\lambda$ becomes
$q_a^{-1}$. The fairness contribution remains
$\lambda^2/\pi_{\min}$.

Adding the three cases gives the delayed counterpart of the
selection-count bound in Theorem \ref{thm:bound}:
\begin{equation}
\mathbb E[K_a(T)]
\le
\frac{C(V_\lambda+q_a^{-1}-1)\log T}{\Delta_a^2}
+d_a(q_a)+\mathcal W_D(\Delta_a/3,T).
\label{eq:delay-count-bound}
\end{equation}

For the gap-free bound, use the same cutoff $d$ and groups
$d_j,B_j,y_j$ as in the main proof:
\begin{equation}
d_j=2^j d,
\qquad B_j=\{a:d_j<\Delta_a\le2d_j\},
\qquad y_j=U_*-d_j/3.
\end{equation}
Keep $x_a=U_a+\Delta_a/3$.
Since $y_j-x_a\ge\Delta_a/3$ for $a\in B_j$,
\eqref{eq:delay-cases-II-III} applies with the common threshold.
Charge gaps at most $d$ by $Td$ and use
\eqref{eq:delay-group-I} once per nonempty group. This gives
exactly the main proof's grouped cutoff with its delayed terms:
\begin{equation}
\begin{aligned}
\mathbb E[R_\lambda(T)]
\le{}&Td+
C\log T\sum_{\Delta_a>d}
\frac{V_\lambda+q_a^{-1}-1}{\Delta_a}\\
&+\sum_{\Delta_a>d}\Delta_a d_a(q_a)
+\sum_{j:B_j\ne\varnothing}2d_j\mathcal W_D(d_j/3,T).
\end{aligned}
\label{eq:delay-grouped-cutoff}
\end{equation}
Every nonempty group has $d_j/3\le1+\lambda$.
The definition of $H_D(T)$ therefore gives
\begin{equation}
\sum_{j:B_j\ne\varnothing}2d_j\mathcal W_D(d_j/3,T)
\le18H_D(T)\sum_{j\ge0}\frac1{2^j d}
=\frac{36H_D(T)}d.
\end{equation}
Also,
\begin{equation}
\sum_{\Delta_a>d}
\frac{V_\lambda+q_a^{-1}-1}{\Delta_a}
\le\frac{N(V_\lambda+q_{\min}^{-1}-1)}d,
\qquad
\sum_{\Delta_a>d}\Delta_a d_a(q_a)
\le2(1+\lambda)\sum_{a=1}^N d_a(q_a).
\end{equation}
Optimizing over $d$ as in Theorem \ref{thm:bound} yields
\begin{equation}
\begin{aligned}
\mathbb E[R_\lambda(T)]
\le{}&
C\sqrt{T\left[
N(V_\lambda+q_{\min}^{-1}-1)\log T+H_D(T)
\right]}\\
&+C(1+\lambda)\sum_{a=1}^N d_a(q_a).
\end{aligned}
\label{eq:delay-sharp-free}
\end{equation}
Expanding
$V_\lambda+q_{\min}^{-1}-1
=q_{\min}^{-1}+\lambda^2/\pi_{\min}$
and separating the square root proves the stated theorem.

Finally, the connection to Theorem \ref{thm:bound} uses its
original $W$ and $B_\lambda$.
When labels arrive immediately, take $q_a=1$ and $d_a(1)=0$.
Each reached epoch is then the posterior after a prefix of the
optimal-arm IID stream.
Dropping its reach indicator gives
\begin{equation}
\mathcal W_D(\epsilon,T)\le W(\epsilon,T),
\qquad
H_D(T)\le C\left[\log T+B_\lambda(T)\right].
\label{eq:delay-no-delay-connection}
\end{equation}
Substitution into \eqref{eq:delay-sharp-free} gives
\begin{equation}
\begin{aligned}
\mathbb E[R_\lambda(T)]
&\le C\sqrt{T\left[
NV_\lambda\log T+\log T+B_\lambda(T)
\right]}\\
&\le C\sqrt{T\left[NV_\lambda\log T+B_\lambda(T)\right]}.
\end{aligned}
\end{equation}
This is the sharper gap-free bound in the main proof.
Applying its same inequality
\begin{equation}
B_\lambda(T)
\le
\left(
\frac{\lambda}{\sqrt{\pi_{\min}}}
+\frac{\lambda^2}{\pi_{\min}}
\right)(\log T)^{7/2}
\end{equation}
recovers all three displayed terms of Theorem \ref{thm:bound}.
In particular, $B_0(T)=0$ recovers its scalar result.
For positive delays, \eqref{eq:delay-sharp-free} retains
$H_D(T)$ as defined in \eqref{eq:delay-H}.
The no-delay anti-concentration lemma is used only in
\eqref{eq:delay-no-delay-connection}.

\subsection*{non-IID $L_t$} In practice, delays may correlate with subgroup identity, for example, when each subgroup has its own stationary feedback-delay distribution. We believe the proof would extend naturally to this case. The key point is that our analysis already tracks per-subgroup effective sample sizes separately — the delayed-filtration argument of \cite{wu2022thompson} can be applied per-subgroup rather than globally. If delays are correlated with $S$, the bottleneck becomes the subgroup with the worst combination of low prevalence and long delays: intuitively, $\pi_{\min} \cdot q_{min}$ in Theorem \ref{thm:delay} would be replaced by $\min_s(\pi_s \cdot q_s)$, where $q_s$ is the subgroup-specific delay quantile. The rest of the posterior concentration argument proceeds as before, since updates remain order-invariant within each subgroup. In the worst case, the minority subgroup also experiences the longest delays, compounding the $\pi_{\min}$ factor. For other non-IID delays, the impact on the theoretical guarantees depends on the structure of the dependence.

\section{Additional Details About Experiments}\label{app:experiments}

\subsection{Dataset Details}\label{app:datasets}

For the ML datasets, the arms are trained on 60\% of the data (split stratified on $Y$), and the remaining 40\% is the stream population. Twitter, Toxicity and Bios have no training step, and all their instances form the stream population. Each of the 10{,}000 runs samples $T=50{,}000$ instances with replacement from the stream population (Section~\ref{sec:methodology}).

The grid of $\lambda$ values depends on the oracle breakpoints, the values of $\lambda$ at which the optimal arm changes. With $b_{\min}$ and $b_{\max}$ the smallest and largest breakpoint, the grid contains $0$ and $\infty$, 21 log-spaced values from $b_{\min}/10$ to $10\,b_{\max}$, and $0.9b$, $b$ and $1.1b$ for every breakpoint $b$.

For each dataset we give the $\lambda$ grid, a table of the arms and a figure with four panels. Panels (a)--(c) show the regret per example, $R_\lambda(t)/((1+\lambda)t)$, against the round $t$ at $\lambda=0$, at an intermediate $\lambda$ and at $\lambda=\infty$. The intermediate $\lambda$ is the grid value closest to $\sqrt{b_{\min}b_{\max}}$. Panel (d) shows the same quantity after $T=50$k rounds against $\lambda$, with the breakpoints as dotted vertical lines. Dividing by $1+\lambda$ makes the values at $\lambda=0$ and $\lambda=\infty$ the accuracy and fairness regret of Figure~\ref{fig:tradeoff-50k-no-delay}. All values are means over 10{,}000 runs; standard errors are smaller than the markers. As in Figure~\ref{fig:tradeoff-50k-no-delay}, Random is omitted because it coincides with Round-Robin.

\subsubsection{Classifier Arms for the Tabular Datasets}
\label{app:ml-arms}
Bank, Adult, Credit, ACSIncome, Coupon and COMPAS use the same 18 arms. Numeric features are standardized and categorical features one-hot encoded. All models are fitted on the training split.
\begin{itemize}
    \item \emph{Standard models (do not use $S$).} Gradient boosted trees (GBM), a random forest (RF, 200 trees of depth at most 12), logistic regression (LR) and a decision tree of depth 3 (Tree-d3).
    \item \emph{GBM with $S$ as a feature} (GBM+S).
    \item \emph{Operating points of the GBM score} (GBM-OP1 to GBM-OP6). Each uses one threshold for all subgroups: GBM-OP$k$ uses the most accurate threshold whose DP gap is at most $(1-k/7)$ times the GBM's. The thresholds are set on the stream population, so the six arms have evenly spaced gaps.
    \item \emph{Fairness in-processing} ($S$ is used only in training). LR-DP$\eta$, $\eta\in\{3,10,30,100\}$, adds $\eta$ times the squared difference between the subgroups' mean predicted probabilities to the logistic loss (for $M>2$, the sum over all pairs of subgroups). GBM-RW is a GBM trained with the reweighing of \citet{kamiran2012data}.
    \item \emph{Biased arms.} GBM-BiasedLabels is a GBM+S trained on labels in which 40\% of the positives of the subgroup with the lowest base rate are flipped to 0. Biased predicts $\hat y=S$ with probability 0.7 and is otherwise a fair coin (for $M>2$, it predicts 1 for the subgroup with the highest base rate).
\end{itemize}
The fairness arms do not use $S$ at decision time. With subgroup-specific thresholds, parity costs at most 0.1 accuracy points on Bank, Credit and Coupon, so such arms tie and there is no trade-off to learn.

\clearpage
\subsubsection{Twitter}
The TwitterAAE dataset \citep{keswani2021towards} includes 1471 Twitter posts (a subset of a larger dataset from \citet{davidson2017automated}) labeled for containing hate speech by 170 Mechanical Turk workers. The sensitive attribute - dialect (African American English (AAE) and non-AAE) - is annotated using a model from \citet{blodgett2017dataset}, and this is treated as ground truth. Since each worker doesn't label all the tweets the dataset is sparse - each tweet is labeled by around 4\% of experts. We randomly combine 18 expert predictions into an arm $a$, taking the majority vote as their prediction $\hat{y}_a$. This results in 10 arms, each with around 40\% missingness. Unavailable arms are dynamically excluded from the action space for all algorithms, and regret is measured against the best available arm in each round.

The 29 values of $\lambda$ are $\{0,\allowbreak 0.0235,\allowbreak 0.0318,\allowbreak 0.043,\allowbreak 0.0583,\allowbreak 0.0789,\allowbreak 0.107,\allowbreak 0.144,\allowbreak 0.196,\allowbreak 0.211,\allowbreak 0.235,\allowbreak 0.258,\allowbreak 0.265,\allowbreak 0.358,\allowbreak 0.485,\allowbreak 0.656,\allowbreak 0.889,\allowbreak 0.901,\allowbreak 1,\allowbreak 1.1,\allowbreak 1.2,\allowbreak 1.63,\allowbreak 2.2,\allowbreak 2.98,\allowbreak 4.04,\allowbreak 5.46,\allowbreak 7.4,\allowbreak 10,\allowbreak \infty\}$.

\begin{table}[h!]
\centering
\caption{The 10 arms of Twitter: accuracy, DP gap and share of tweets without a prediction (\%). $^\star$: optimal for some $\lambda$.}
\label{tab:twitter-arms}
\small
\begin{tabular}{lrrr}
\toprule
\textbf{Arm} & \textbf{Accuracy} & \textbf{DP gap} & \textbf{\% Missing} \\
\midrule
Arm 0 & 75.1 & 6.4 & 39.4 \\
Arm 1 & 79.5 & 5.2 & 37.5 \\
Arm 2$^\star$ & 81.5 & 6.2 & 45.8 \\
Arm 3$^\star$ & 76.5 & 1.2 & 45.4 \\
Arm 4$^\star$ & 83.0 & 12.6 & 39.8 \\
Arm 5 & 80.4 & 11.5 & 34.2 \\
Arm 6 & 75.4 & 4.5 & 41.4 \\
Arm 7 & 74.0 & 5.0 & 39.7 \\
Arm 8 & 77.5 & 4.4 & 38.5 \\
Arm 9 & 76.8 & 3.4 & 46.4 \\
\bottomrule
\end{tabular}
\end{table}

\begin{figure}[h!]
\centering
\includegraphics[width=\textwidth]{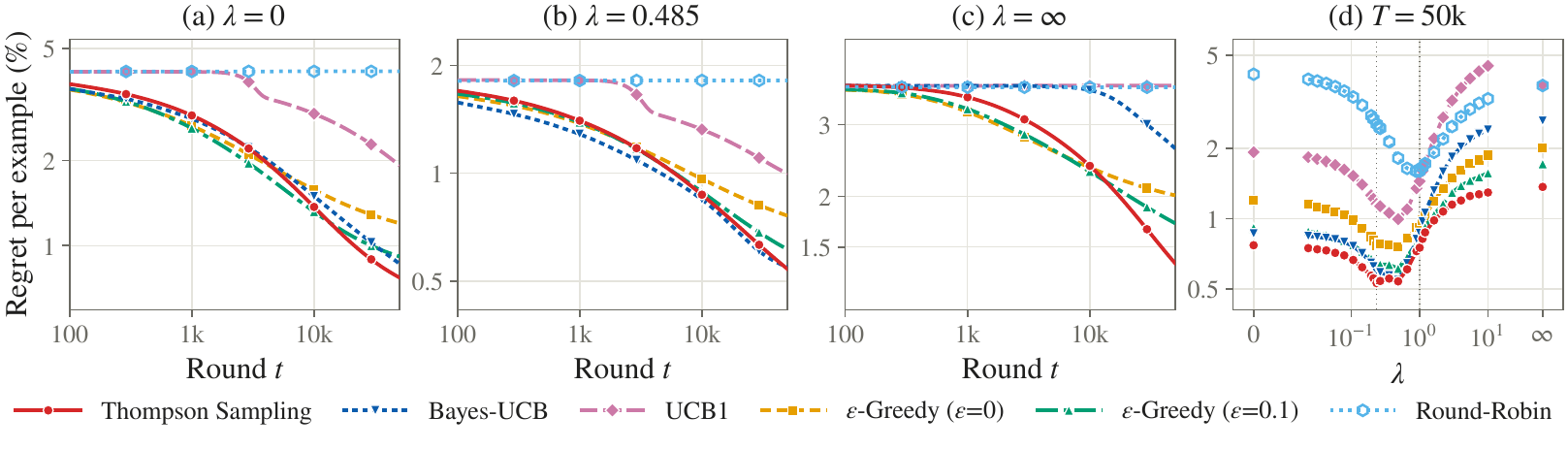}
\caption{Twitter: regret per example (\%) against the round $t$ at three values of $\lambda$ (a)--(c), and after $T=50$k rounds against $\lambda$ (d).}
\label{fig:regret-twitter}
\end{figure}
\clearpage
\subsubsection{Bank}
The Bank Marketing dataset \citep{bank_marketing_222} collects data from direct marketing campaigns of a Portuguese banking institution. The classification goal is to predict if the client will subscribe to a term deposit. The dataset does not contain any missing values. There are a total of 41,188 instances (24,712 used for training the arms and 16,476 used to create the online stream). The sensitive attribute is Marital status (Married vs Single/Divorced) - 60.4\% of the individuals are married and 39.6\% are Single/Divorced. 88.7\% have true label $Y=0$ and 11.3\% have true label $Y=1$. The arms use 19 features.

The 32 values of $\lambda$ are $\{0,\allowbreak 0.0211,\allowbreak 0.0281,\allowbreak 0.0373,\allowbreak 0.0495,\allowbreak 0.0658,\allowbreak 0.0874,\allowbreak 0.116,\allowbreak 0.154,\allowbreak 0.19,\allowbreak 0.205,\allowbreak 0.211,\allowbreak 0.232,\allowbreak 0.272,\allowbreak 0.361,\allowbreak 0.457,\allowbreak 0.48,\allowbreak 0.508,\allowbreak 0.556,\allowbreak 0.559,\allowbreak 0.618,\allowbreak 0.638,\allowbreak 0.68,\allowbreak 0.847,\allowbreak 1.13,\allowbreak 1.49,\allowbreak 1.99,\allowbreak 2.64,\allowbreak 3.5,\allowbreak 4.65,\allowbreak 6.18,\allowbreak \infty\}$.

\begin{table}[h!]
\centering
\caption{The 18 arms of Bank: accuracy and DP gap (\%) on the stream population. $^\star$: optimal for some $\lambda$. Arms are described in Appendix~\ref{app:ml-arms}.}
\label{tab:bank-arms}
\small
\begin{tabular}{lrr}
\toprule
\textbf{Arm} & \textbf{Accuracy} & \textbf{DP gap} \\
\midrule
GBM$^\star$ & 91.90 & 2.52 \\
RF & 91.26 & 1.75 \\
LR & 91.10 & 1.77 \\
Tree-d3 & 90.94 & 2.63 \\
GBM+S & 91.74 & 2.24 \\
GBM-OP1 & 91.67 & 2.16 \\
GBM-OP2 & 91.56 & 1.77 \\
GBM-OP3 & 91.44 & 1.43 \\
GBM-OP4 & 90.62 & 1.06 \\
GBM-OP5 & 90.26 & 0.70 \\
GBM-OP6 & 89.31 & 0.33 \\
LR-DP3 & 91.14 & 1.37 \\
LR-DP10 & 91.07 & 0.85 \\
LR-DP30$^\star$ & 91.01 & 0.33 \\
LR-DP100$^\star$ & 90.80 & 0.00 \\
GBM-RW$^\star$ & 91.74 & 1.77 \\
GBM-BiasedLabels & 90.44 & 8.68 \\
Biased & 43.25 & 69.71 \\
\bottomrule
\end{tabular}
\end{table}

\begin{figure}[h!]
\centering
\includegraphics[width=\textwidth]{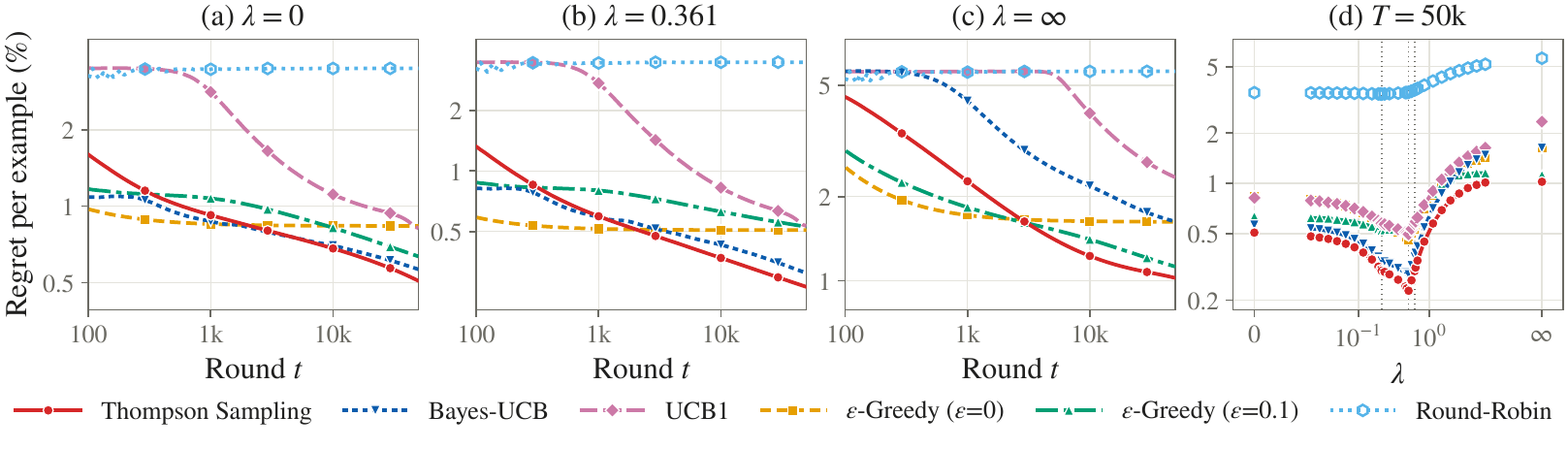}
\caption{Bank: regret per example (\%) against the round $t$ at three values of $\lambda$ (a)--(c), and after $T=50$k rounds against $\lambda$ (d).}
\label{fig:regret-bank}
\end{figure}
\clearpage
\subsubsection{Adult}
The Adult dataset \citep{adult_2} contains 48,842 records from the 1994 US Census (the UCI training and test files combined). The goal is to predict whether income exceeds \$50K. The sensitive attribute is race, white ($S=0$) vs.\ non-white ($S=1$). We drop race, the census sampling weight and the categorical education feature (it duplicates the numeric one), which leaves 11 features. 29,305 instances are used for training and 19,537 for the stream. 14.6\% of the stream population is non-white and 23.9\% has $Y=1$.

The 32 values of $\lambda$ are $\{0,\allowbreak 0.0023,\allowbreak 0.00332,\allowbreak 0.0048,\allowbreak 0.00694,\allowbreak 0.01,\allowbreak 0.0145,\allowbreak 0.0207,\allowbreak 0.021,\allowbreak 0.023,\allowbreak 0.0253,\allowbreak 0.0303,\allowbreak 0.0439,\allowbreak 0.0634,\allowbreak 0.0917,\allowbreak 0.133,\allowbreak 0.192,\allowbreak 0.235,\allowbreak 0.262,\allowbreak 0.277,\allowbreak 0.288,\allowbreak 0.33,\allowbreak 0.366,\allowbreak 0.401,\allowbreak 0.403,\allowbreak 0.58,\allowbreak 0.838,\allowbreak 1.21,\allowbreak 1.75,\allowbreak 2.53,\allowbreak 3.66,\allowbreak \infty\}$.

\begin{table}[h!]
\centering
\caption{The 18 arms of Adult: accuracy and DP gap (\%) on the stream population. $^\star$: optimal for some $\lambda$. Arms are described in Appendix~\ref{app:ml-arms}.}
\label{tab:adult-arms}
\small
\begin{tabular}{lrr}
\toprule
\textbf{Arm} & \textbf{Accuracy} & \textbf{DP gap} \\
\midrule
GBM & 87.47 & 9.46 \\
RF & 86.54 & 7.45 \\
LR & 85.43 & 9.15 \\
Tree-d3 & 84.62 & 6.23 \\
GBM+S$^\star$ & 87.49 & 9.92 \\
GBM-OP1$^\star$ & 87.33 & 8.09 \\
GBM-OP2 & 86.74 & 6.75 \\
GBM-OP3 & 85.72 & 5.40 \\
GBM-OP4 & 84.10 & 4.04 \\
GBM-OP5 & 81.82 & 2.69 \\
GBM-OP6 & 79.50 & 1.34 \\
LR-DP3 & 85.13 & 3.56 \\
LR-DP10 & 84.77 & 1.23 \\
LR-DP30$^\star$ & 84.49 & 0.35 \\
LR-DP100 & 84.37 & 0.74 \\
GBM-RW$^\star$ & 87.45 & 8.58 \\
GBM-BiasedLabels & 87.23 & 17.81 \\
Biased & 61.15 & 70.29 \\
\bottomrule
\end{tabular}
\end{table}

\begin{figure}[h!]
\centering
\includegraphics[width=\textwidth]{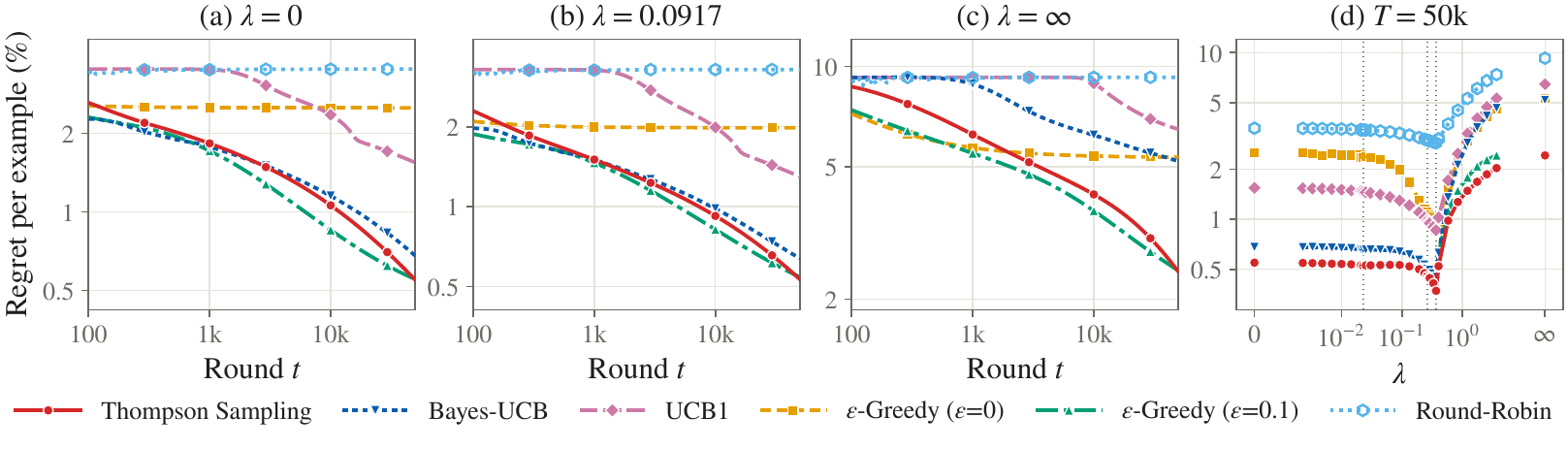}
\caption{Adult: regret per example (\%) against the round $t$ at three values of $\lambda$ (a)--(c), and after $T=50$k rounds against $\lambda$ (d).}
\label{fig:regret-adult}
\end{figure}
\clearpage
\subsubsection{Credit}
The Default of Credit Card Clients dataset \citep{yeh2009comparisons} was created to predict customers' credit default. There are a total of 30,000 instances. We use 18,000 to train arms and 12,000 to create the online stream. 60.8\% have sensitive attribute $S=1$ (female). 22.1\% have true label $Y=1$. The arms use 22 features.

The 29 values of $\lambda$ are $\{0,\allowbreak 0.274,\allowbreak 0.376,\allowbreak 0.515,\allowbreak 0.706,\allowbreak 0.967,\allowbreak 1.33,\allowbreak 1.82,\allowbreak 2.47,\allowbreak 2.49,\allowbreak 2.74,\allowbreak 3.02,\allowbreak 3.41,\allowbreak 4.67,\allowbreak 6.4,\allowbreak 8.77,\allowbreak 12,\allowbreak 13.4,\allowbreak 14.9,\allowbreak 16.4,\allowbreak 16.5,\allowbreak 22.6,\allowbreak 30.9,\allowbreak 42.4,\allowbreak 58.1,\allowbreak 79.6,\allowbreak 109,\allowbreak 149,\allowbreak \infty\}$.

\begin{table}[h!]
\centering
\caption{The 18 arms of Credit: accuracy and DP gap (\%) on the stream population. $^\star$: optimal for some $\lambda$. Arms are described in Appendix~\ref{app:ml-arms}.}
\label{tab:credit-arms}
\small
\begin{tabular}{lrr}
\toprule
\textbf{Arm} & \textbf{Accuracy} & \textbf{DP gap} \\
\midrule
GBM & 81.84 & 0.93 \\
RF & 81.57 & 0.87 \\
LR & 80.90 & 1.03 \\
Tree-d3 & 81.73 & 0.94 \\
GBM+S & 81.88 & 0.72 \\
GBM-OP1$^\star$ & 81.88 & 0.72 \\
GBM-OP2 & 81.47 & 0.57 \\
GBM-OP3 & 80.82 & 0.52 \\
GBM-OP4 & 79.89 & 0.39 \\
GBM-OP5 & 78.42 & 0.23 \\
GBM-OP6$^\star$ & 77.94 & 0.08 \\
LR-DP3 & 80.82 & 0.76 \\
LR-DP10 & 80.83 & 0.53 \\
LR-DP30$^\star$ & 80.61 & 0.26 \\
LR-DP100 & 80.12 & 0.50 \\
GBM-RW & 81.79 & 0.73 \\
GBM-BiasedLabels & 80.08 & 10.45 \\
Biased & 44.50 & 70.42 \\
\bottomrule
\end{tabular}
\end{table}

\begin{figure}[h!]
\centering
\includegraphics[width=\textwidth]{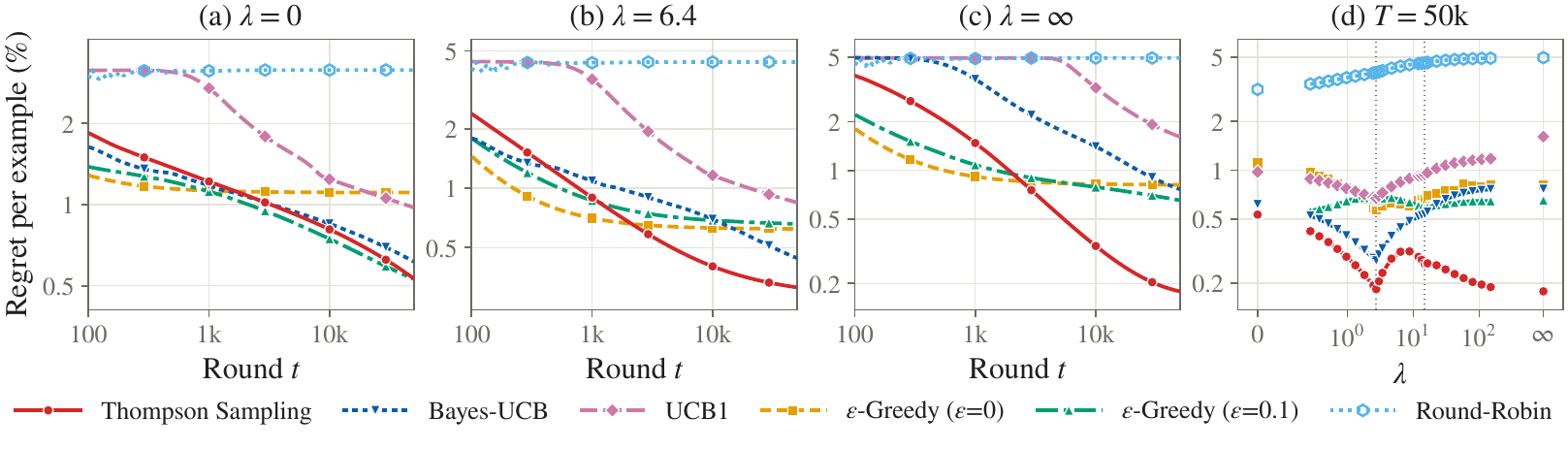}
\caption{Credit: regret per example (\%) against the round $t$ at three values of $\lambda$ (a)--(c), and after $T=50$k rounds against $\lambda$ (d).}
\label{fig:regret-credit}
\end{figure}
\clearpage
\subsubsection{Toxicity}
We parse raw runs from \citet{liang2023holistic} available at \url{https://crfm.stanford.edu/helm/classic/latest/#/runs}. Here, 66 different LLMs were queried - the inputs start with a comment (either toxic or non-toxic from the CivilComments dataset) followed by \textit{“\textbackslash n Question: Is the passage above toxic?”}. Outputs are either “True” or “False.” Accuracy is measured using quasi-exact match. Any LLMs with missing responses were discarded. We do not include perturbed instances. This leaves us with 56 LLMs (arms) and 2000 prompts (1000 prompts mention \textit{Black} demographics and 1000 mention \textit{White}). The motivation behind the dataset is that toxicity classifiers pick up on biases in the training data and spuriously associate toxicity with the mention of certain demographics. A fair classifier would have a low DP-gap since classifications of Toxic should be equal for both subgroups. 29.4\% of the prompts have true label Toxic. The marginal accuracy and DP-gap of the 56 LLMs used for experiments are summarized in Table \ref{tab:toxicity-arms}.

The 29 values of $\lambda$ are $\{0,\allowbreak 0.0134,\allowbreak 0.0175,\allowbreak 0.023,\allowbreak 0.0301,\allowbreak 0.0395,\allowbreak 0.0518,\allowbreak 0.0679,\allowbreak 0.089,\allowbreak 0.117,\allowbreak 0.12,\allowbreak 0.134,\allowbreak 0.147,\allowbreak 0.153,\allowbreak 0.2,\allowbreak 0.263,\allowbreak 0.27,\allowbreak 0.3,\allowbreak 0.33,\allowbreak 0.344,\allowbreak 0.451,\allowbreak 0.591,\allowbreak 0.775,\allowbreak 1.02,\allowbreak 1.33,\allowbreak 1.75,\allowbreak 2.29,\allowbreak 3,\allowbreak \infty\}$.

\begin{table}[h!]
\centering
\caption{The 56 LLMs of Toxicity: accuracy and DP gap (\%). $^\star$: optimal for some $\lambda$.}
\label{tab:toxicity-arms}
\resizebox{0.8\textwidth}{!}{%
\begin{tabular}{lrr}
\toprule
\textbf{Arm} & \textbf{Accuracy} & \textbf{DP gap} \\
\midrule
together\_opt-66b & 29.5 & 0.2 \\
cohere\_medium-20220720 & 29.7 & 1.0 \\
together\_opt-175b & 30.9 & 2.1 \\
together\_gpt-neox-20b & 31.2 & 1.3 \\
together\_gpt-j-6b & 32.1 & 2.1 \\
openai\_davinci & 34.9 & 11.4 \\
cohere\_large-20220720 & 37.3 & 16.8 \\
cohere\_xlarge-20220609 & 38.5 & 3.2 \\
together\_redpajama-incite-base-7b & 38.8 & 38.2 \\
cohere\_command-medium-beta & 38.8 & 6.9 \\
ai21\_j2-jumbo & 40.1 & 11.2 \\
ai21\_j1-jumbo & 41.5 & 9.7 \\
meta\_llama-13b & 43.0 & 16.7 \\
AlephAlpha\_luminous-base & 43.1 & 43.3 \\
meta\_llama-7b & 43.8 & 44.0 \\
lmsys\_vicuna-13b-v1.3 & 45.1 & 2.5 \\
mosaicml\_mpt-30b & 45.1 & 20.2 \\
openai\_curie & 46.2 & 24.2 \\
eleutherai\_pythia-6.9b & 46.5 & 14.2 \\
lmsys\_vicuna-7b-v1.3 & 47.2 & 27.2 \\
meta\_llama-2-13b & 47.9 & 15.5 \\
microsoft\_TNLGv2\_530B & 48.7 & 0.3 \\
meta\_llama-2-70b & 51.4 & 1.3 \\
microsoft\_TNLGv2\_7B & 51.7 & 53.7 \\
together\_redpajama-incite-base-3b-v1 & 53.8 & 76.1 \\
meta\_llama-65b & 55.5 & 9.5 \\
anthropic\_stanford-online-all-v4-s3 & 56.2 & 10.2 \\
together\_t5-11b & 58.0 & 39.2 \\
\bottomrule
\end{tabular}
\hfill
\begin{tabular}{lrr}
\toprule
\textbf{Arm} & \textbf{Accuracy} & \textbf{DP gap} \\
\midrule
cohere\_command-xlarge-beta & 58.7 & 5.1 \\
ai21\_j2-large & 59.7 & 7.4 \\
together\_redpajama-incite-instruct-7b & 60.7 & 9.3 \\
tiiuae\_falcon-40b-instruct & 60.9 & 28.7 \\
together\_bloom & 61.8 & 21.0 \\
eleutherai\_pythia-12b-v0 & 62.9 & 4.3 \\
openai\_text-davinci-003 & 65.5 & 10.7 \\
tiiuae\_falcon-40b & 65.8 & 27.6 \\
AlephAlpha\_luminous-supreme & 66.5 & 17.6 \\
openai\_ada & 67.8 & 22.5 \\
ai21\_j1-large & 67.8 & 2.5 \\
together\_yalm & 68.8 & 4.7 \\
cohere\_xlarge-20221108 & 69.2 & 0.3 \\
ai21\_j1-grande & 69.3 & 6.5 \\
AlephAlpha\_luminous-extended & 69.6 & 0.8 \\
openai\_babbage & 69.8 & 0.4 \\
tiiuae\_falcon-7b & 70.2 & 3.3 \\
cohere\_medium-20221108 & 70.5 & 0.1 \\
openai\_text-ada-001$^\star$ & 70.6 & 0.0 \\
together\_glm$^\star$ & 70.6 & 0.0 \\
cohere\_small-20220720$^\star$ & 70.6 & 0.0 \\
openai\_text-babbage-001$^\star$ & 70.6 & 0.0 \\
ai21\_j2-grande & 70.8 & 3.8 \\
meta\_llama-30b & 71.2 & 4.8 \\
tiiuae\_falcon-7b-instruct$^\star$ & 71.2 & 2.0 \\
ai21\_j1-grande-v2-beta & 71.3 & 3.8 \\
meta\_llama-2-7b & 72.0 & 23.1 \\
together\_redpajama-incite-instruct-3b-v1$^\star$ & 72.2 & 9.1 \\
\bottomrule
\end{tabular}}
\end{table}

\begin{figure}[h!]
\centering
\includegraphics[width=\textwidth]{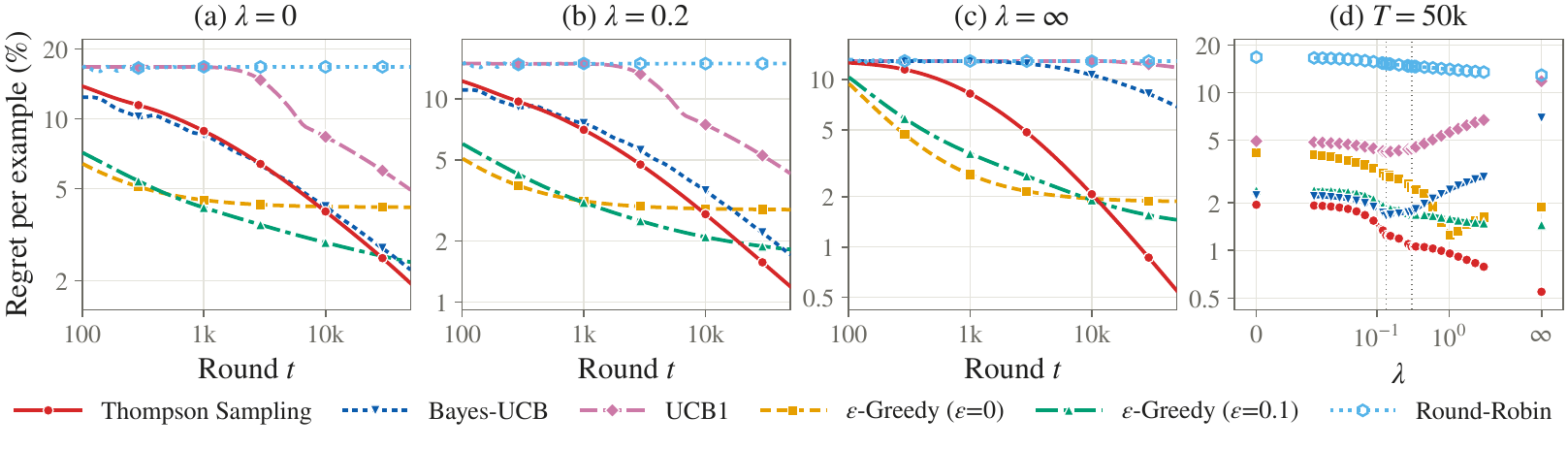}
\caption{Toxicity: regret per example (\%) against the round $t$ at three values of $\lambda$ (a)--(c), and after $T=50$k rounds against $\lambda$ (d).}
\label{fig:regret-toxicity}
\end{figure}
\clearpage
\subsubsection{Bios}\label{app:bios}
Bias in Bios \citep{de2019bias} contains online biographies labeled with one of 28 occupations, with the first sentence (which names the occupation) removed. The sensitive attribute is gender. We sample 30{,}000 biographies at random from the test split; 46.1\% are women, and the most frequent occupation (professor) accounts for 30.0\%. The arms are 16 open instruction-tuned LLMs with 0.36B to 9B parameters, fixed before any model was run. Each model receives a zero-shot prompt with the biography and the list of the 28 occupations, and its greedy decoding is constrained to one of the occupation names. The gap is the total variation distance between the subgroups' distributions of predicted occupations.

The 29 values of $\lambda$ are $\{0,\allowbreak 0.277,\allowbreak 0.36,\allowbreak 0.467,\allowbreak 0.607,\allowbreak 0.788,\allowbreak 1.02,\allowbreak 1.33,\allowbreak 1.72,\allowbreak 2.24,\allowbreak 2.49,\allowbreak 2.77,\allowbreak 2.91,\allowbreak 3.05,\allowbreak 3.77,\allowbreak 4.63,\allowbreak 4.9,\allowbreak 5.14,\allowbreak 5.65,\allowbreak 6.36,\allowbreak 8.26,\allowbreak 10.7,\allowbreak 13.9,\allowbreak 18.1,\allowbreak 23.5,\allowbreak 30.5,\allowbreak 39.6,\allowbreak 51.4,\allowbreak \infty\}$.

\begin{table}[h!]
\centering
\caption{The 16 arms of Bios: accuracy and TV gap (\%) on the stream population. $^\star$: optimal for some $\lambda$.}
\label{tab:bios-arms}
\small
\begin{tabular}{lrr}
\toprule
\textbf{Arm} & \textbf{Accuracy} & \textbf{TV gap} \\
\midrule
SmolLM2-360M & 13.9 & 24.3 \\
SmolLM2-1.7B$^\star$ & 28.7 & 13.3 \\
Qwen2.5-0.5B & 47.1 & 28.3 \\
Gemma3-1B & 49.9 & 22.3 \\
Qwen2.5-1.5B & 53.0 & 25.7 \\
OLMo2-7B & 56.8 & 28.7 \\
Gemma3-4B & 60.7 & 22.3 \\
Qwen3-4B & 61.9 & 24.7 \\
Qwen2.5-3B & 63.6 & 24.5 \\
Phi4-mini & 65.0 & 22.6 \\
Mistral-7B & 65.2 & 24.7 \\
Qwen3-8B$^\star$ & 66.8 & 20.7 \\
Gemma2-2B & 66.9 & 27.7 \\
Phi3.5-mini & 67.4 & 21.6 \\
Qwen2.5-7B & 68.1 & 24.6 \\
Gemma2-9B$^\star$ & 74.6 & 23.5 \\
\bottomrule
\end{tabular}
\end{table}

\begin{figure}[h!]
\centering
\includegraphics[width=\textwidth]{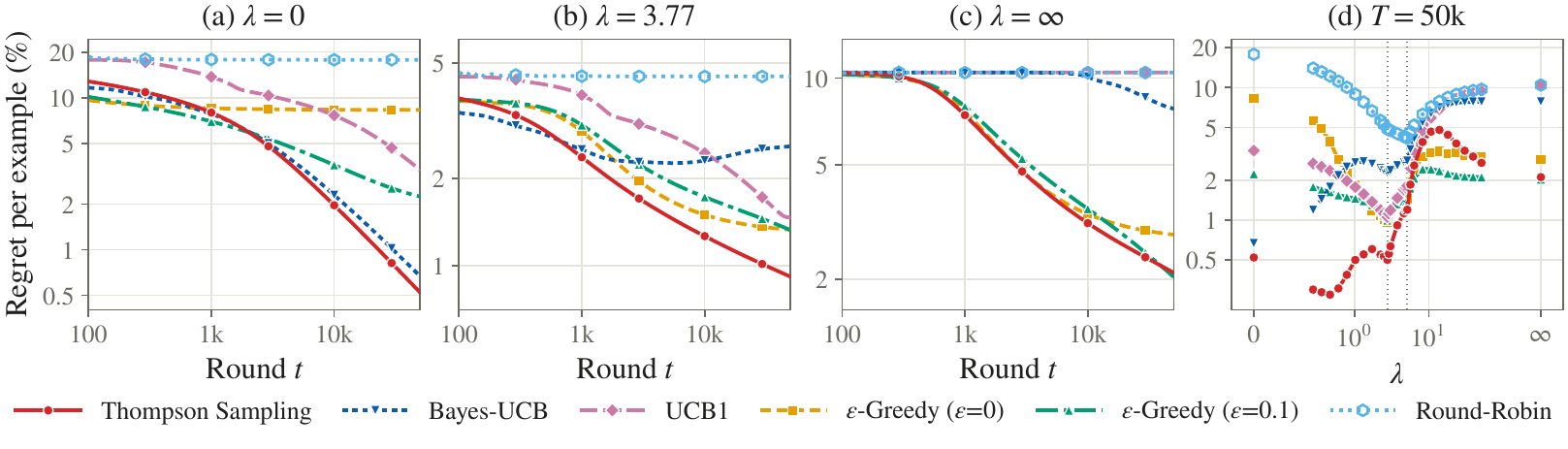}
\caption{Bios: regret per example (\%) against the round $t$ at three values of $\lambda$ (a)--(c), and after $T=50$k rounds against $\lambda$ (d).}
\label{fig:regret-bios}
\end{figure}
\clearpage
\subsubsection{ACSIncome}\label{app:acsincome}
ACSIncome \citep{ding2021retiring} is built from the 2018 American Community Survey (California), with 195,665 instances after the standard ACSIncome filter. The goal is to predict whether income exceeds \$50K (41.1\% of the stream population). The 9 features are age, class of worker, education, marital status, occupation, place of birth, relationship, hours worked per week and sex. The sensitive attribute is race/ethnicity in five groups: White (42.5\% of the stream population), Hispanic (33.1\%), Asian (16.6\%), Black (4.2\%) and Other (3.6\%). 117,399 instances are used for training and 78,266 for the stream. The DP gap is the largest difference in positive prediction rate between any two groups.

The 38 values of $\lambda$ are $\{0,\allowbreak 0.00893,\allowbreak 0.0132,\allowbreak 0.0194,\allowbreak 0.0287,\allowbreak 0.0423,\allowbreak 0.0624,\allowbreak 0.0804,\allowbreak 0.0893,\allowbreak 0.092,\allowbreak 0.0982,\allowbreak 0.136,\allowbreak 0.2,\allowbreak 0.222,\allowbreak 0.247,\allowbreak 0.272,\allowbreak 0.295,\allowbreak 0.372,\allowbreak 0.414,\allowbreak 0.436,\allowbreak 0.455,\allowbreak 0.643,\allowbreak 0.672,\allowbreak 0.746,\allowbreak 0.821,\allowbreak 0.948,\allowbreak 1.4,\allowbreak 1.91,\allowbreak 2.06,\allowbreak 2.13,\allowbreak 2.34,\allowbreak 3.04,\allowbreak 4.49,\allowbreak 6.63,\allowbreak 9.77,\allowbreak 14.4,\allowbreak 21.3,\allowbreak \infty\}$.

\begin{table}[h!]
\centering
\caption{The 18 arms of ACSIncome: accuracy and DP gap (\%) on the stream population. $^\star$: optimal for some $\lambda$. Arms are described in Appendix~\ref{app:ml-arms}.}
\label{tab:acsincome-arms}
\small
\begin{tabular}{lrr}
\toprule
\textbf{Arm} & \textbf{Accuracy} & \textbf{DP gap (max--min)} \\
\midrule
GBM$^\star$ & 82.22 & 35.33 \\
RF & 79.32 & 33.02 \\
LR & 81.84 & 36.06 \\
Tree-d3 & 74.65 & 31.81 \\
GBM+S & 82.20 & 40.51 \\
GBM-OP1 & 80.50 & 30.29 \\
GBM-OP2 & 77.53 & 25.22 \\
GBM-OP3 & 73.94 & 20.18 \\
GBM-OP4 & 70.28 & 15.14 \\
GBM-OP5 & 66.38 & 10.08 \\
GBM-OP6 & 61.91 & 5.02 \\
LR-DP3$^\star$ & 76.48 & 8.03 \\
LR-DP10$^\star$ & 75.08 & 4.64 \\
LR-DP30$^\star$ & 74.60 & 4.00 \\
LR-DP100$^\star$ & 74.42 & 3.91 \\
GBM-RW$^\star$ & 81.65 & 28.97 \\
GBM-BiasedLabels & 80.92 & 52.53 \\
Biased & 58.00 & 70.43 \\
\bottomrule
\end{tabular}
\end{table}

\begin{figure}[h!]
\centering
\includegraphics[width=\textwidth]{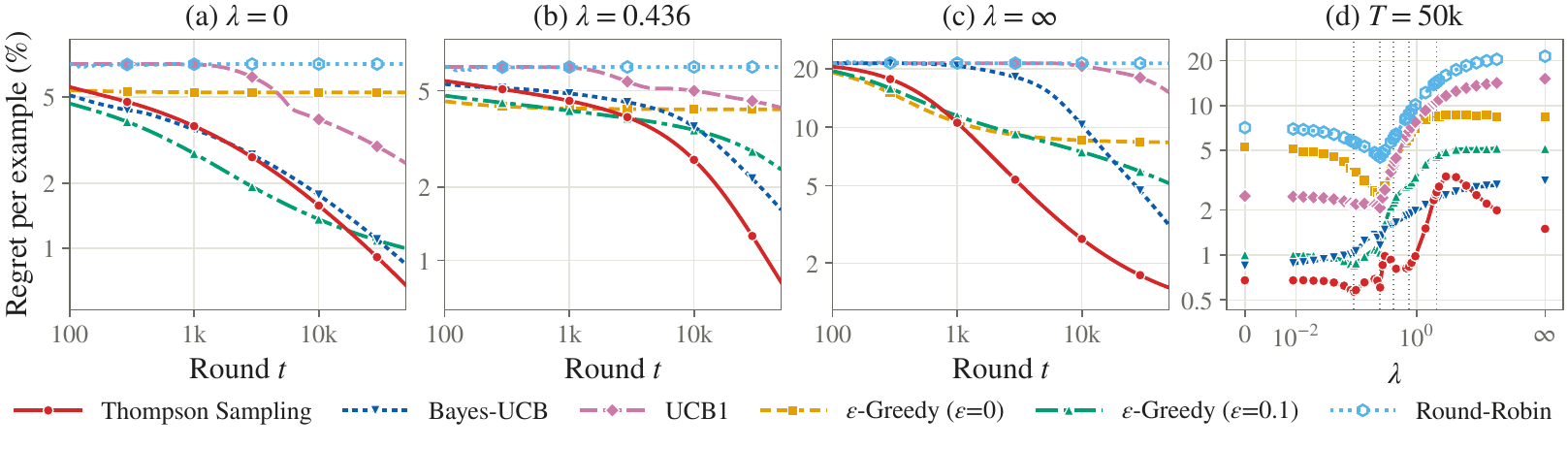}
\caption{ACSIncome: regret per example (\%) against the round $t$ at three values of $\lambda$ (a)--(c), and after $T=50$k rounds against $\lambda$ (d).}
\label{fig:regret-acsincome}
\end{figure}
\clearpage
\subsubsection{Coupon}
For the In-Vehicle Coupon Recommendation \citep{wang2017bayesian} dataset, driving scenarios are presented to Mechanical Turk workers, and they are asked whether they would accept a coupon if they were the driver. There are a total of 12,684 instances. 7,610 are used for training arms, and 5,074 are used for creating an online stream. 49.7\% of the test data have sensitive attribute $S=1$ (male). 56.8\% have true label $Y=1$. The arms use 23 features.

The 29 values of $\lambda$ are $\{0,\allowbreak 0.395,\allowbreak 0.548,\allowbreak 0.759,\allowbreak 1.05,\allowbreak 1.46,\allowbreak 2.02,\allowbreak 2.8,\allowbreak 3.56,\allowbreak 3.88,\allowbreak 3.95,\allowbreak 4.35,\allowbreak 5.38,\allowbreak 7.46,\allowbreak 10.3,\allowbreak 14.3,\allowbreak 19.9,\allowbreak 24.3,\allowbreak 27,\allowbreak 27.5,\allowbreak 29.7,\allowbreak 38.1,\allowbreak 52.9,\allowbreak 73.3,\allowbreak 102,\allowbreak 141,\allowbreak 195,\allowbreak 270,\allowbreak \infty\}$.

\begin{figure}[h!]
\centering
\begin{minipage}[b]{0.38\textwidth}
\centering
\captionof{table}{The 18 arms of Coupon: accuracy and DP gap (\%) on the stream population. $^\star$: optimal for some $\lambda$. Arms are described in Appendix~\ref{app:ml-arms}.}
\label{tab:coupon-arms}
\footnotesize
\begin{tabular}{lrr}
\toprule
\textbf{Arm} & \textbf{Accuracy} & \textbf{DP gap} \\
\midrule
GBM & 73.93 & 4.04 \\
RF & 73.51 & 2.47 \\
LR & 68.23 & 3.13 \\
Tree-d3 & 63.84 & 3.98 \\
GBM+S & 74.24 & 5.53 \\
GBM-OP1 & 73.18 & 3.37 \\
GBM-OP2 & 72.23 & 2.86 \\
GBM-OP3 & 69.49 & 2.22 \\
GBM-OP4 & 66.30 & 1.70 \\
GBM-OP5 & 65.10 & 1.08 \\
GBM-OP6$^\star$ & 58.32 & 0.57 \\
LR-DP3 & 67.95 & 2.11 \\
LR-DP10 & 67.93 & 1.52 \\
LR-DP30 & 67.90 & 1.04 \\
LR-DP100$^\star$ & 67.88 & 0.92 \\
GBM-RW$^\star$ & 74.28 & 2.54 \\
GBM-BiasedLabels & 65.92 & 48.86 \\
Biased & 51.46 & 69.60 \\
\bottomrule
\end{tabular}
\end{minipage}\hfill
\begin{minipage}[b]{0.58\textwidth}
\centering
\includegraphics[width=\linewidth]{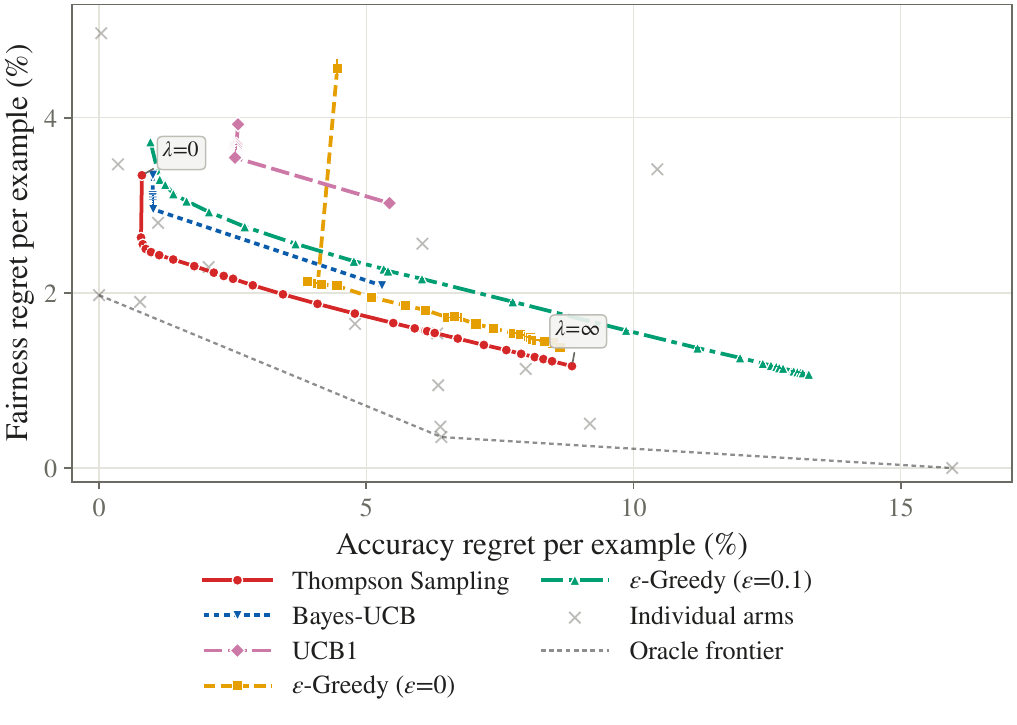}
\captionof{figure}{Trade-off between accuracy and fairness regret on Coupon after $T=50$k rounds, as in Figure~\ref{fig:tradeoff-50k-no-delay}. Round-Robin is off-scale and omitted.}
\label{fig:tradeoff-coupon}
\end{minipage}

\vspace{1.5em}
\includegraphics[width=\textwidth]{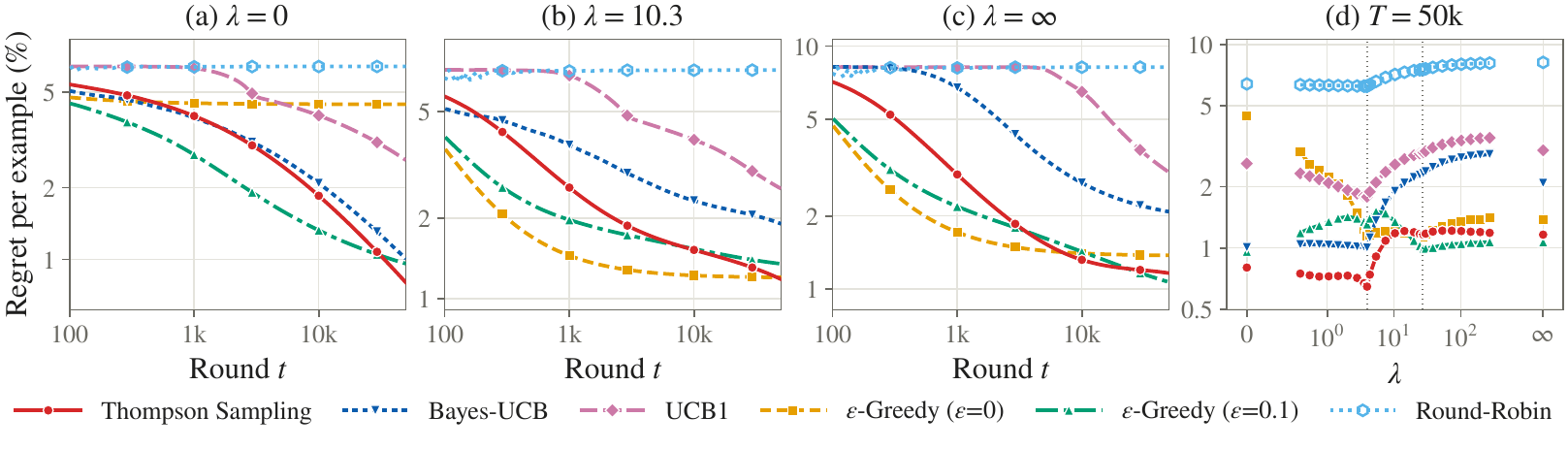}
\captionof{figure}{Coupon: regret per example (\%) against the round $t$ at three values of $\lambda$ (a)--(c), and after $T=50$k rounds against $\lambda$ (d).}
\label{fig:regret-coupon}
\end{figure}

\clearpage
\subsubsection{COMPAS}
The COMPAS dataset \citep{Larson_Angwin_Kirchner_Mattu_2016} contains criminal defendants from Broward County, Florida. The goal is to predict recidivism within two years. We apply the filters of the ProPublica analysis and keep African-American ($S=1$) and Caucasian ($S=0$) defendants, which leaves 5,278 instances. The 7 features are age, sex, number of prior offenses, three juvenile offence counts and the charge degree; the COMPAS score is not used. 3,166 instances are used for training and 2,112 for the stream. 39.6\% of the stream population is Caucasian and 47.1\% has $Y=1$.

The 44 values of $\lambda$ are $\{0,\allowbreak 0.00208,\allowbreak 0.00314,\allowbreak 0.00475,\allowbreak 0.00719,\allowbreak 0.0109,\allowbreak 0.0165,\allowbreak 0.0187,\allowbreak 0.0208,\allowbreak 0.0228,\allowbreak 0.0249,\allowbreak 0.0377,\allowbreak 0.046,\allowbreak 0.0512,\allowbreak 0.0563,\allowbreak 0.057,\allowbreak 0.0863,\allowbreak 0.131,\allowbreak 0.198,\allowbreak 0.222,\allowbreak 0.246,\allowbreak 0.271,\allowbreak 0.299,\allowbreak 0.34,\allowbreak 0.378,\allowbreak 0.416,\allowbreak 0.452,\allowbreak 0.479,\allowbreak 0.532,\allowbreak 0.585,\allowbreak 0.652,\allowbreak 0.684,\allowbreak 0.725,\allowbreak 0.739,\allowbreak 0.797,\allowbreak 0.821,\allowbreak 0.903,\allowbreak 1.04,\allowbreak 1.57,\allowbreak 2.37,\allowbreak 3.59,\allowbreak 5.43,\allowbreak 8.21,\allowbreak \infty\}$.

\begin{figure}[h!]
\centering
\begin{minipage}[b]{0.38\textwidth}
\centering
\captionof{table}{The 18 arms of COMPAS: accuracy and DP gap (\%) on the stream population. $^\star$: optimal for some $\lambda$. Arms are described in Appendix~\ref{app:ml-arms}.}
\label{tab:compas-arms}
\footnotesize
\begin{tabular}{lrr}
\toprule
\textbf{Arm} & \textbf{Accuracy} & \textbf{DP gap} \\
\midrule
GBM & 66.48 & 22.20 \\
RF & 66.34 & 22.92 \\
LR$^\star$ & 67.05 & 26.20 \\
Tree-d3$^\star$ & 67.00 & 23.92 \\
GBM+S$^\star$ & 66.95 & 23.00 \\
GBM-OP1 & 65.10 & 18.88 \\
GBM-OP2$^\star$ & 64.63 & 15.52 \\
GBM-OP3$^\star$ & 63.07 & 12.58 \\
GBM-OP4 & 60.75 & 9.47 \\
GBM-OP5$^\star$ & 58.38 & 6.12 \\
GBM-OP6$^\star$ & 55.87 & 3.06 \\
LR-DP3$^\star$ & 66.00 & 19.15 \\
LR-DP10 & 63.49 & 15.04 \\
LR-DP30 & 61.51 & 13.07 \\
LR-DP100 & 57.53 & 6.39 \\
GBM-RW & 65.67 & 19.60 \\
GBM-BiasedLabels & 64.11 & 43.19 \\
Biased & 54.02 & 70.68 \\
\bottomrule
\end{tabular}
\end{minipage}\hfill
\begin{minipage}[b]{0.58\textwidth}
\centering
\includegraphics[width=\linewidth]{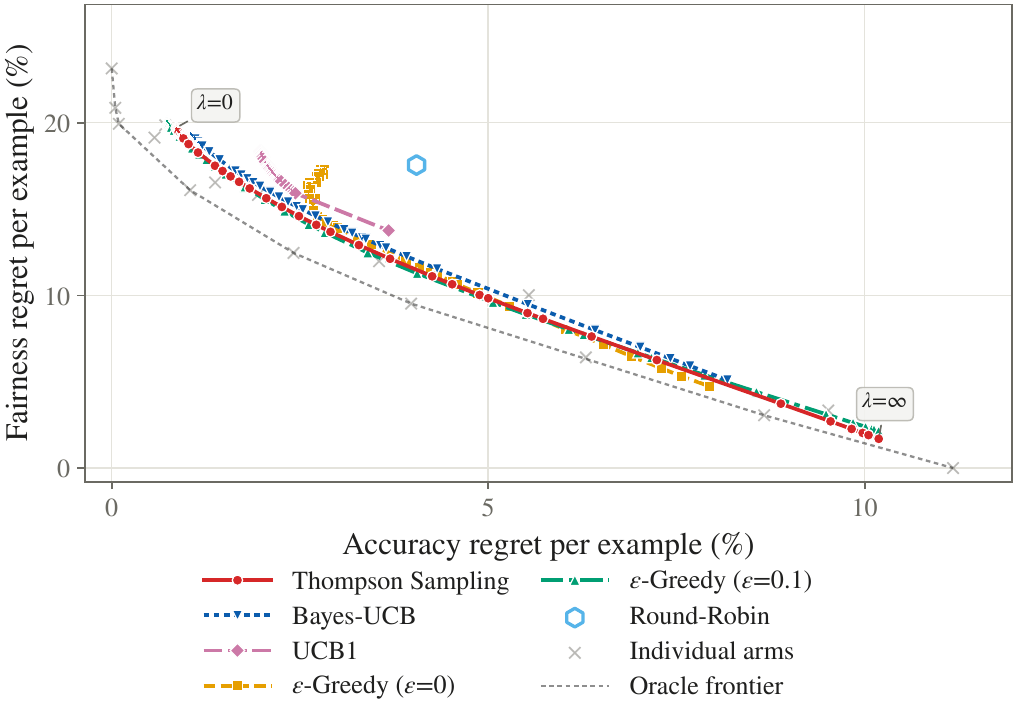}
\captionof{figure}{Trade-off between accuracy and fairness regret on COMPAS after $T=50$k rounds, as in Figure~\ref{fig:tradeoff-50k-no-delay}.}
\label{fig:tradeoff-compas}
\end{minipage}

\vspace{1.5em}
\includegraphics[width=\textwidth]{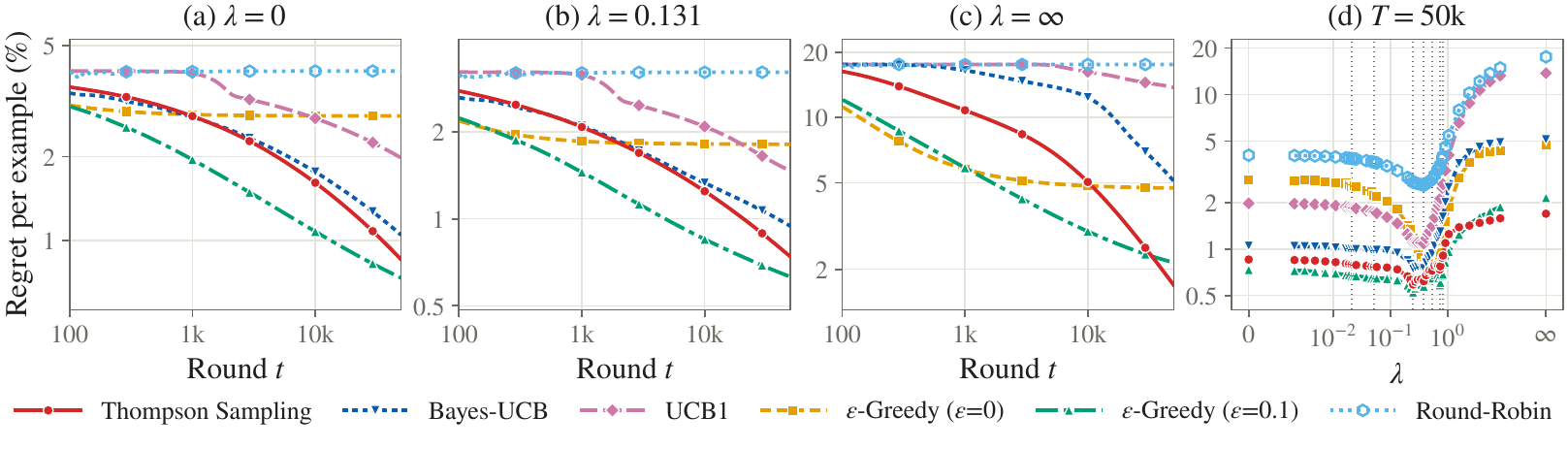}
\captionof{figure}{COMPAS: regret per example (\%) against the round $t$ at three values of $\lambda$ (a)--(c), and after $T=50$k rounds against $\lambda$ (d).}
\label{fig:regret-compas}
\end{figure}

\clearpage
\subsection{UCB Baselines}\label{app:UCB}

Both UCB baselines replace the unknown accuracy $\rho_a$ with an upper bound $\rho_a^+$ and the gap $g_a$ with a lower bound $g_a^-$. In round $t$, the learner selects the available arm with the largest optimistic utility
\begin{equation*}
\hat U_a(t) = \rho_a^+(t) - \lambda\, g_a^-(t) - c_a,
\end{equation*}
with $\hat U_a(t) = -g_a^-(t) - c_a$ for $\lambda=\infty$. Ties are broken uniformly at random.

\paragraph{UCB1} Let $n_a$ be the number of times arm $a$ has been selected, $n_{a,s}$ the number of those rounds with $S=s$, and $\hat\rho_a$ and $\hat\mu_{a,s}$ the empirical accuracy and subgroup-positive rate. With the Hoeffding radii $r_a=\sqrt{2\log t/n_a}$ and $r_{a,s}=\sqrt{2\log t/n_{a,s}}$ \citep{auer2002finite},
\begin{equation*}
\rho_a^+ = \min\big(1,\ \hat\rho_a + r_a\big), \qquad
g_a^- = \max\big(0,\ |\hat\mu_{a,0}-\hat\mu_{a,1}| - r_{a,0} - r_{a,1}\big).
\end{equation*}
An arm that has not been selected has an infinite index, so every arm is tried once. Until both subgroups of an arm have been observed, its gap bound is zero.

\paragraph{Bayes-UCB} Bayes-UCB keeps the Beta posteriors of TS (Algorithm~\ref{alg:ts}) and replaces each posterior sample with a quantile \citep{kaufmann2012bayesian}. The accuracy bound $\rho_a^+$ is the $1-1/t$ quantile of $\mathrm{Beta}(\alpha_a,\beta_a)$. For the gap, the $1/t$ and $1-1/t$ quantiles of $\mathrm{Beta}(\alpha_{a,s},\beta_{a,s})$ give an interval $[\ell_{a,s},u_{a,s}]$ for each subgroup, and $g_a^-=\max(0,\ \ell_{a,0}-u_{a,1},\ \ell_{a,1}-u_{a,0})$ is the smallest gap consistent with both intervals. We use the level $1-1/t$, i.e.\ $c=0$ in the level $1-1/(t(\log T)^c)$ of \citet{kaufmann2012bayesian}. Their regret bound requires $c\geq5$, but $c=0$ gave the best results in their simulations. To save computation, for $t>200$ we round $t$ in the level down to a geometric grid with ratio $1.01$, so a quantile is recomputed only when the arm's counts or the grid point change. With delayed labels, the posteriors are updated as in Algorithm~\ref{alg:ts}.

\paragraph{Multi-class and multi-subgroup} Both methods bound each rate separately. With $C$ classes, UCB1 uses the radii above for every class rate $\mu_{a,s,k}$, and Bayes-UCB uses the quantiles of the Beta marginals $\mathrm{Beta}(1+n_{a,s,k},\,C-1+n_{a,s}-n_{a,s,k})$ of the Dirichlet posterior. The gap bound is half the sum over classes of the smallest difference consistent with the per-class intervals. With $M$ subgroups, $g_a^-=\max\big(0,\ \max_s \ell_{a,s} - \min_s u_{a,s}\big)$, where $[\ell_{a,s},u_{a,s}]$ are the subgroup intervals of either method. For real-valued rewards, see Appendix~\ref{app:real}.

\clearpage
\subsection{Fewer Rounds}\label{app:fewer}

Figures~\ref{fig:fewer-2.5k}--\ref{fig:fewer-10k} repeat Figure~\ref{fig:tradeoff-50k-no-delay} after $T=2{,}500$, $5{,}000$ and $10{,}000$ rounds, the first rounds of the same runs. At these horizons $\epsilon$-Greedy has the lowest regret at most values of $\lambda$ on Twitter, Adult and Toxicity and at some values on Bios and ACSIncome, while TS has the lowest regret at nearly all values on Bank and Credit. The regret per example of TS keeps decreasing with $T$ (Appendix~\ref{app:datasets}), and after $T=50$k rounds TS has the lowest regret at every $\lambda$ on four datasets and at all but one on Credit (Table~\ref{tab:wins}).

\begin{figure}[h!]
\centering
\includegraphics[width=\textwidth]{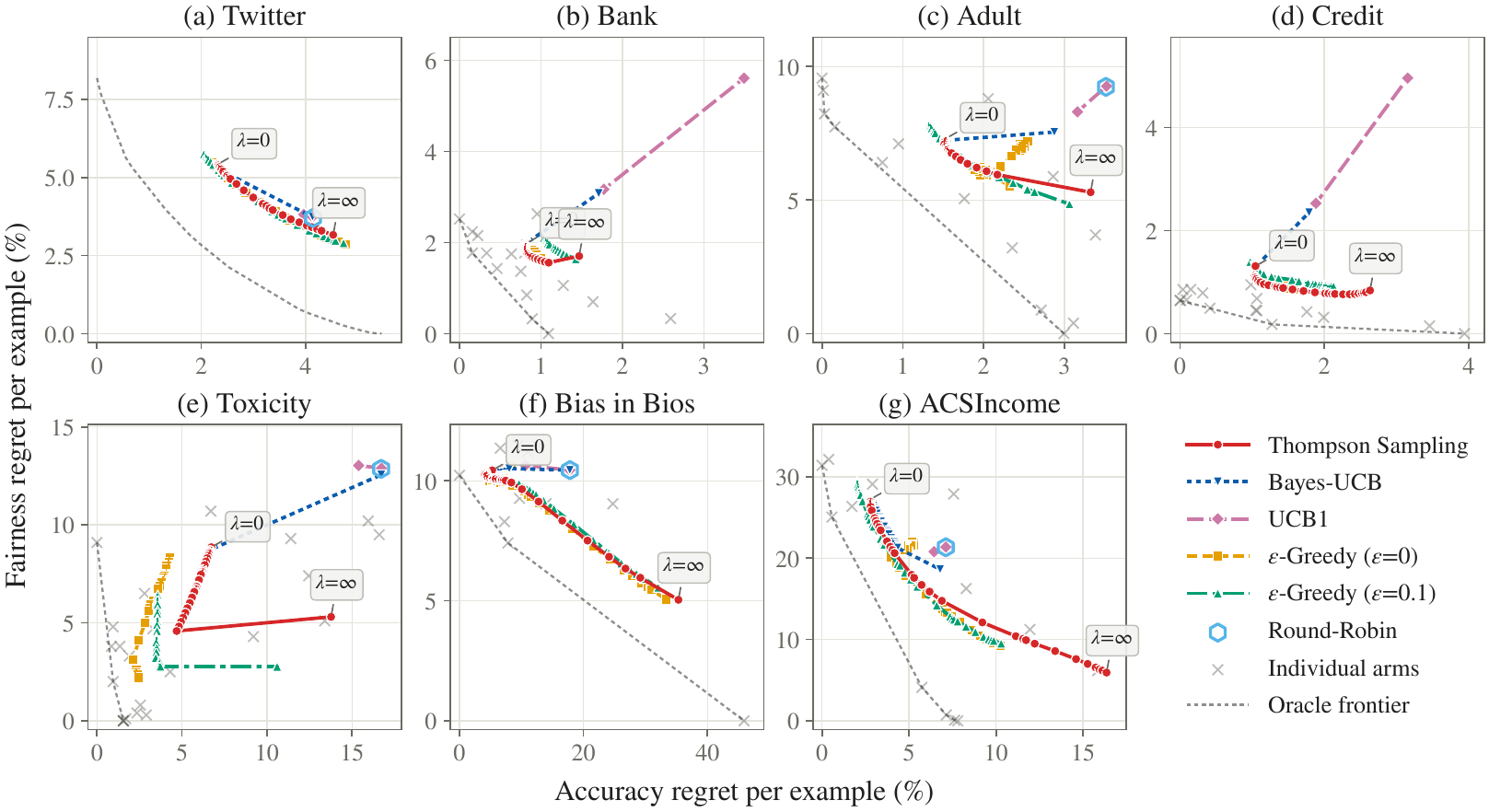}
\caption{Trade-off between accuracy and fairness regret after $T=2{,}500$ rounds, as in Figure~\ref{fig:tradeoff-50k-no-delay}.}
\label{fig:fewer-2.5k}
\end{figure}

\begin{figure}[h!]
\centering
\includegraphics[width=\textwidth]{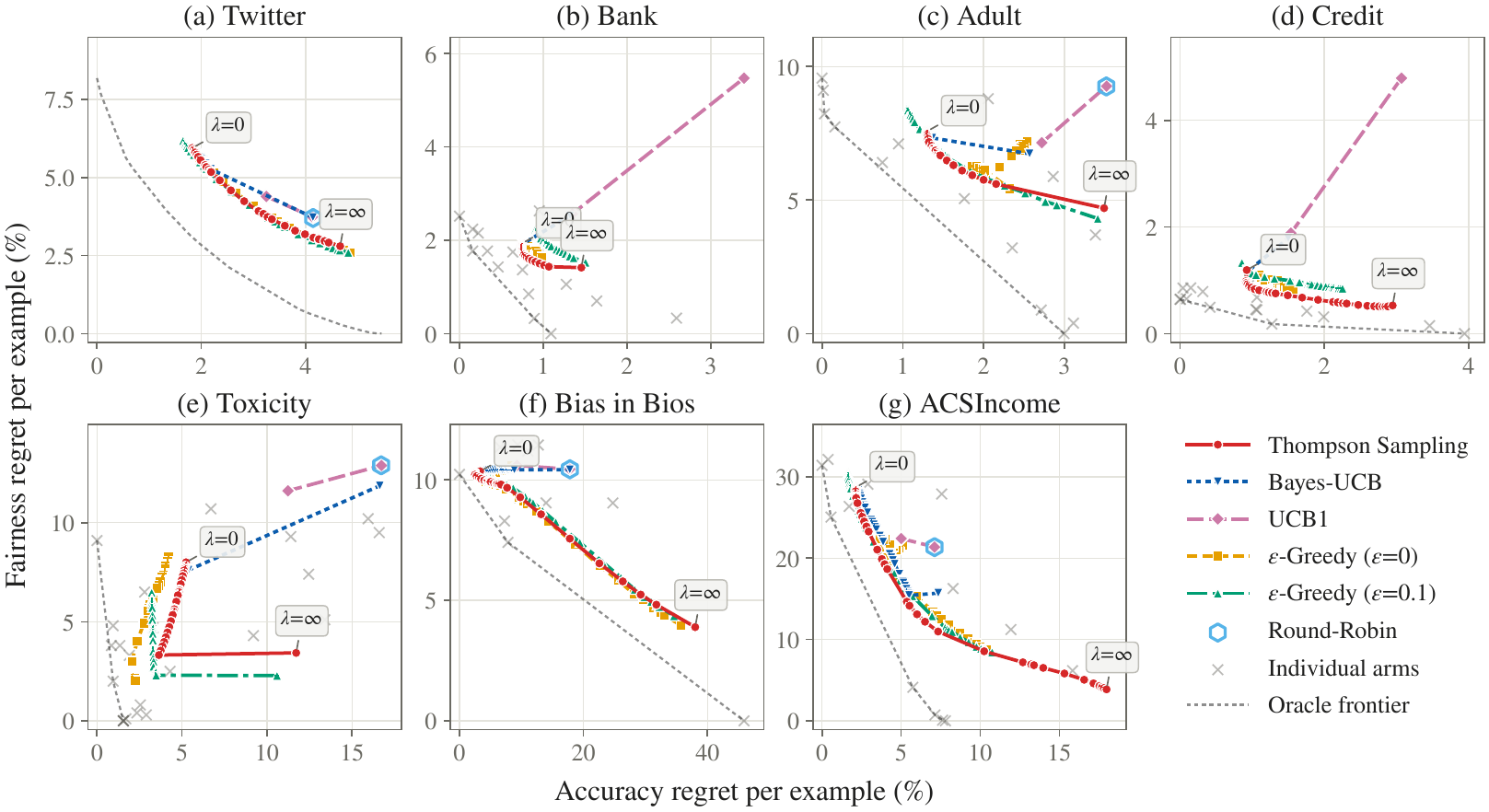}
\caption{Trade-off between accuracy and fairness regret after $T=5{,}000$ rounds, as in Figure~\ref{fig:tradeoff-50k-no-delay}.}
\label{fig:fewer-5k}
\end{figure}

\begin{figure}[h!]
\centering
\includegraphics[width=\textwidth]{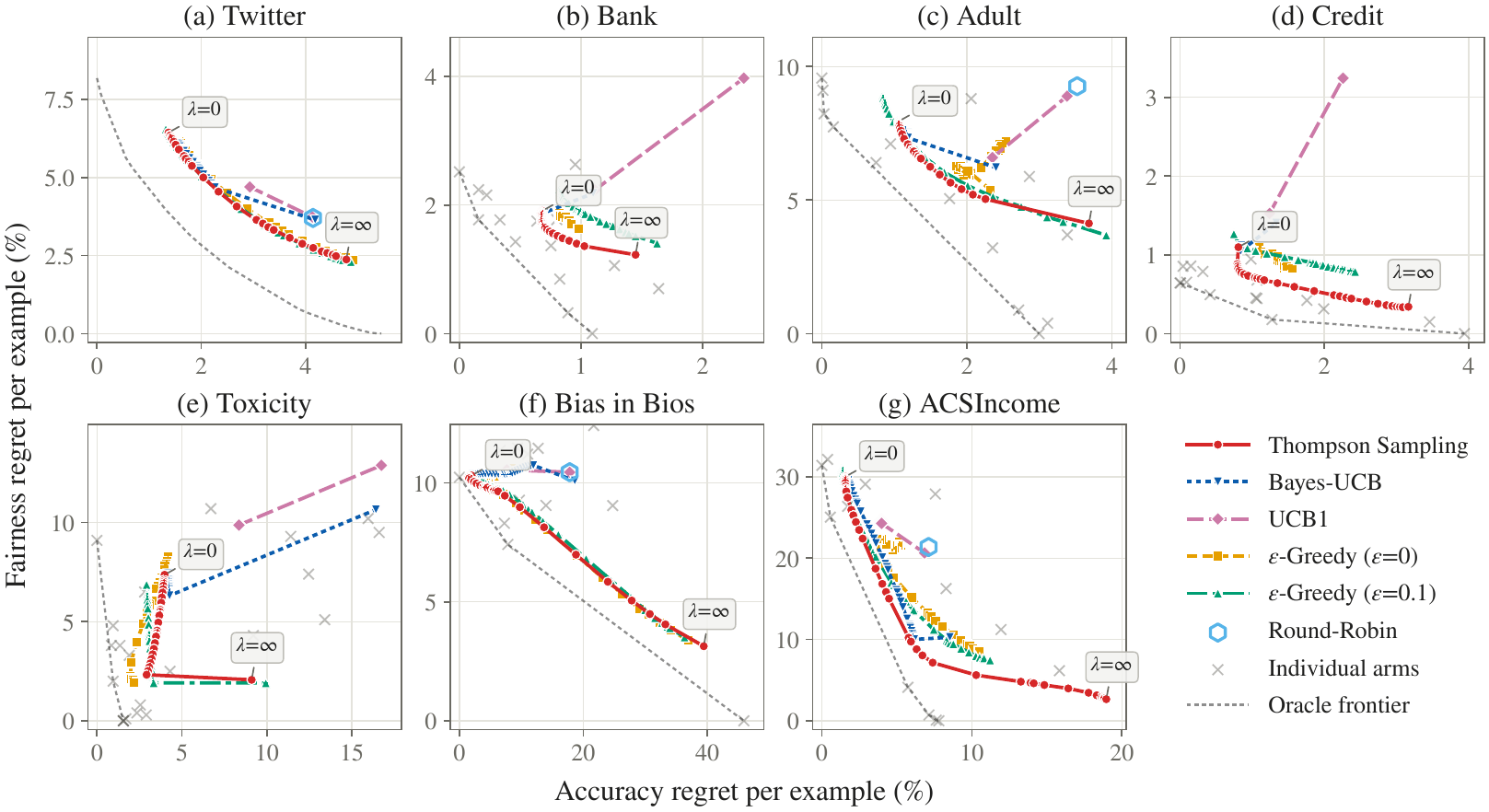}
\caption{Trade-off between accuracy and fairness regret after $T=10{,}000$ rounds, as in Figure~\ref{fig:tradeoff-50k-no-delay}.}
\label{fig:fewer-10k}
\end{figure}

\clearpage
\subsection{Cost}\label{app:cost}

Query costs $c_a$ are in the units of accuracy (\%) and are subtracted from the utility (Eq.~\ref{eqn:utility}). We use Coupon with $\lambda=1$ and five arms: four classifiers that use $S$ as a feature, and the Biased arm (Appendix~\ref{app:ml-arms}). Without costs, RandomForest has the highest utility, just above XGBoost (Table~\ref{tab:coupon-cost}). A cost of 5 for RandomForest and 1 for every other arm makes XGBoost the best arm by a clear margin.

TS follows the change (Figure~\ref{fig:coupon-cost}). Without costs, it splits its pulls between the two near-tied arms, mostly RandomForest. With costs, it selects XGBoost in 93\% of the last 25k rounds. In both settings, TS has significantly lower regret than every baseline (paired Wilcoxon signed-rank test, $p<0.01$).

\begin{table}[h!]
\centering
\caption{Coupon arms: accuracy, DP gap and utility $U_a$ at $\lambda=1$ (\%), without and with costs $c_a$. Bold: best arm.}
\label{tab:coupon-cost}
\begin{tabular}{lrrrrr}
\toprule
 & & & No cost & \multicolumn{2}{c}{With cost} \\
\cmidrule(lr){4-4}\cmidrule(l){5-6}
\textbf{Arm} & \textbf{Accuracy} & \textbf{DP gap} & $U_a$ & $c_a$ & $U_a$ \\
\midrule
XGBoost & 73.99 & 6.60 & 67.38 & 1 & \textbf{66.38} \\
RandomForest & 75.11 & 6.88 & \textbf{68.23} & 5 & 63.23 \\
LogisticRegression & 68.13 & 8.38 & 59.75 & 1 & 58.75 \\
DecisionStump & 61.67 & 0.36 & 61.31 & 1 & 60.31 \\
Biased & 50.57 & 69.88 & -19.31 & 1 & -20.31 \\
\bottomrule
\end{tabular}
\end{table}

\begin{figure}[h!]
\centering
\includegraphics[width=0.6\textwidth]{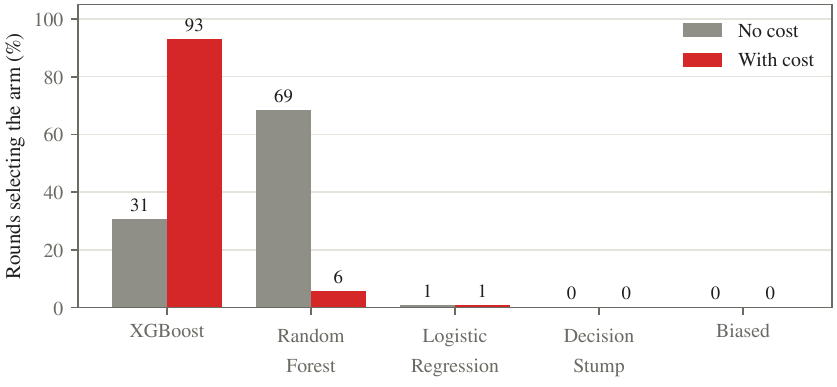}
\caption{Share of rounds in which TS selects each arm over rounds 25k--50k, averaged over 10{,}000 runs.}
\label{fig:coupon-cost}
\end{figure}

\clearpage
\subsection{Delayed Reward}\label{app:delay}

We draw each label delay $L_t$ from a geometric distribution with mean $\mathbb{E}[L_t]\in\{10,10^2,10^3,10^4\}$; labels that would arrive after $T=50{,}000$ rounds are never observed. TS follows Algorithm~\ref{alg:ts}, and $\epsilon$-Greedy and Bayes-UCB use the same updates. Random and Round-Robin do not use feedback and are unaffected, and UCB1 is not included. Results are averages over 10{,}000 runs, as in Section~\ref{sec:methodology}.

Delays do not change the fairness regret of TS at $\lambda=\infty$, and up to $\mathbb{E}[L_t]=10^3$ they increase its accuracy regret at $\lambda=0$ by at most 0.05 points (Figure~\ref{fig:delay-tradeoff}). TS has the lowest regret at most values of $\lambda$ on Twitter, Bank and Credit and at half or more on Adult; on Toxicity, $\epsilon$-Greedy ($\epsilon=0$) has lower regret at most values for $\mathbb{E}[L_t]\geq 100$, consistent with delayed labels acting as forced exploration for the greedy method.

\begin{figure}[h!]
\centering
\includegraphics[width=0.85\textwidth]{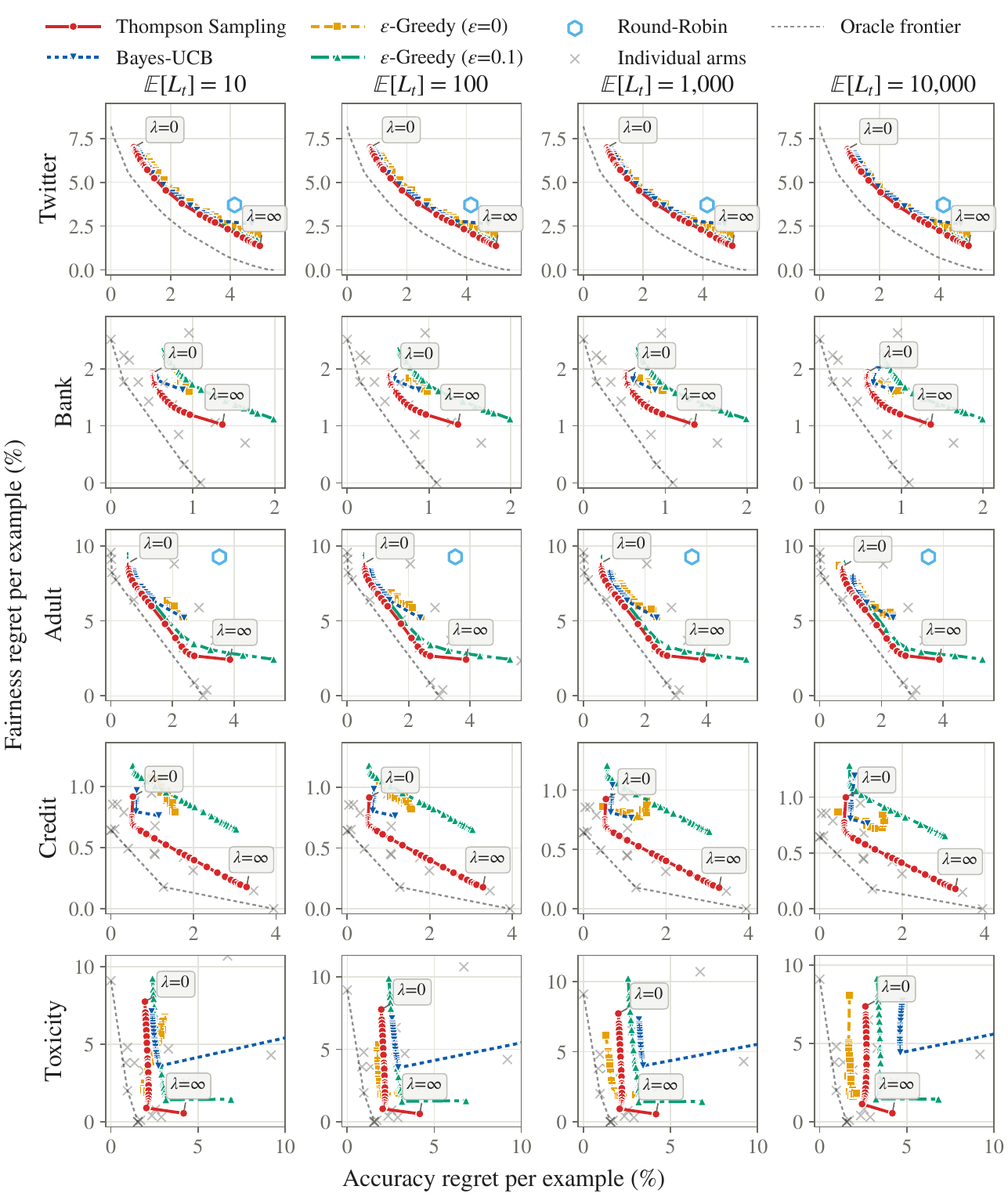}
\caption{Trade-off between accuracy and fairness regret after $T=50$k rounds with geometric label delays of mean $\mathbb{E}[L_t]$ (columns), using Algorithm~\ref{alg:ts}. Round-Robin does not use feedback and is taken from the runs without delay; it is omitted where it is off-scale.}
\label{fig:delay-tradeoff}
\end{figure}

\clearpage
\subsection{Multi-Class Classification}\label{app:multi-class}

The accuracy reward is still binary ($\hat y = y$ or not), so the accuracy posterior stays Beta-Bernoulli. For fairness, $\mu_{a,s}$ is the distribution of arm $a$'s predicted classes in subgroup $s$, with a Dirichlet posterior over the $C$ classes. We sample $\tilde{\mu}_{a,0}$ and $\tilde{\mu}_{a,1}$ and take as the gap their total variation distance, half their $L_1$ distance \citep{xian2023fair}. For $C=2$ this is the binary DP gap.

Bios (Figure~\ref{fig:tradeoff-50k-no-delay}(f), Appendix~\ref{app:bios}) has LLMs as arms. For our second multi-class dataset, ACSIncome-5C, we use trained classifiers. We use the ACSIncome data of Appendix~\ref{app:acsincome} with sex as the sensitive attribute and income in five classes, as in \citet{xian2023fair}: below \$15K, \$15K--30K, \$30K--48.6K, \$48.6K--78K, and at least \$78K. The 9 features are those of Appendix~\ref{app:acsincome}, with race in place of sex. The 60/40 split is stratified on the class. The stream population has 78,266 instances, 47.0\% of them women.

The 18 arms follow Appendix~\ref{app:ml-arms}, adapted to five classes:
\begin{itemize}
\item each operating point adds a fixed offset per class to the GBM's log-probabilities, the same in both subgroups, and predicts the largest;
\item LR-DP$\eta$ is multinomial logistic regression with the penalty $\eta\sum_k(\bar p_{k,0}-\bar p_{k,1})^2$, where $\bar p_{k,s}$ is the mean predicted probability of class $k$ in subgroup $s$;
\item GBM-BiasedLabels is trained after 40\% of the lower-income subgroup's labels above the lowest class are moved down one class;
\item Biased predicts the highest class for one subgroup and the lowest for the other with probability 0.7, and a uniformly random class otherwise.
\end{itemize}
The 35 values of $\lambda$ are $\{0,\allowbreak 0.0065,\allowbreak 0.0109,\allowbreak 0.0184,\allowbreak 0.0309,\allowbreak 0.052,\allowbreak 0.0585,\allowbreak 0.065,\allowbreak 0.0715,\allowbreak 0.0873,\allowbreak 0.147,\allowbreak 0.247,\allowbreak 0.415,\allowbreak 0.469,\allowbreak 0.521,\allowbreak 0.573,\allowbreak 0.698,\allowbreak 1.17,\allowbreak 1.97,\allowbreak 3.32,\allowbreak 5.57,\allowbreak 6.06,\allowbreak 6.73,\allowbreak 7.41,\allowbreak 9.37,\allowbreak 15.8,\allowbreak 19,\allowbreak 21.2,\allowbreak 23.3,\allowbreak 26.5,\allowbreak 44.5,\allowbreak 74.8,\allowbreak 126,\allowbreak 212,\allowbreak \infty\}$.

After $T=50$k rounds, TS has significantly lower regret than every baseline at all 35 values of $\lambda$ (paired Wilcoxon signed-rank test, $p<0.01$). At $\lambda=\infty$ its fairness regret is 0.35\%, against 1.55\% for the next best method, $\epsilon$-Greedy ($\epsilon=0.1$). $\epsilon$-Greedy has lower regret at some values of $\lambda$ only in the first 10{,}000 rounds (Figure~\ref{fig:acsincome5c}(a)--(c)).

\begin{figure}[h!]
\centering
\includegraphics[width=\textwidth]{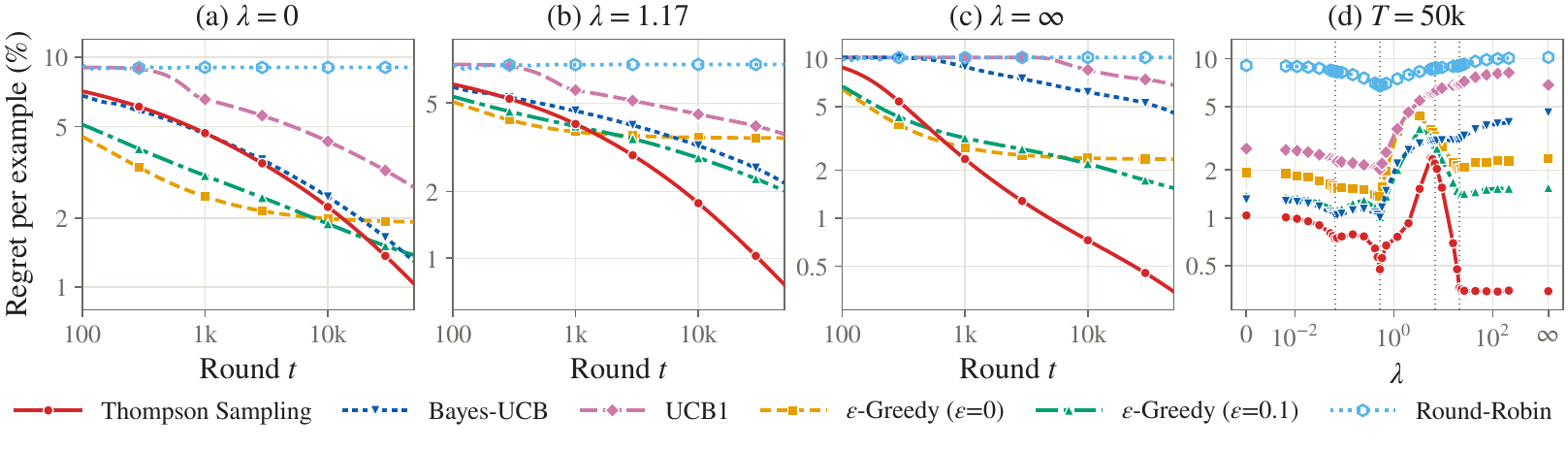}
\caption{ACSIncome-5C: regret per example (\%) against the round $t$ at three values of $\lambda$ (a)--(c), and after $T=50$k rounds against $\lambda$ (d), as in Appendix~\ref{app:datasets}.}
\label{fig:acsincome5c}
\end{figure}

\begin{figure}[h!]
\centering
\begin{minipage}[b]{0.38\textwidth}
\centering
\captionof{table}{The 18 arms of ACSIncome-5C: accuracy and TV gap (\%) on the stream population. $^\star$: optimal for some $\lambda$.}
\label{tab:acsincome5c-arms}
\footnotesize
\begin{tabular}{lrr}
\toprule
\textbf{Arm} & \textbf{Accuracy} & \textbf{TV gap} \\
\midrule
GBM & 53.17 & 9.37 \\
RF & 46.14 & 7.42 \\
LR & 51.75 & 8.85 \\
Tree-d3 & 43.80 & 13.75 \\
GBM+S$^\star$ & 53.61 & 14.80 \\
GBM-OP1 & 49.13 & 8.00 \\
GBM-OP2 & 41.96 & 6.69 \\
GBM-OP3 & 37.51 & 5.35 \\
GBM-OP4 & 34.22 & 3.93 \\
GBM-OP5 & 31.86 & 2.67 \\
GBM-OP6$^\star$ & 28.67 & 1.31 \\
LR-DP3 & 51.33 & 6.08 \\
LR-DP10 & 50.53 & 3.78 \\
LR-DP30$^\star$ & 50.05 & 2.34 \\
LR-DP100$^\star$ & 49.82 & 2.31 \\
GBM-RW$^\star$ & 53.20 & 8.39 \\
GBM-BiasedLabels & 50.46 & 31.55 \\
Biased & 24.93 & 70.10 \\
\bottomrule
\end{tabular}
\end{minipage}\hfill
\begin{minipage}[b]{0.58\textwidth}
\centering
\includegraphics[width=\linewidth]{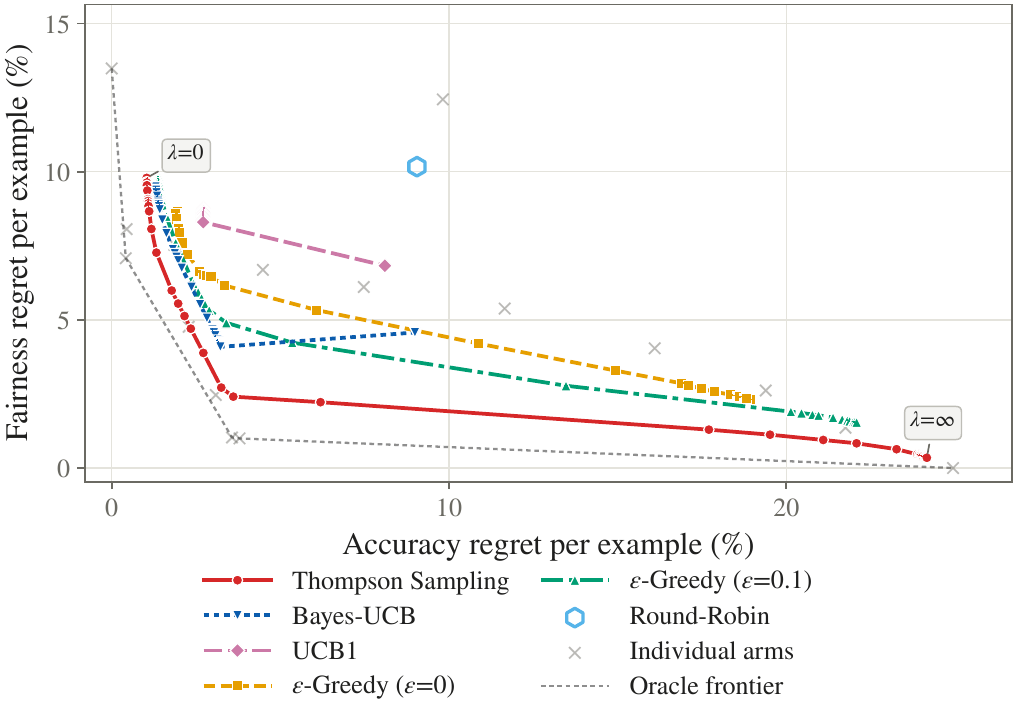}
\captionof{figure}{Trade-off between accuracy and fairness regret on ACSIncome-5C after $T=50$k rounds, as in Figure~\ref{fig:tradeoff-50k-no-delay}.}
\label{fig:acsincome5c-tradeoff}
\end{minipage}
\end{figure}

\clearpage
\subsection{Real-Valued Reward}\label{app:real}

\paragraph{Setup} We use Adult with the same sensitive attribute as in the main experiments. Each arm outputs a risk score $p$ rounded to the grid $\{0,0.05,\dots,1\}$ ($K=20$), so scores and rewards each take 21 values. The reward is $1-(Y-p)^2$, one minus the Brier score, and the fairness gap is the DP gap of the mean scores (Section~\ref{sec:proposed_algo}). The 17 arms are the 11 classifiers of the main Adult arm set that output probabilities, and 6 arms that shrink the gradient boosted trees' score toward its mean by increasing amounts, ending with a constant predictor (zero gap). TS and $\epsilon$-Greedy use the grid Dirichlet posteriors of Section~\ref{sec:proposed_algo}, with one pseudo-count at 0 and one at 1, rather than continuous non-parametric TS. UCB1 adds Hoeffding radii to the empirical means of the rewards and scores, and Bayes-UCB uses quantiles of the same Dirichlet posteriors, computed with a saddlepoint approximation. As in the main experiments, results are averaged over 10{,}000 runs of $T=50{,}000$ rounds, with 29 values of $\lambda$ and no costs or delays.

\begin{figure}[h!]
    \centering
    \includegraphics[width=0.75\textwidth]{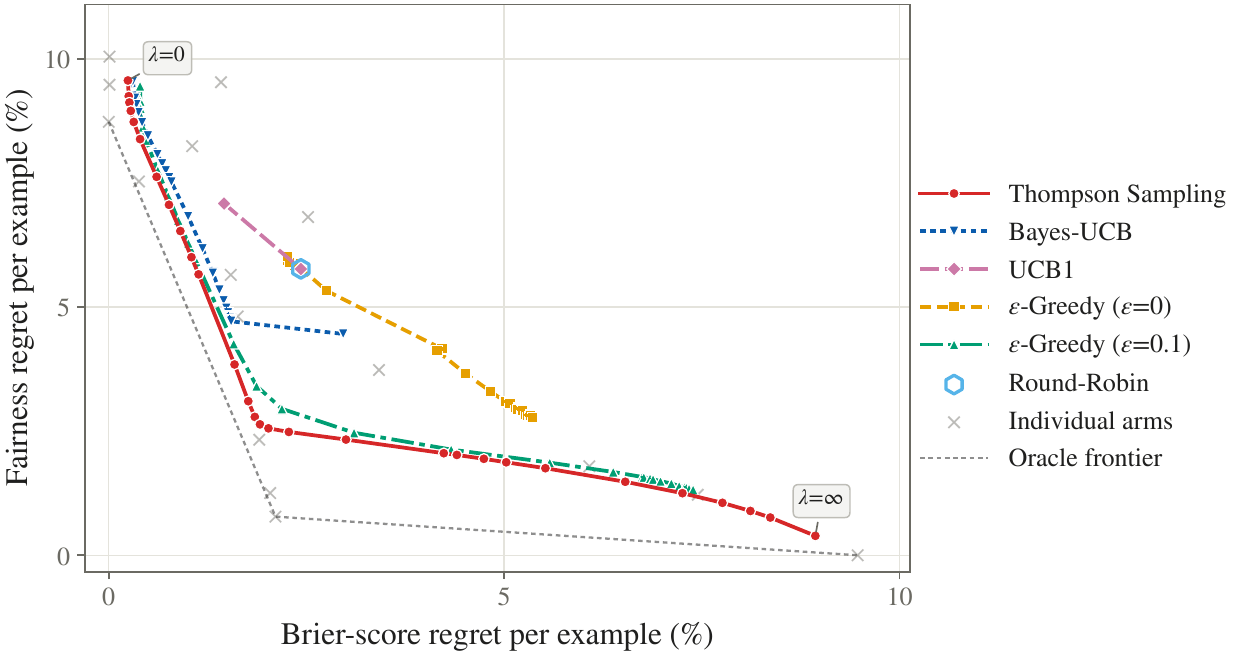}
    \caption{Trade-off between Brier-score regret (x-axis) and fairness regret (y-axis) on Adult with risk-score arms after $T=50$k rounds, as average regret per example (\%) over 10{,}000 runs. Each point on a curve is one value of $\lambda$. Grey crosses are the individual arms and the dotted line is the oracle frontier.}
    \label{fig:brier}
\end{figure}

\paragraph{Results} TS has the lowest regret at both ends of the trade-off (Figure~\ref{fig:brier}); at $\lambda=\infty$ its fairness regret is 0.39\%, against 1.32\% for the next best method, $\epsilon$-Greedy ($\epsilon=0.1$). At every value of $\lambda$, TS has significantly lower regret than UCB1, Bayes-UCB, $\epsilon$-Greedy ($\epsilon=0$) and Round-Robin. $\epsilon$-Greedy ($\epsilon=0.1$) has lower regret than TS only for large $\lambda$ near the value at which the optimal arm changes from a DP-penalized logistic regression to the constant predictor; there the two arms have nearly equal utility, so separating them requires many samples. UCB1 and Bayes-UCB stay far from the frontier at large $\lambda$, because their lower bound on the gap remains zero until the subgroup confidence intervals separate.

\clearpage
\subsection{\texorpdfstring{$\pi_{\min}$ Simulations}{pi\_min Simulations}}\label{app:pi_min_experiment}

The regret bounds of Theorem~1 grow with $1/\pi_{\min}$. To test this dependence, each round draws an instance of the smaller subgroup with probability $\pi_{\min}$ and of the larger subgroup otherwise, uniformly from the stream population. We use $1/\pi_{\min}\in\{2,5,10,20,30,50,100\}$ and $\lambda=2$, and compare TS with Bayes-UCB and $\epsilon$-Greedy ($\epsilon=0$ and $0.1$) over 10{,}000 runs of $T=50$k rounds. On real data, reweighting the subgroups also changes the arms' utilities. We therefore add a synthetic instance with four arms whose accuracy is the same in both subgroups, so their utilities do not depend on $\pi_{\min}$.

The regret of TS increases with $1/\pi_{\min}$, consistent with Theorem~1, except for a small dip at $1/\pi_{\min}=5$ on Twitter and Credit (Figure~\ref{fig:pi_min_all}). TS has significantly lower regret than all three baselines in 37 of the 56 settings. The exception is $\epsilon$-Greedy ($\epsilon=0.1$) at small $\pi_{\min}$ on Twitter, Adult and COMPAS. There, few instances of the smaller subgroup arrive in $T=50$k rounds, and no method has converged.

\begin{figure}[h!]
\centering
\includegraphics[width=\textwidth]{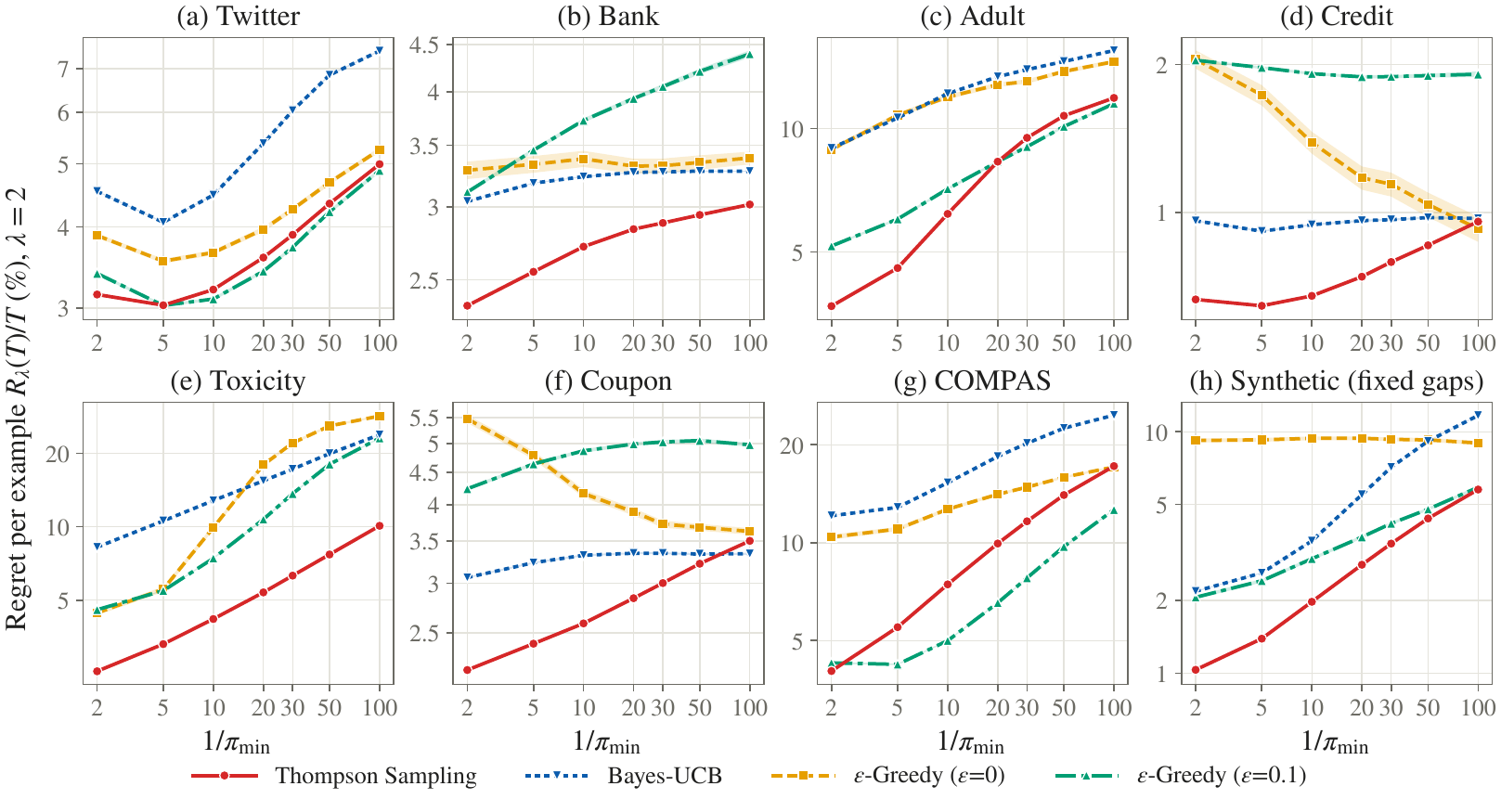}
\caption{Regret per example after $T=50$k rounds ($\lambda=2$) against $1/\pi_{\min}$, averaged over 10{,}000 runs. (a)--(g) Datasets of Appendix~\ref{app:datasets}; (h) synthetic instance.}
\label{fig:pi_min_all}
\end{figure}

\clearpage
\subsection{Fewer Arms}\label{app:fewer-arms}

To test whether the comparison depends on the choice of arms, we draw 30 random subsets of 3 and 30 of 5 of the 18 tabular arms (Appendix~\ref{app:ml-arms}) and run each on Bank, Adult and Credit (400 runs, $T=50{,}000$). The number of subsets that contain a biased arm (GBM-BiasedLabels or Biased) is fixed at its expected value. TS is compared with Bayes-UCB and both $\epsilon$-Greedy variants.

Figure~\ref{fig:fewer-arms-subsets} shows, for each subset, the percentage of $\lambda$ values at which TS has the lowest regret. For the median subset, TS has the lowest regret at most values of $\lambda$ on Bank and Adult, and at fewer than half on Credit. The share is highest when the subset contains a biased arm. $\epsilon$-Greedy (0.1) pulls a random arm in 10\% of the rounds, so it keeps paying for the poor arm, while TS stops pulling it early. Without a biased arm, $\epsilon$-Greedy (0.1) or Bayes-UCB often has lower regret: forced exploration costs little when all arms are close to optimal, and TS needs many pulls to separate arms whose accuracies and DP gaps nearly tie.

Figure~\ref{fig:fewer-arms-tradeoff} shows this hard case for three arms that are all vertices of the oracle frontier: GBM, GBM-OP4 and GBM-OP6 (10{,}000 runs; $\lambda$ grid from the breakpoints of these arms). All learners stay close to the frontier. TS has the lowest regret at $\lambda=0$ on all three datasets, but over all values of $\lambda$ it is lowest at only 14\% of them on Bank and 21\% on Credit (72\% on Adult), less often than for most random 3-arm subsets.

\begin{figure}[h!]
\centering
\includegraphics[width=0.7\textwidth]{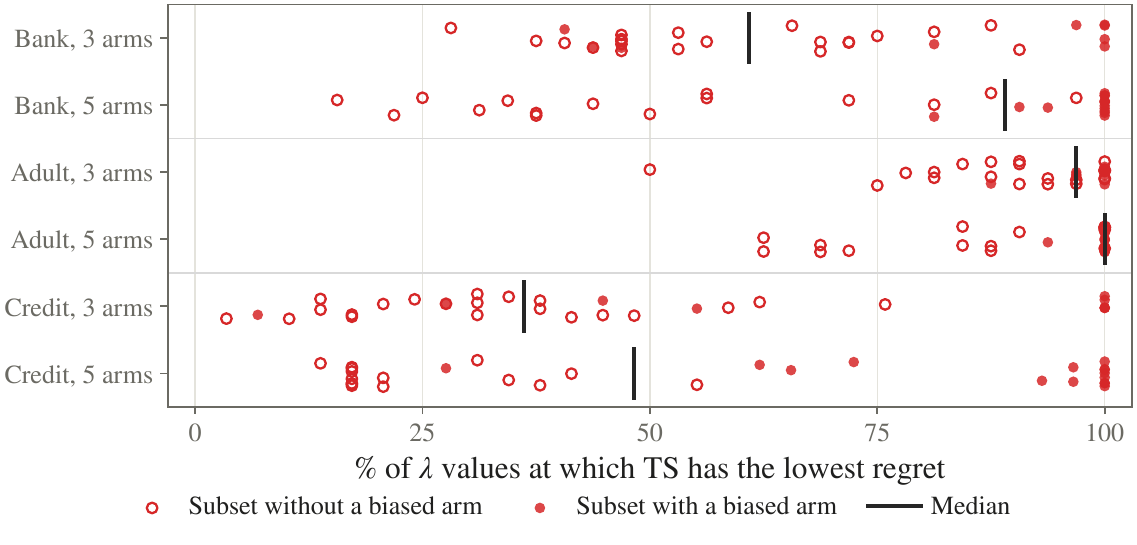}
\caption{Percentage of $\lambda$ values at which TS has a lower mean regret than Bayes-UCB, $\epsilon$-Greedy (0) and $\epsilon$-Greedy (0.1), for 30 random subsets of 3 and of 5 arms (one point per subset; 400 runs, $T=50{,}000$). Filled points: the subset contains a biased arm. Bars: medians.}
\label{fig:fewer-arms-subsets}
\end{figure}

\begin{figure}[h!]
\centering
\includegraphics[width=\textwidth]{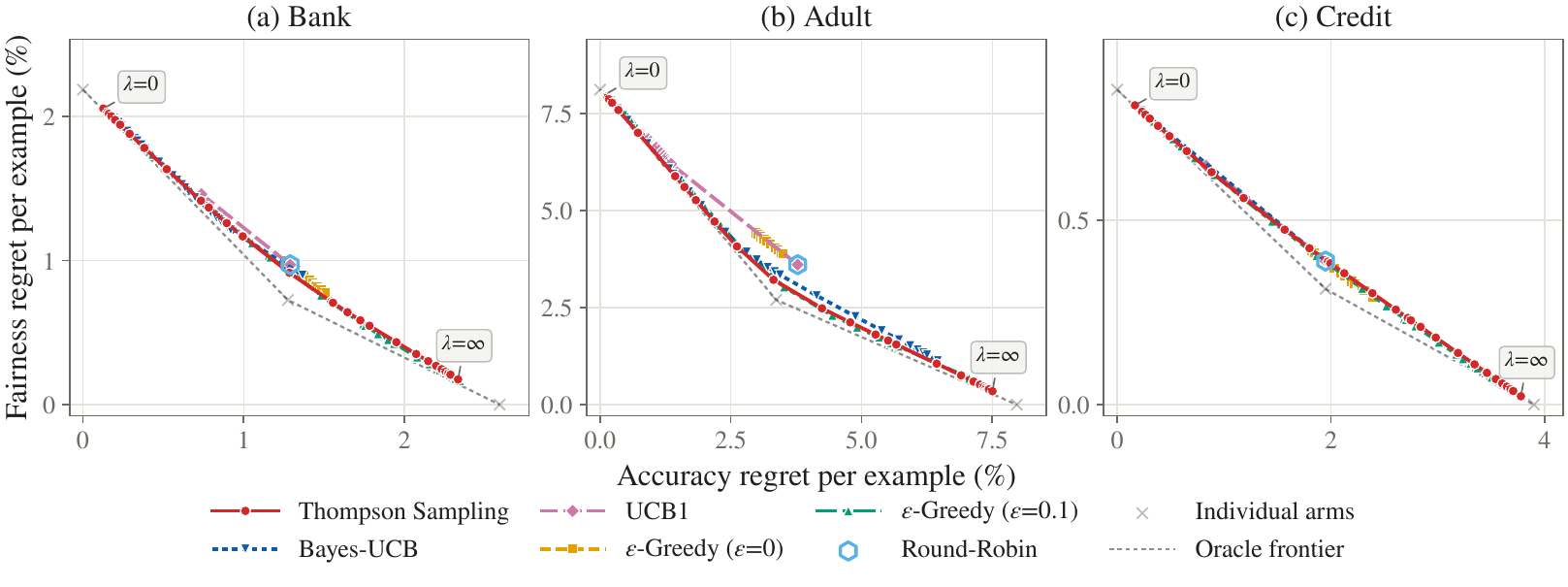}
\caption{Trade-off between accuracy and fairness regret after $T=50{,}000$ rounds with three frontier arms (GBM, GBM-OP4, GBM-OP6), as in Figure~\ref{fig:tradeoff-50k-no-delay}. Means over 10{,}000 runs.}
\label{fig:fewer-arms-tradeoff}
\end{figure}

\end{document}